\documentclass{article}

\usepackage[main, final, nonatbib]{neurips_2026}

\usepackage[utf8]{inputenc} 
\usepackage[T1]{fontenc}    
\usepackage{hyperref}       
\usepackage{url}            
\usepackage{booktabs}       
\usepackage{amsfonts}       
\usepackage{nicefrac}       
\usepackage{microtype}      
\usepackage{xcolor}         
\usepackage{subcaption}

\usepackage[numbers]{natbib}
\usepackage{amsmath}
\usepackage{amssymb}
\usepackage{mathtools}
\usepackage{amsthm}
\usepackage{bm}
\usepackage{comment}
\usepackage{etoc}
\usepackage[capitalize,noabbrev]{cleveref}
\usepackage{float}

\theoremstyle{plain}
\newtheorem{theorem}{Theorem}[section]
\newtheorem{proposition}[theorem]{Proposition}
\newtheorem{lemma}[theorem]{Lemma}

\theoremstyle{definition}
\newtheorem{definition}[theorem]{Definition}

\theoremstyle{remark}

\title{Grokking or Glitching? How Low-Precision Drives Slingshot Loss Spikes}

\author{%
  Hanqing Liu$^{1,2}$\thanks{Correspondence: \texttt{liu.hanqing.phy@gmail.com} or \texttt{liu-hanqing@g.ecc.u-tokyo.ac.jp}} \quad 
    Jianjun Cao$^1$ \quad 
    Yuanze Li$^1$ \quad 
    Zijian Zhou$^1$ \\
    \small $^1$Tsinghua University \quad $^2$The University of Tokyo
}

\begin{document}

\addtocontents{toc}{\protect\setcounter{tocdepth}{0}}

\maketitle

\begin{abstract}
  Deep neural networks exhibit periodic loss spikes during unregularized long-term training, a phenomenon known as the ``Slingshot Mechanism''~\cite{thilak2022}.
  Existing work usually attributes this to intrinsic optimization dynamics, but its triggering mechanism remains unclear. 
  This paper proves that this phenomenon is a result of floating-point arithmetic precision limits.
  As training enters a high-confidence stage, the difference between the correct-class logit and 
  the other logits may exceed the absorption-error threshold. Then during backpropagation, 
  the gradient of the correct class is rounded exactly to zero, while the gradients of the incorrect classes remain nonzero. 
  This breaks the zero-sum constraint of gradients across classes and 
  introduces a systematic drift in the parameter update of the classifier layer. 
  We prove that this drift forms a positive feedback loop with the feature, 
  causing the global classifier mean and the global feature mean to grow exponentially. 
  We call this mechanism \emph{Numerical Feature Inflation} ($\mathcal{NFI}$). This mechanism explains the rapid norm growth before a Slingshot spike, 
  the subsequent reappearance of gradients, and the resulting loss spike. We further show that $\mathcal{NFI}$ is not equivalent to 
  an observed loss spike: in more practical tasks, partial absorption may not produce visible spikes, 
  but it can still break the zero-sum constraint and drive rapid growth of parameter norms.
  Our results reinterpret Slingshot as a numerical dynamic of finite-precision training, 
  and provide a testable explanation for abnormal parameter growth and logit divergence in late-stage training.
\end{abstract}

\section{Introduction}

Loss spikes are a persistent puzzle in neural network training, creating difficulties for both theoretical understanding and practical stability. 
One representative example is the \emph{Slingshot Mechanism}, first observed in the study of grokking under no explicit regularization. 
Grokking refers to the phenomenon where neural networks achieve sudden generalization long after reaching perfect training accuracy~\cite{power2022}. 
In such settings, training is often accompanied by periodic instabilities: 
the norm of the last-layer parameters grows rapidly, often close to exponentially, 
and is followed by an abrupt training loss spike.

Existing work has mainly interpreted Slingshot as an intrinsic optimization phenomenon. 
For example, Thilak et al.~\cite{thilak2022} related it to the Edge of Stability (EOS)~\cite{cohen2021}, 
where the optimizer periodically crosses stability boundaries. 
Nanda et al.~\cite{nanda2023} suggested that the effect may arise from the interaction between gradients of 
different scales and adaptive optimizer dynamics. \emph{In contrast}, 
we show that Slingshot is not primarily caused by the intrinsic optimization dynamics of the model. 
Instead, it is triggered by finite-precision arithmetic in the computation of the cross-entropy (CE) loss.

After long training, the model enters a high-confidence regime: 
it reaches perfect training accuracy, and for each training sample, 
the correct-class logit becomes much larger than the other logits. 
CE loss converts logits into probabilities through the softmax function. 
Prieto et al.~\cite{prieto2025} showed that when the gap between the largest logit $z_m$ and 
the other logits exceeds a threshold determined by floating-point precision, 
absorption error causes the computed result to differ from the exact real-number value. 
As a result, during backpropagation, the gradient of the loss with respect to the correct-class logit, 
$\partial L / \partial z_m$, is rounded exactly to zero. They call this phenomenon \emph{Softmax Collapse} (SC).

We show that SC has a second effect that directly drives Slingshot. 
In exact arithmetic, the gradients on the rows $\bm{W}_k$ of the final classifier satisfy 
a zero-sum constraint across classes. Therefore, the global classifier mean
$\bm{W}_G = \frac{1}{K}\sum_{k=1}^{K} \bm{W}_k$
remains unchanged under gradient updates. However, SC breaks this zero-sum constraint. 
We prove that this introduces a drift of $\bm{W}_G$ in the direction of $-\bm{\mu}_G$, where
$\bm{\mu}_G = \frac{1}{B}\sum_{k,i} \bm{h}_{k,i}$
is the mean of the penultimate-layer features in a batch. We further prove that, once $\bm{W}_G$ becomes nonzero, 
the feature mean $\bm{\mu}_G$ also receives a drift in the direction of $-\bm{W}_G$. 
These two effects form a positive feedback loop: $\bm{W}_G$ and $\bm{\mu}_G$ become anti-parallel, 
and their norms grow exponentially. We call this process \emph{Numerical Feature Inflation} ($\mathcal{NFI}$).

$\mathcal{NFI}$ explains how Slingshot loss spikes are triggered. As $\bm{W}_G$ and $\bm{\mu}_G$ grow, 
the sample-level margin $z_m - \max_{k \neq m} z_k$ can become fragile for some samples and 
eventually fall below the SC threshold. The correct-class gradient then reappears abruptly. 
Before this moment, the correct-class gradient has been zero and the incorrect-class gradients have been 
extremely small, so Adam can amplify the effective learning rate to a large value. When the correct-class gradient 
suddenly changes from zero to a finite value, this large effective learning rate produces a large parameter update, 
causing the loss spike.

\begin{figure}[t]
  \centering
     \begin{subfigure}[b]{0.325\textwidth}
         \centering
         \includegraphics[width=0.96\textwidth]{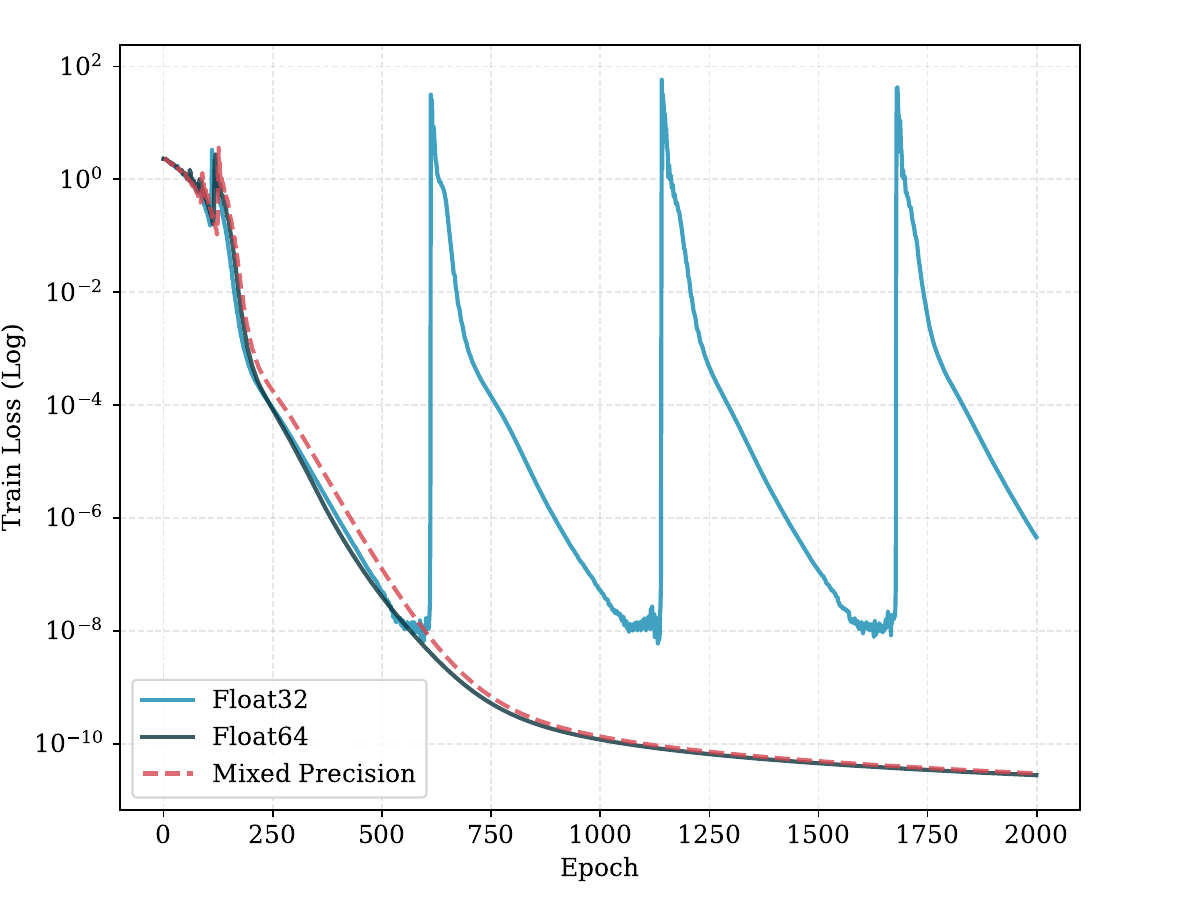}
         \caption{}
         \label{fig:sling}
     \end{subfigure}
     \hfill 
     \begin{subfigure}[b]{0.325\textwidth}
         \centering
         \includegraphics[width=0.96\textwidth]{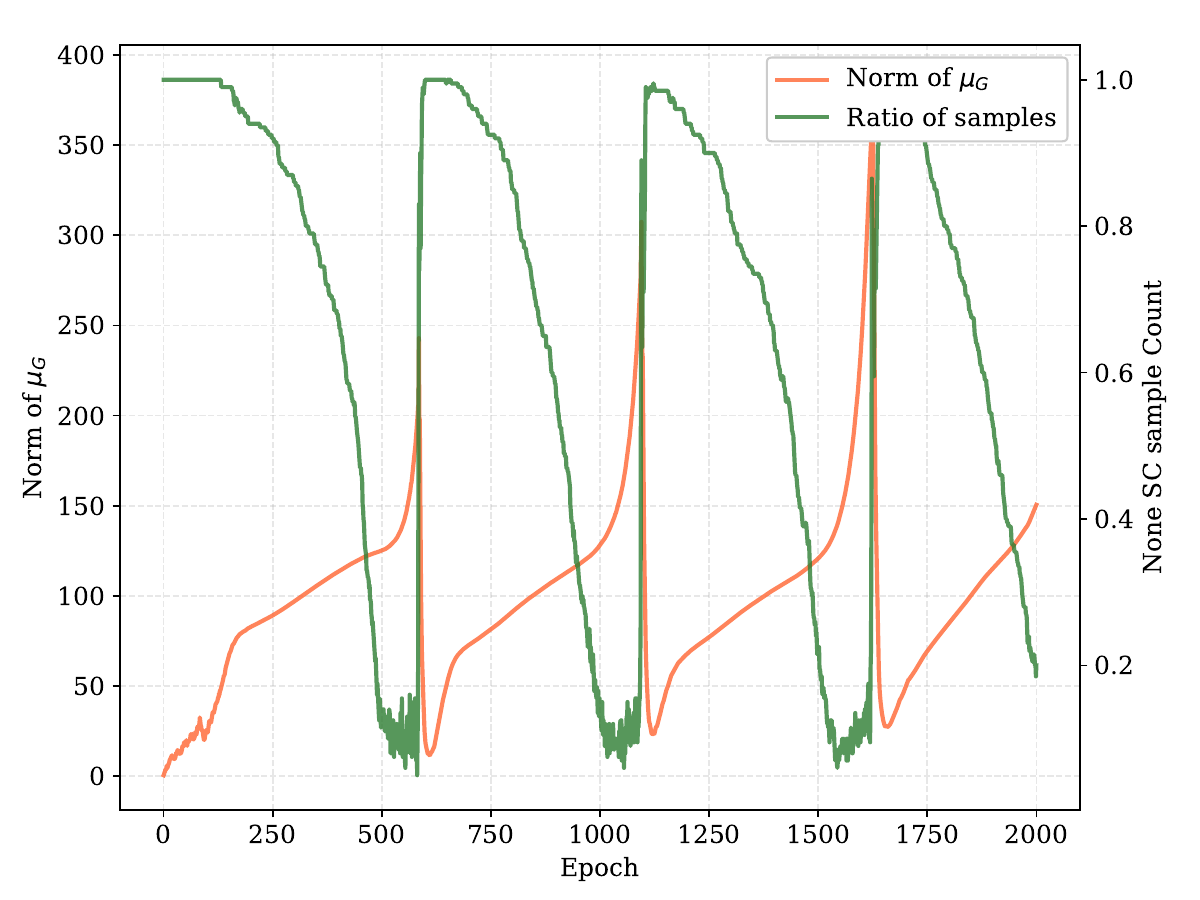}
         \caption{}
         \label{fig:count}
     \end{subfigure}
     \hfill
    \begin{subfigure}[b]{0.325\textwidth}
        \centering
         \includegraphics[width=0.96\textwidth]{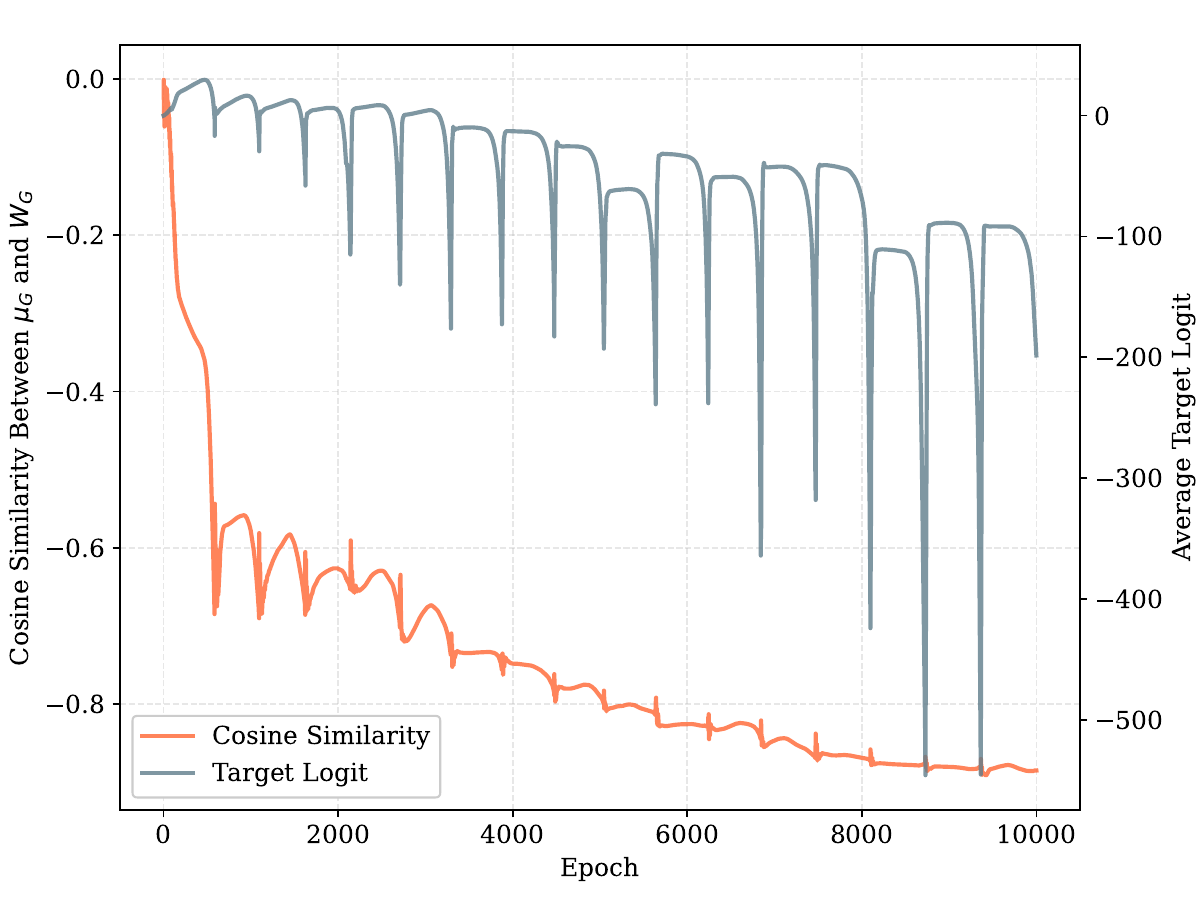}
         \caption{}
         \label{fig:drift}
     \end{subfigure}
        
     \caption{
\textbf{Precision-induced $\mathcal{NFI}$ dynamics.}
\textbf{(a)} Slingshot loss spikes disappear when training is performed in float64. Casting only the logits/loss computation to float64 is also sufficient to remove the spikes, showing that the instability originates from the loss computation.
\textbf{(b)} Before most samples enter Softmax Collapse, the global feature mean grows slowly. Once most samples collapse, $\|\bm{\mu}_G\|$ enters a rapid-growth phase.
\textbf{(c)} The cosine similarity between $\bm{W}_G$ and $\bm{\mu}_G$ approaches $-1$. Under this anti-parallel alignment, even the correct-class logits become deeply negative.
}
\vspace{-0.6cm}
\end{figure}

Beyond Slingshot, we show that $\mathcal{NFI}$ also explains abnormal parameter growth in more realistic settings. 
Classical gradient-flow analyses~\cite{lyu2020,soudry2018} suggest that, without explicit regularization, 
separable models maximize margin and their parameter norms grow only logarithmically. 
This prediction is often inconsistent with practical, 
where unregularized models can show rapid growth in parameter norms and logits. 
We argue that this discrepancy arises because gradient-flow analyses do not account for 
finite-precision effects in loss computation. After removing the $\mathcal{NFI}$ dynamics, 
the growth of model parameters is substantially reduced.

Our work is also relevant to the current shift toward low-precision training. 
Recent studies have identified several sources of instability in low-precision pre-training, 
including quantization errors in matrix multiplication~\cite{wortsman2023}, 
attention-sink-induced loss spikes~\cite{xiao2024}, and rounding-error accumulation in Flash Attention~\cite{qiu2025}. 
Our work identifies a different failure mode: numerical error in the computation of log-probabilities. 
As training systems move toward increasingly low precision, 
understanding this mechanism is important for improving the stability of large-scale model training.

Our contributions are summarized as follows:
\begin{itemize}
    \item We provide a theoretical explanation of the Slingshot Mechanism based on finite-precision cross-entropy computation.
    \item We identify Numerical Feature Inflation, a feedback loop between the global classifier mean and the global feature mean, as a mechanism for abnormal parameter and logit growth after long-term unregularized training.
    \item We propose and validate practical interventions, such as restoring the zero-sum constraint or removing the last-layer mean, that suppress $\mathcal{NFI}$-induced instability and parameter growth.
\end{itemize}

\section{Related Work}

\paragraph{Grokking and Numerical Dynamics.}
Weight decay has traditionally been credited as the primary driver of grokking~\cite{liu2023, lyu2024, boursier2025, tian2025}. 
However, growing empirical evidence shows that grokking can also occur without explicit regularization, 
although many of these demonstrations are conducted under mean-squared error (MSE) loss~\cite{gromov2023, kumar2024, golechha2024, prakash2025}. 
In contrast, achieving grokking with CE loss and no weight decay appears theoretically impractical: 
Lyu et al.~\cite{lyu2024} estimate that it would require approximately $10^{100}$ steps under Gradient Descent. 
Nevertheless, Thilak et al.~\cite{thilak2022} observed grokking under CE loss with Adam, 
where generalization is accompanied by recurrent Slingshot loss spikes. 
Given the theoretical difficulty of unregularized CE grokking, 
recent work has pivoted to examining the role of numerical precision.
Nanda et al.~\cite{nanda2023} first linked the Slingshot Mechanism to numerical instability, noting that lower precision exacerbates these events.
Xu et al.~\cite{xu2025} subsequently found that employing \texttt{float64} precision could mitigate the Slingshot effect. 
More recently, Prieto et al.~\cite{prieto2025} conducted a detailed analysis of these instabilities, identifying ``Softmax Collapse'' (SC) as a critical phenomenon.
They showed that in PyTorch's CE loss implementation, absorption errors occur 
when the correct logit $z_m$ exceeds incorrect logits by a threshold determined by the mantissa precision 
(for \texttt{float32}, $z_{m}-\max_{k\neq m} z_{k} > 23\ln 2\approx 16$). 
Under SC, the gradient on the correct class strictly rounds to zero, halting test accuracy improvement.
However, Prieto et al. stopped short of establishing a causal link between SC and the Slingshot mechanism itself, 
hypothesizing instead that Slingshots might serve as an intrinsic optimization response to avoid SC.

\paragraph{Neural Collapse.}
Recent studies indicate that neural networks tend to converge to a state known as Neural Collapse (NC) after prolonged training \cite{papyan2020}. 
Theoretical works have established that under CE loss without weight decay, 
the NC state represents a global optimum \cite{lu2022, garrod2025}. 
Furthermore, Dang et al. \cite{dang2024} demonstrated that with ReLU activation, 
the class means of feature vectors tend to become mutually orthogonal. 
Sakamoto \& Sato \cite{sakamoto2025} established a connection between grokking and NC. 
While generalization does not strictly require NC, Han et al. \cite{han2025} suggest that models exhibit a strong propensity toward NC in the absence of explicit regularization.

\paragraph{The Instability of Adam and Loss Spikes.}
The convergence properties of Adam~\cite{kingma2015} have been extensively studied. 
Reddi et al. \cite{reddi2018} highlighted the non-convergence issues of Adam. 
Shazeer \& Stern \cite{shazeer2018} showed that a slow decay rate of the second moment accumulator $v_t$ can cause larger-than-desired updates and training instability.
Zhang et al. \cite{zhang2022} showed that appropriate hyperparameter selection can ensure convergence. 
Cohen et al. \cite{cohen2023, cohen2025} investigated the EOS phenomenon and loss spikes in adaptive optimization.
Molybog et al. \cite{molybog2023} attributed loss spikes to time-domain gradient correlations 
that cause the Adam update ratio to shift into a bimodal distribution.
In this regime, anomalous data batches trigger a chain reaction 
that forces the update vector to depart from its suppressed state, causing the training loss to explode. 
Subsequently, Bai et al. \cite{bai2025} provided a stability-based explanation, 
observing that when gradients remain small for extended periods (flat landscapes), 
the adaptive term $v_t$ decays, causing the effective learning rate $\eta / (\sqrt{v_t} + \varepsilon)$ to explode. 
In standard settings ($\eta=10^{-3}, \varepsilon=10^{-8}$), 
this effective step size can be amplified by orders of magnitude (e.g., $10^5$), 
making the model highly susceptible to drastic updates upon the re-emergence of gradient signals. 
However, these analyses typically focus on MSE loss where the Hessian is non-vanishing. 
In contrast, Ma et al. \cite{ma2022} theoretically argued that under CE loss, 
the Hessian eigenvalues vanish in the late training phase, implying that the optimization should stabilize and be immune to such spikes.


\section{The Mechanics of \emph{Numerical Feature Inflation}}
\label{sec:theory}
In this section, we derive the theoretical mechanism behind the training instabilities 
observed in low-precision training.
As shown in \cref{fig:sling}, our results indicate that the Slingshot mechanism is fundamentally a numerical precision artifact. 
Even when model parameters are stored in \texttt{float32}, 
casting the output logits to \texttt{float64} solely during the loss computation is sufficient to eliminate the Slingshot effect. 
As analyzed by Prieto et al.~\cite{prieto2025}, this instability arises from absorption errors in floating-point arithmetic.

\subsection{Preliminaries}
According to the IEEE 754 standard~\cite{IEEE754-2019}, 
a floating-point number consists of 1 sign bit $s$, $E$ exponent bits, and $p$ mantissa bits.

\begin{definition}[\textbf{Absorption Error}]
Consider the addition of two non-zero floating-point numbers $a$ and $b$ (where $|a| \ge |b|$). The operation requires exponent alignment. If the ratio satisfies
$\frac{|b|}{|a|} < 2^{-(p-1)}$,
the smaller value $b$ cannot be represented within the mantissa precision after alignment. This results in the absorption error, defined as $a + b = a$.
\end{definition}
For \texttt{float32}, the mantissa precision is $p=24$ bits. 
Consequently, the critical threshold is $2^{-(24-1)} = 2^{-23} \approx 1.19 \times 10^{-7}$.

Absorption errors occur in Softmax Cross-Entropy loss calculation.
PyTorch \cite{paszke2019} implements the Softmax CE loss using the Log-Sum-Exp trick for numerical stability. 
The denominator is stored as:
\begin{equation}
Z = \log\left(\sum_k \exp(z_k)\right) = z_m + \log\left(\sum_k \exp(z_k - z_m)\right)
\end{equation}
where $z_k$ represents the output logit for class $k$, and $z_m = \max z_k$.
If the margin between the maximum logit and others satisfies $z_m - \max_{k \neq m} z_k > (p-1)\ln 2$, 
the second term vanishes due to absorption error, resulting in $Z = z_m$.
In the late stages of training, the maximum logit $z_m$ typically corresponds to the correct class logit $z_r$. 
\begin{definition}[\textbf{Softmax Collapse}~\cite{prieto2025}]
Under absorption error (where $Z = z_r$), the gradient for the logit of correct class $r$ becomes strictly zero:
\begin{equation}
g_r = \hat{y}_r - y_r = e^{z_r - Z} - 1 = e^{z_r - z_r} - 1 = 0
\end{equation}
At this point, the loss also becomes $-\log(\hat{y}_r) = 0$. 
However, for incorrect classes $k \neq r$, the gradients remain small but non-zero ($g_k = e^{z_k - z_r} \neq 0$). 
\end{definition}

Since Softmax Collapse (SC) arises in the terminal phase of training, 
its impact is inextricably linked to the geometric structure of the feature space in this regime.
We assume the model has converged to the Neural Collapse (NC) state,
which characterizes the geometry of
the classifier weight matrix $\bm{W}$ and its input activations (features) $\bm{h}$.

\begin{definition}[\textbf{Neural Collapse}~\cite{papyan2020}]
Consider a classification task with $K$ classes. The Neural Collapse state is defined by the following conditions: 
\textbf{NC1:} Variability Collapse. Intra-class feature variability vanishes. For any sample $i$ of class $k$, the feature vector $\bm{h}_{k,i}$ converges to the class mean $\bm{\mu}_k$;
\textbf{NC2:} Simplex ETF. The centered class means $\bm{\mu}_k^* = \bm{\mu}_k - \bm{\mu}_G$ (where $\bm{\mu}_G$ is the global mean) form a Simplex Equiangular Tight Frame (ETF);
\textbf{NC3:} Self-Duality. The classifier weight row vectors $\bm{W}_k$ align with the centered class means $\bm{\mu}_k^*$.
\end{definition}

Prieto et al.~\cite{prieto2025} identified SC, 
and hypothesized that it merely halts generalization, causing test accuracy to plateau. 
However, our investigation reveals a critical secondary effect: 
the interaction between SC and NC induces \textit{Numerical Feature Inflation}, 
a process that directly triggers the Slingshot Mechanism.

\subsection{Numerical Feature Inflation}
\label{sec:nfc} 

In this section, we formalize the mechanism of Numerical Feature Inflation ($\mathcal{NFI}$). 
We demonstrate that the interaction between SC and NC creates 
a deterministic feedback loop that drives feature inflation.

As established before, floating-point arithmetic introduces absorption errors. 
We first quantify the breaking of the zero-sum constraint.

\begin{theorem}
\label{thm:weight_drift}
Consider a network with CE loss. 
Assume the model is in an approximate NC state and satisfies the SC condition. 
During Gradient Descent with learning rate $\eta$, 
the expected update to global mean of the classifier weights $\bm{W}_G= \frac{1}{K} \sum_{k=1}^K \bm{W}_k$ 
on a class-balanced batch $\mathcal{B}$ is:
\begin{equation}
    \mathbb{E}_{\mathcal{B}}[\Delta \bm{W}_G] = - \frac{\eta \epsilon }{K} \bm{\mu}_G
\end{equation}
where $\epsilon = \mathbb{E}[\sum_{k \neq r} \hat{y}_k]$ represents the expected residual probability mass on incorrect classes.
\end{theorem}

\begin{proof}[Proof Sketch]
Consider an input $\bm{x}$ with feature $\bm{h}$ and label $y_r$. Ideally, the gradient of the loss $\mathcal{L}$ with respect to $\bm{W}_k$ satisfies a strict zero-sum constraint:
\begin{equation}
    \sum_{k=1}^K \nabla_{\bm{W}_k} \mathcal{L} = \sum_{k=1}^K (\hat{y}_k - y_k) \bm{h} = \left(\underbrace{\sum_k \hat{y}_k}_{1} - \underbrace{\sum_k y_k}_{1}\right) \bm{h} = 0 \label{eq:zero_sum}
\end{equation}  
However, under SC, the gradient on correct class vanishes mathematically. The sum becomes:
    $\sum_{k=1}^K \nabla_{\bm{W}_k} \mathcal{L} \xrightarrow{SC} \sum_{k \neq r} \hat{y}_k \bm{h} = \epsilon \bm{h}$.
The update rule $\bm{W} \leftarrow \bm{W} - \eta \nabla \bm{W}$ thus imparts a net drift to $\bm{W}_G$ in the direction of $-\bm{\mu}_G$.
\end{proof}

This drift necessitates a redefinition of the geometric alignment.

\begin{definition}[NC3$'$]
When $\bm{W}_G \neq \bm{0}$, the original weights $\bm{W}_k$ no longer form an ETF. Instead, the centered weights $\bm{W}_k^* = \bm{W}_k - \bm{W}_G$ satisfy the self-duality property, aligning with $\bm{\mu}_k^*$ and forming an ETF.
\end{definition}

This weight drift alters the gradient of feature layer.
\begin{proposition}
\label{thm:nfc_force}
Assume the model satisfies NC1, NC2, and NC3$'$, and that $\bm{W}_G$ is orthogonal to 
the classification subspace ($\bm{W}_G \perp \text{span}\{\bm{W}_k^*\}$). 
Under the SC regime, the gradient of the loss with respect to the feature vector $\bm{h}$ contains a non-zero component parallel to $\bm{W}_G$:
\begin{equation}
    \text{Proj}_{\bm{W}_G}(\nabla_{\bm{h}} \mathcal{L}) = \epsilon \bm{W}_G
\end{equation}
\end{proposition}

\begin{proof}[Proof Sketch]
The gradient is $\nabla_{\bm{h}} \mathcal{L} = \sum_k (\hat{y}_k - y_k)\bm{W}_k=\sum_k (\hat{y}_k - y_k)(\bm{W}_G + \bm{W}_k^*)$. 
Since $\bm{W}_G \perp \text{span}\{\bm{W}_k^*\}$, the projection onto $\bm{W}_G$ is proportional to $\bm{W}_G$.
\end{proof}

Combining Theorem \ref{thm:weight_drift} and Proposition \ref{thm:nfc_force} reveals the feedback mechanism.
The mutual reinforcement between the weight drift $\bm{W}_G$ and the feature drift $\bm{\mu}_G$ creates a coupled dynamic system:

\begin{theorem}[\textbf{Numerical Feature Inflation}]
\label{thm:nfc}
Consider a ReLU network with CE loss on class-balanced batches as in \cref{thm:weight_drift}. 
Under the geometric assumptions of Proposition \ref{thm:nfc_force}, 
Softmax Collapse induces a positive feedback loop. 
If the change of $\epsilon$ is negligible compared to the parameter dynamics (i.e., $|\dot{\epsilon}| \ll \frac{d}{dt}\|\bm{W}_G\|$),
The norms of the global weight mean $\bm{W}_G$ and feature mean $\bm{\mu}_G$ exhibit exponential growth after long-term training:
\begin{align}
    \lim_{t \to \infty} \|\bm{W}_G^{(t)}\| &\propto \left(1 + \frac{\eta\epsilon}{\sqrt{K}}\right)^t \\
    \lim_{t \to \infty}\|\bm{\mu}_G^{(t)}\| &\propto \left(1 + \frac{\eta\epsilon}{\sqrt{K}}\right)^t \label{eq:mu_G}
\end{align}
where $\bm{W}_G$ progressively aligns anti-parallel to $\bm{\mu}_G$:
\begin{equation}
    \lim_{t \to \infty} \cos(\bm{W}_G^{(t)}, \bm{\mu}_G^{(t)}) \to -1
\end{equation}
\end{theorem}

See Appendix \ref{appendix:proofs} for the detailed proof.

\subsection{The Mechanism of the Slingshot Mechanism}
\label{sec:mec}
The exponential growth of the global classifier mean and the global feature mean 
does not by itself immediately produce a loss spike. 
Instead, it drives the system into a regime where the sample-level absorption condition becomes fragile.

In an ideal NC state, all features $\bm{h}_{k,i}$ of class $k$ are close to their class mean $\bm{\mu}_k$. 
In this case, within the subspace $\text{span}\{\bm{\mu}_k^*\}$ orthogonal to $\bm{\mu}_G$, 
the remaining gradients from incorrect classes still push the features of different classes away from each other. 
Therefore, a Slingshot spike does not occur immediately. 
However, this approximation no longer holds once the exponential growth of the global feature mean becomes dominant. 
As shown in Proposition \ref{thm:nfc_force}, the growth along the $\bm{\mu}_G / \bm{W}_G$ direction is proportional to 
the residual probability mass on incorrect classes, denoted by $\epsilon$. 
Thus, samples closer to the decision boundary, which have larger residual probability mass, 
move faster along the $\bm{\mu}_G$ direction. After this exponential growth continues for some time, 
the intra-class variance can become comparable to the inter-class variance. 
At this stage, the gradient updates may still increase the distance between two class means $\bm{\mu}_p$ and $\bm{\mu}_q$, 
but they can also reduce the margin between some outlier samples $\bm{h}_{p,o}$ and the decision boundary. 
Once this margin falls below the absorption threshold $z_m - \max_{k \neq m} z_k > (p - 1) \ln 2$, 
the correct-class gradient signal reappears.

Before this re-emergence, the correct-class gradient is zero and the incorrect-class gradients are extremely small. 
As a result, the effective learning rate of Adam is amplified toward its theoretical maximum $\eta / \varepsilon_\text{Adam}$. 
When the correct-class gradient of some samples suddenly changes from zero to a finite value, 
Adam cannot immediately reduce it because the first and second moments respond on different timescales. 
This produces an excessively large update step. In our experiments, 
the update magnitude at the spike epoch is typically about 50 times larger than that in the previous epoch. 
Such a large update substantially changes the model parameters and drives the loss back to the random-guessing level. 
This produces the loss spikes known as the \emph{Slingshot Mechanism}.

Our theory explains the exponential growth of $\|\bm{W}_G\|$, $\|\bm{\mu}_G\|$ and the resulting logit shift before the spike. 
After the spike, parameter norms remain relatively stable, while $\|\bm{\mu}_G\|$ and the logit magnitudes rapidly return to their pre-inflation levels. 
We hypothesize that training progressively aligns the backbone weights with patterns in the inputs, 
amplifying feature responses across layers. The large update at the spike may disrupt this alignment, 
so the same inputs produce smaller features and logits at comparable weight norms.

\begin{figure}[t]
     \centering
     
     \begin{subfigure}[b]{0.325\textwidth}
         \centering
        \includegraphics[width=\textwidth]{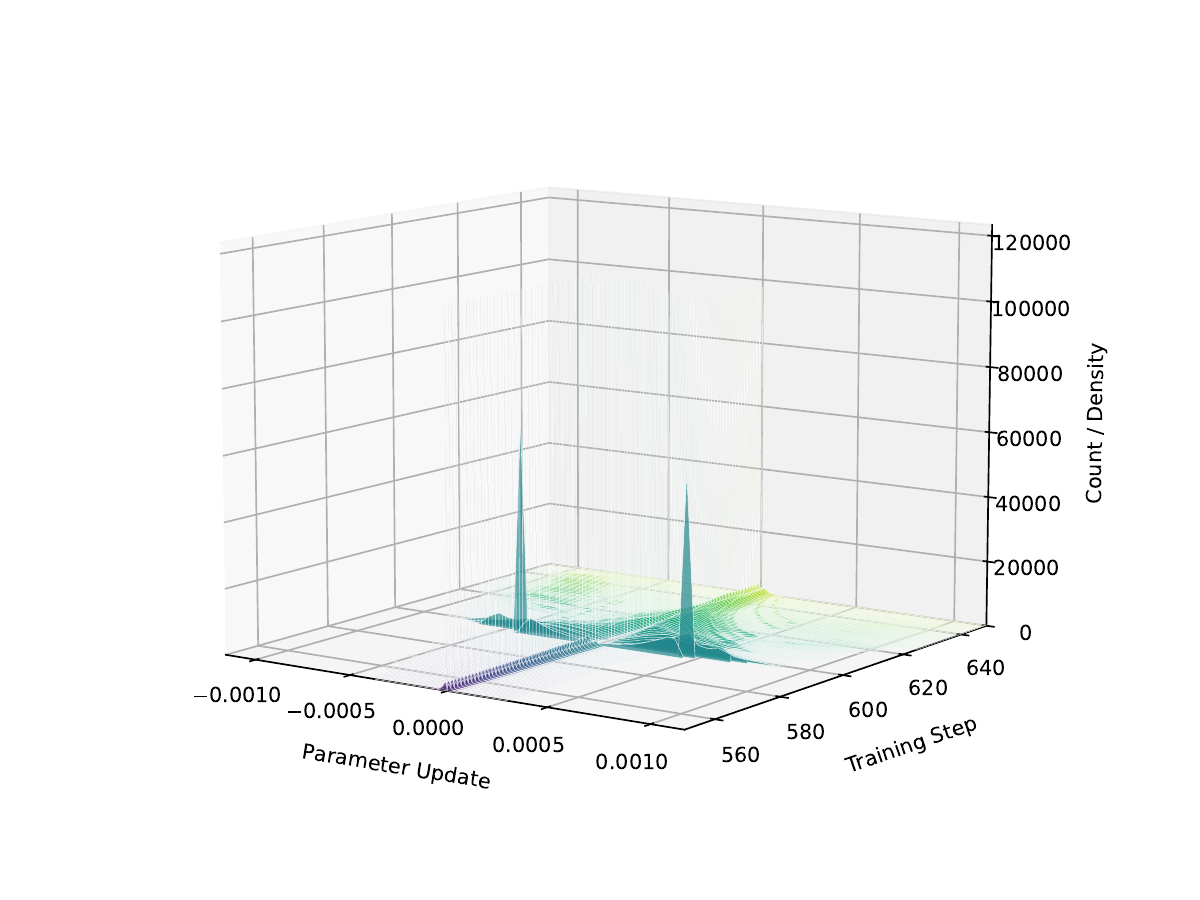}
        \caption{}
        \label{fig:weight_update}
    \end{subfigure}
    \hfill
    \begin{subfigure}[b]{0.325\textwidth}
        \centering
         \includegraphics[width=\textwidth]{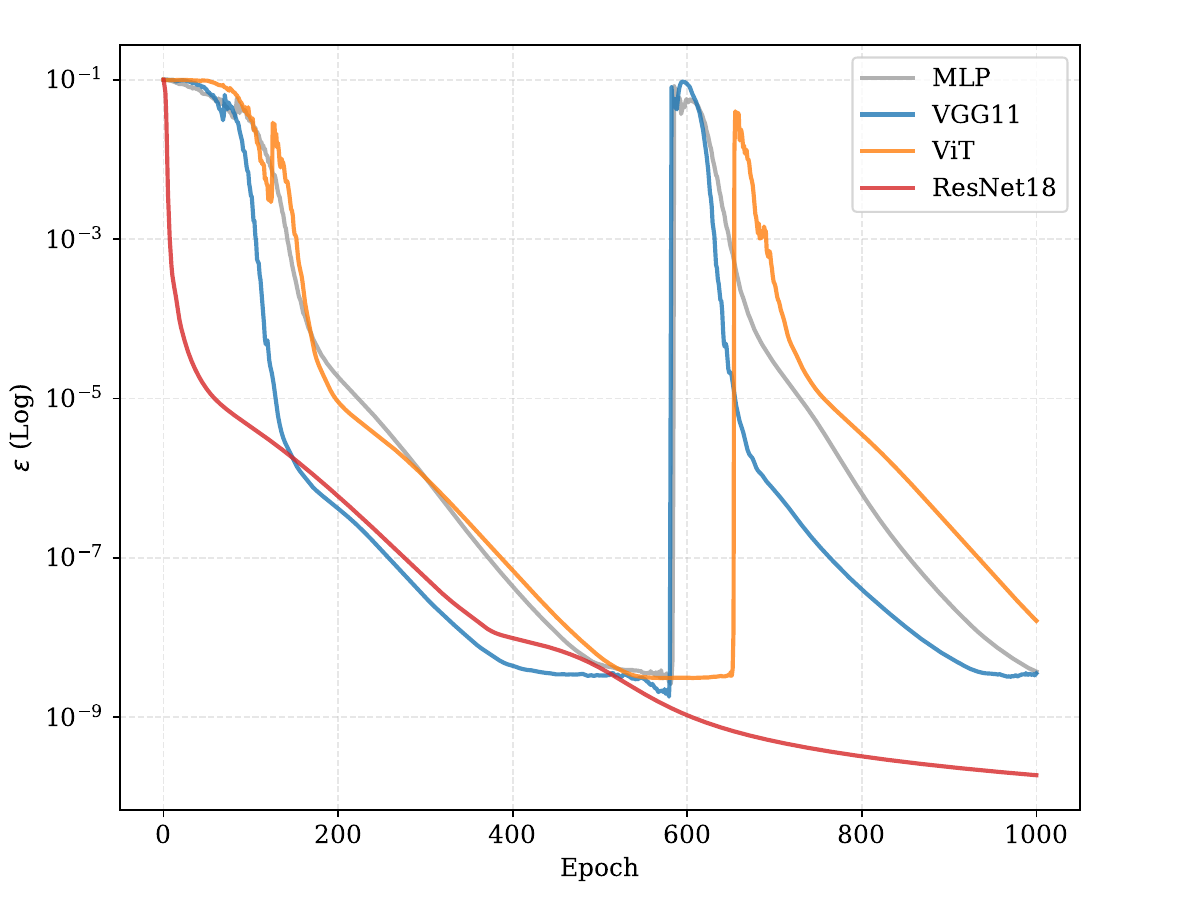}
         \caption{}
         \label{fig:eps}
     \end{subfigure}
     \hfill
     \begin{subfigure}[b]{0.325\textwidth}
         \centering
         \includegraphics[width=\textwidth]{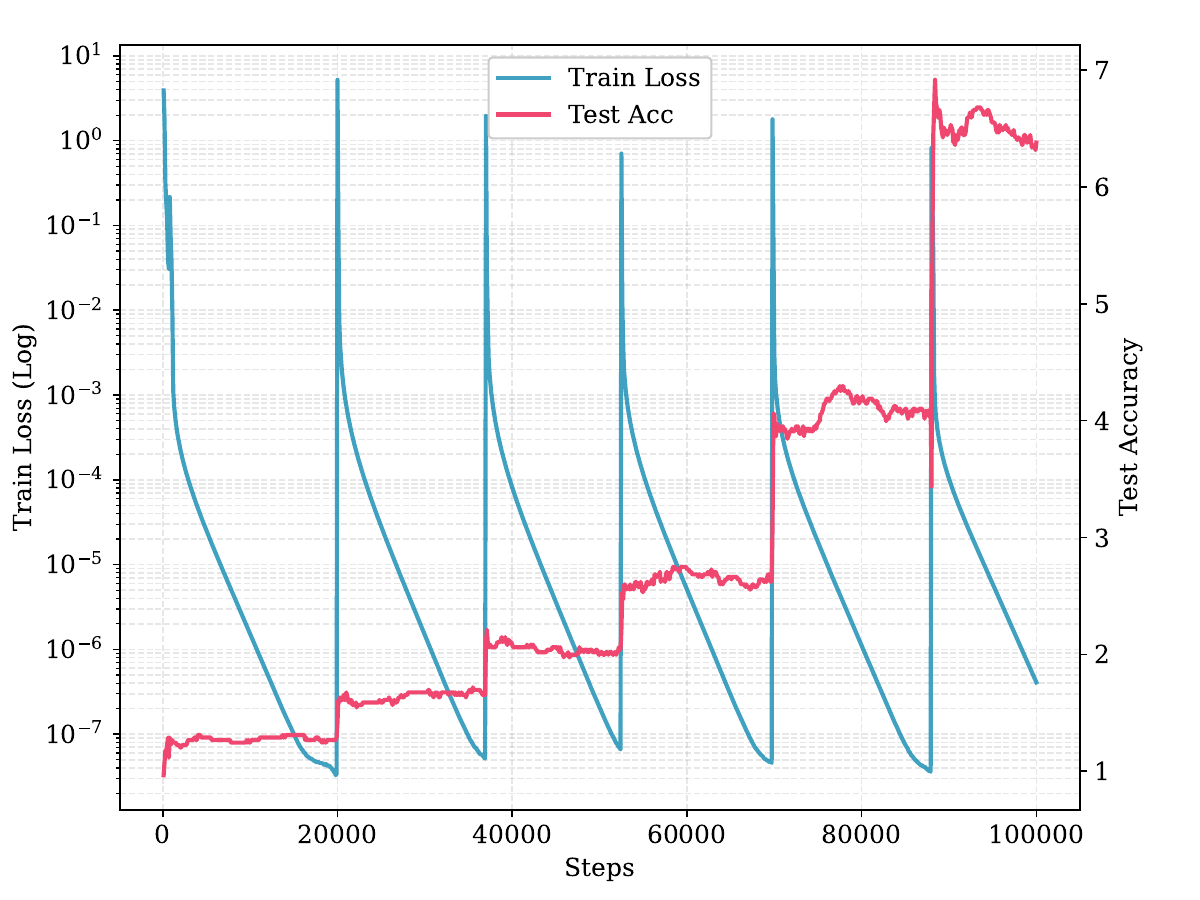}
         \caption{}
         \label{fig:trm_loss}
     \end{subfigure}

     \caption{
\textbf{Mechanistic evidence for Slingshot spikes.}
\textbf{(a)} Distribution of classifier-layer parameter updates around a loss spike. Before the spike, updates concentrate near zero. At the spike step, the distribution forms two sharp modes near $-4\times 10^{-4}$ and $4\times 10^{-4}$. After the spike, update magnitudes become dispersed across parameters.
\textbf{(b)} Evolution of the residual probability mass $\epsilon$ across architectures. Before a spike, $\epsilon$ decreases slowly or stagnates in MLP, VGG11, and ViT, allowing the $\mathcal{NFI}$ feedback loop to grow. In contrast, $\epsilon$ in ResNet18 keeps decreasing rapidly, preventing Slingshot.
\textbf{(c)} In modular division, each Slingshot loss spike is accompanied by a stepwise increase in test accuracy.
}
     \label{fig:three_graphs}
\end{figure}

\section{Empirical Results and Analysis}
\label{sec:experiment}
\paragraph{Experiment Setup.}
We adopt the experimental framework of Thilak et al.~\cite{thilak2022}, 
covering all their configurations except for linear models, 
while extending our evaluation to additional architectures. 
Our experiments focus on 3 datasets:
(1) \textbf{Modular Arithmetic:} For the modular division task with prime $p=97$, 
we employ a 2-layer decoder-only Transformer and a 6-layer fully connected MLP.
(2) \textbf{Image Classification:} We utilize the CIFAR-10 dataset. To ensure a comprehensive evaluation, 
We train a 6-layer MLP, CNNs (VGG11, ResNet18), 
and a small 12-layer Vision Transformer (ViT).
(3) \textbf{Language Modeling:} We train a nanoGPT model with 110M parameters on the FineWeb dataset, following the setup of Stollenwerk et al.~\cite{stollenwerk2026}. 
Across all experiments, we use the CE loss with zero weight decay, and the Adam optimizer. 
Detailed hyperparameters are provided in Appendix \ref{app:exp_details}.

\subsection{Mechanistic Validation of $\mathcal{NFI}$}

\paragraph{Gradient Re-emergence and Loss Spikes.}
Following the mechanism described in Section~\ref{sec:mec}, we estimate how a Slingshot loss spike is triggered. 
Consider a scalar toy model with global learning rate $\eta = 10^{-3}$, $\beta_1, \beta_2 = 0.9, 0.95$. 
Before the spike, the average gradient signal measured over all samples is approximately $3 \times 10^{-9}$. 
When the spike is triggered, the re-emerged gradient signal is larger than $\exp(-(p-1)\ln 2) = 1.19 \times 10^{-7}$. 
For the first moment, we obtain $m_t = 0.9 \times 3 \times 10^{-9} + 0.1 \times 1.19 \times 10^{-7} = 1.46 \times 10^{-8}$.
For the second moment, we obtain$\sqrt{v}=\sqrt{9\times 10^{-18}\times 0.95+1.19^2 \times 10^{-14}\times 0.05}=2.7\times 10^{-8}$. 
Therefore, the Adam update magnitude at the spike step is approximately $\eta\frac{m_t}{\sqrt{v_t}+\varepsilon}=4\times 10^{-4}$.

This estimate closely matches our empirical observation. 
See \cref{fig:weight_update}, before the spike, the average update magnitude in the classifier layer is around $1 \times 10^{-5}$. 
At the spike epoch, however, the update distribution develops two sharp modes around $4 \times 10^{-4}$ and $-4 \times 10^{-4}$. 
This more than $40$-fold increase in update magnitude is equivalent to 
applying a large signed displacement to many parameters. 
As a result, the predictions become almost random, and the loss increases to the $10^0$ scale. 
This produces the observed Slingshot loss spike. 
We also test the commonly used Adam hyperparameters in computer vision, $\beta_1, \beta_2 = 0.9, 0.999$, 
and the suggested setting by Orvieto and Gower~\cite{orvieto2026}, $\beta_1, \beta_2 = 0.9, 0.9$. 
In both cases, the observed update spikes occur at magnitudes consistent with the same calculation.
\paragraph{Zero-Sum Projection.}
To verify Theorem~\ref{thm:weight_drift}, we constrain the update of $W$ to remain in the subspace satisfying 
the zero-sum constraint in Eq.~\ref{eq:zero_sum}. Specifically, we apply a projection to the logit gradient $g=\nabla_z L$, 
We replace it by $g \gets g - \frac{1}{K}\sum^K_{k} g_k$. 
This projection makes the gradient update of every sample satisfies the zero-sum constraint in Eq.~\ref{eq:zero_sum}. 
After enforcing this constraint throughout training, the Slingshot spike is eliminated. 
This result supports the prediction that the zero-sum-breaking component is necessary for the $\mathcal{NFI}$ feedback loop.

\paragraph{Architectural Dependence.}
We further test whether Slingshot occurs across different datasets and architectures, as summarized in \cref{tab:slingshot_summary}. 
We find that all tested models exhibit Slingshot spikes except ResNet18. 
A closer analysis shows that this exception is consistent with Theorem~\ref{thm:nfc}. 
The $\mathcal{NFI}$ mechanism requires the change of $\epsilon$ to be negligible compared with the dynamics of $W$ and $\mu$. 
This condition holds for the other architectures. 
As shown in \cref{fig:eps}, before the loss spike occurs, $\epsilon$ decreases slowly or stays nearly constant. 
Therefore, the factor $(1+\frac{\eta\epsilon}{\sqrt{K}})^t$ can support exponential growth.
In contrast, if $\epsilon$ decays too fast, for example $\epsilon(t) \propto \frac{1}{t}$ or 
$\epsilon(t) \propto \frac{\log t}{t}$, 
then the accumulated growth of $(1+\frac{\eta\epsilon(t)}{\sqrt{K}})^t$ is polynomial or logarithmic. 
Therefore, $\mathcal{NFI}$ cannot develop into the exponential-growth regime, and Slingshot does not occur. 
ResNet18 appears to learn fast enough to escape this unstable region before the $\mathcal{NFI}$ feedback loop becomes dominant.
Across five seeds per architecture, spikes occurred in all 20 CIFAR-10 runs with the MLP, VGG11, VGG11 with BN, and ViT.

\paragraph{Connection to grokking.}
An interesting consequence of the Slingshot Mechanism is its connection to grokking. 
As shown in Figure~\ref{fig:trm_loss}, in the modular division task, each loss spike is accompanied by 
a stepwise increase in test accuracy, which eventually reaches $100\%$ after 500,000 steps. 
We interpret these periodic spikes as a form of implicit perturbation: 
they repeatedly move the model away from its current low-loss state and allow training to resume in 
a different region of the loss landscape. 
This may help the model reach flatter and more generalizable solutions.

\begin{figure}[ht]
     \centering
     \begin{subfigure}[b]{0.325\textwidth}
         \centering
         \includegraphics[width=0.96\textwidth]{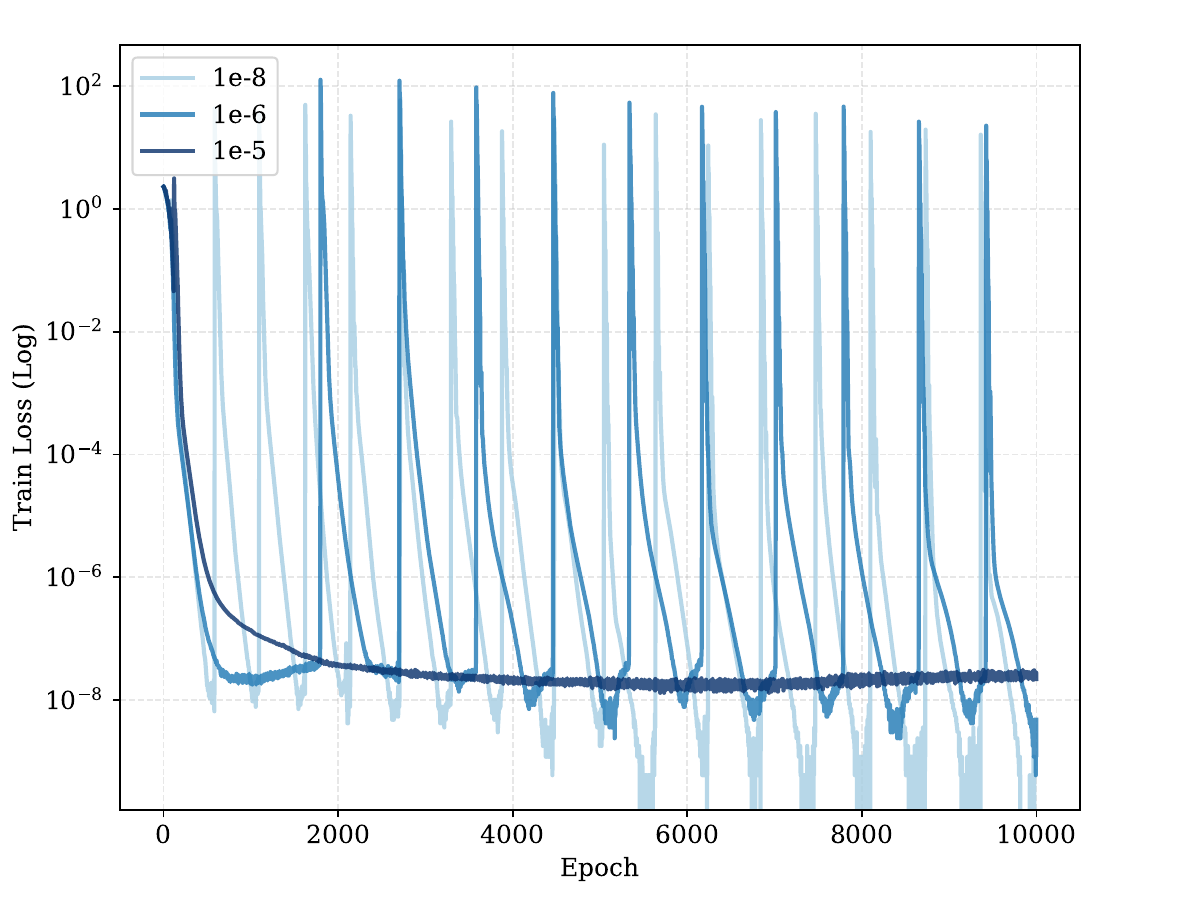}
         \caption{}
         \label{fig:adam}
     \end{subfigure}
     \hfill 
     \begin{subfigure}[b]{0.325\textwidth}
         \centering
        \includegraphics[width=0.96\textwidth]{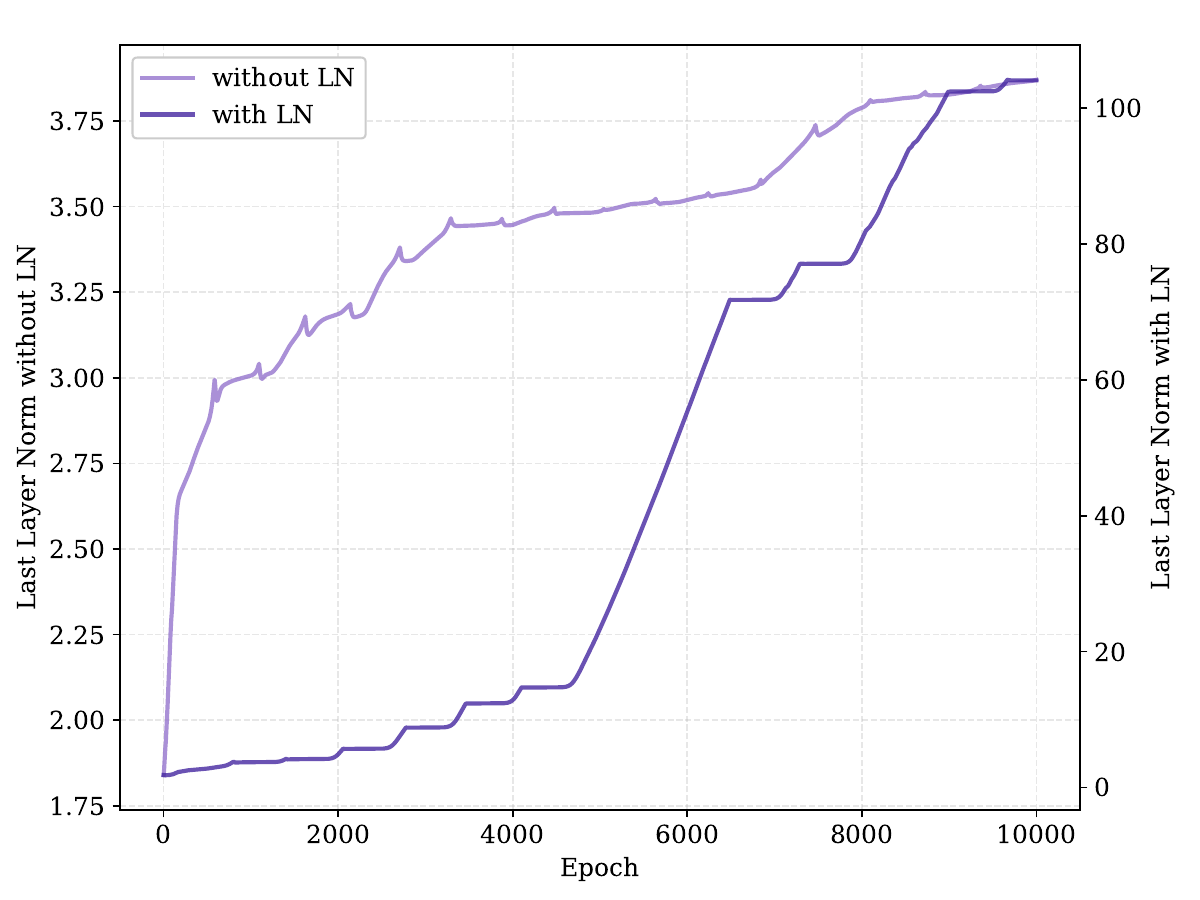}
        \caption{}
        \label{fig:ln}
    \end{subfigure}
    \hfill
    \begin{subfigure}[b]{0.325\textwidth}
        \centering
         \includegraphics[width=0.96\textwidth]{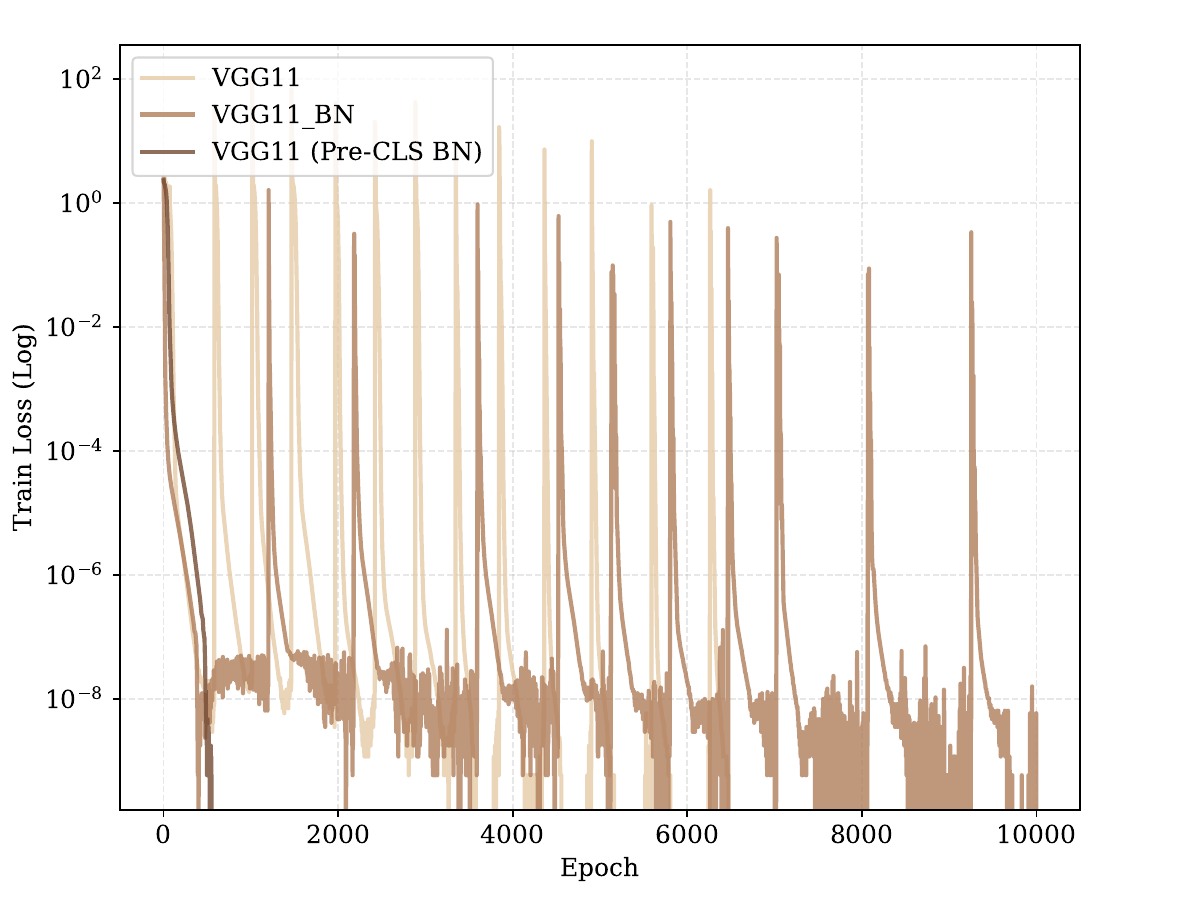}
         \caption{}
         \label{fig:bn}
     \end{subfigure}
        
     \caption{\textbf{Mitigation Study.}
     \textbf{(a) Adam's $\varepsilon$.} 
     Increasing the optimizer's $\varepsilon$ parameter mitigates instability: 
     while $\varepsilon=10^{-6}$ reduces the frequency of spikes, 
     setting $\varepsilon=10^{-5}$ completely eliminates them.
     \textbf{(b) Layer Norm.} 
     Applying LN changes the evolution of the last layer norm from a continuous trajectory 
     to a distinct stepwise pattern. Notably, LN significantly increases the magnitude of the last layer norm.
     \textbf{(c) Batch Norm.} Standard VGG11 with BN within convolutional layers still exhibits Slingshot spikes. 
     However, applying BN before the final classifier eliminates the instability, 
     stabilizing the training at the zero-loss SC state.}
     \label{fig:mitigate}
\end{figure} 

\subsection{Mitigation Study}
\label{sec:mitigation}

Following our identification of the $\mathcal{NFI}$ mechanism, 
we conducted a series of ablation studies to evaluate whether 
specific architectural or algorithmic modifications could mitigate the slingshot effect.
We omit discussion of explicit regularization terms like weight decay or z-loss (constraining the log softmax normalizer $Z$), 
given that their ability to mitigate Slingshot has been demonstrated in existing works~\cite{thilak2022,chowdhery2023,wortsman2024}.

\paragraph{Mixed Precision.}
See \cref{fig:sling}, upgrading the Softmax computation to \texttt{float64} precision is sufficient to eliminate the slingshot. The absorption error threshold for double precision is $\frac{|b|}{|a|} < 2^{-52} \approx 2.22 \times 10^{-16}$. Achieving a training loss magnitude low enough to trigger this threshold is generally infeasible within standard training durations.

\paragraph{Optimizer.} While the Slingshot mechanism is typically associated with adaptive optimizers, 
we demonstrate that the onset of the instability is independent of adaptivity and is driven solely by the magnitude of the step size. 
To verify this, we switched the optimizer from Adam to 
vanilla GD immediately after the training loss reached strict zero. 
We manually set the learning rate to $\eta = 10^5$ to emulate the massive effective step size 
($\approx \eta_{\text{Adam}}/\varepsilon_{\text{Adam}}$). 
Crucially, we observed that GD triggers a Slingshot spike at approximately the same epoch as the baseline Adam run. 
We note, however, a key distinction in the aftermath: 
while Adam is able to adapt its moments to recover from the spike (forming the periodic pattern), 
the high-learning-rate GD fails to reconverge, leading to permanent divergence.

\paragraph{Increasing $\varepsilon_{\text{Adam}}$.} 
Thilak et al.~\cite{thilak2022} reported that increasing the Adam optimizer's $\varepsilon$ parameter to $10^{-5}$ prevents slingshots. 
Our analysis confirms that this intervention works by lowering the upper bound of the effective learning rate ($\eta/\varepsilon_{\text{Adam}}$), 
thereby delaying the amplification of the re-emerging gradient signals.

\paragraph{Layer Normalization (LN).} 
Applying LN to features immediately before the classifier fails to mitigate the effect. 
Since LN normalizes samples individually, it does not prevent the collective alignment of features towards a common direction (i.e., $\mathcal{NFI}$ persists).
In fact, see \cref{fig:ln}, our experiments indicate that LN \textit{accelerates} the onset of slingshots. 
LN constrains the feature norm, 
and smaller feature norms mean that even minor reductions in angular separation can violate the absorption threshold.
Furthermore, LN decouples the dynamics of direction and magnitude. 
In our unnormalized MLP experiments, the last layer norm grows continuously, consistent with Soudry et al. \cite{soudry2018}. 
However, with LN, we observe the stepwise growth noted by Thilak et al.: 
the model alternates between optimizing angular alignment to reduce loss and, when angles stabilize, increasing the scalar scale of the last layer.

\paragraph{Batch Normalization (BN).} 
As shown in \cref{fig:bn}, BN applied directly before the classifier successfully eliminates slingshots. 
By centering the batch features (subtracting the global mean $\bm{\mu}_G$), BN removes the principal drift component.
We note that Thilak et al. observed that BN failed to prevent slingshots. 
This discrepancy arises from their architecture: they applied BN within the convolutional layers of VGG11, 
followed by another ReLU layers which re-introduced rank compression \cite{daneshmand2020}. 
We verify that BN is effective when applied immediately before the final linear layer.

\paragraph{Label Smoothing (LS).}
Label Smoothing sets the target probability to $y < 1$, preventing the infinite growth of correct-class logits. While this successfully eliminates the precision-induced $\mathcal{NFI}$ spikes, our experiments reveal that LS introduces a new class of instabilities independent of numerical precision.
When the target probability is not 1 (and incorrect targets are not 0), 
the theoretical guarantee by Ma et al. \cite{ma2022}—that Hessian eigenvalues vanish late in training—no longer holds. 
Instead, the maximum Hessian eigenvalue remains finite. 
This pushes the optimization into the EOS regime, 
where loss spikes arise from the intrinsic sharpness of the loss landscape (see \cref{fig:ls}). 
A detailed discussion distinguishing Slingshot from EOS is provided in Appendix \ref{app:eos}, 
and why LS leads to finite Hessian eigenvalues is elaborated in Appendix \ref{app:ls}.

\subsection{$\mathcal{NFI}$ in Real-World Tasks}

So far, we have mainly studied the full-batch setting. 
In real-world training, however, mini-batches are used for both efficiency and generalization. 
Mini-batch stochasticity makes it unlikely that all samples fall below the Softmax Collapse threshold at the same time, 
except in highly structured and nearly noiseless tasks such as modular arithmetic. 
Therefore, a visible Slingshot loss spike may not appear. Nevertheless, we find that $\mathcal{NFI}$ can still occur.

\paragraph{Mini batch.}
For VGG11 trained with batch size $256$, 
approximately half of samples still enter Softmax Collapse after $10^6$ steps training, 
with exactly zero loss. Samples that do not collapse keep nonzero correct-class gradients. 
These gradients continue to separate features in the classification subspace, 
keep Adam's moment estimates active, and prevent the effective learning rate from suddenly becoming too large. 
Hence, no Slingshot spike is observed. However, collapsed samples still induce the $\mathcal{NFI}$ drift, 
leading to rapid late-stage growth of $\bm{\mu}_G$ and the logits. 
Consistent with this interpretation, subtracting $\bm{W}_G$ or applying BatchNorm before the classifier, 
which removes $\bm{\mu}_G$ during training, substantially slows down the late-stage parameter growth.

\paragraph{Large Language Models.}
We also examine the language-modeling setting. Softmax Collapse also appears during GPT training: 
in each step with about $1.3 \times 10^5$ tokens, roughly $4000$ tokens have exactly zero loss. 
These tokens correspond to highly predictable contexts. On FineWeb, 
the top-10 tokens with the highest Softmax Collapse frequency are 
``.'', ``org'', ``example'', ``,'', ``to'', ``t'', ``last'', `` '', `` you'', and ``of''. 
Several of them, such as ``example'', ``org'', and ``last'', are not simply the most frequent tokens. 
Instead, they are highly predictable because of dataset-specific templates, 
such as ``example.org'' or form fields containing ``first name'' and ``last name''.

Long language-model training also shows abnormal logit growth. 
To test the effect of numerical precision, we replace the loss computation with \texttt{float64} and 
compare it with standard \texttt{bf16} training (calculate the loss in \texttt{float32}), the result is shown in Figure~\ref{fig:llm_fp}. 
Interestingly, unlike the MLP, CNN, and ViT experiments, 
higher precision does not reduce logit growth in this setting. 
After $10^5$ steps, the mean logit is $183$ under \texttt{float32} training, 
but increases to $498$ when the loss is computed in \texttt{float64}.

We believe this difference comes from the special structure of natural language data. 
In the previous classification tasks, classes are balanced. In language modeling, 
token frequencies follow Zipf's law, 
so the output embedding vectors $\bm{W}_k$ are inherently imbalanced toward frequent tokens. 
This imbalance can create a large global mean $\bm{W}_G$ even without numerical collapse~\cite{gao2019}. 
As a result, the features and output embeddings can reinforce each other in the same direction: 
frequent-token embeddings guide their corresponding features to grow along aligned directions. 
In contrast, $\mathcal{NFI}$ creates an anti-parallel interaction between $\bm{W}_G$ and $\bm{\mu}_G$. 
Therefore, in this setting, low precision can partially suppress 
the faster logit divergence caused by the frequency-induced $\bm{W}_G$.

In both cases, the key factor is the global mean component of the output embedding. 
Removing the last-layer mean $\bm{W}_G$ eliminates this mutual reinforcement and stabilizes the logits. 
We provide further discussion of logit divergence in language models in Appendix~\ref{app:llm}.


\section{Conclusion}
In this work, we have demystified the Slingshot Mechanism, 
revealing it to be artifacts of floating-point arithmetic—
specifically Numerical Feature Inflation ($\mathcal{NFI}$)—rather than intrinsic optimization dynamics. 
Our results reveal a gap between gradient-flow analysis and real optimization. 
Existing refinements, such as central flows~\cite{cohen2025} and neural thermodynamics~\cite{ziyin2025}, still mainly describe the intrinsic dynamics of the model, while numerical dynamics are often ignored. 
We further show that SC and $\mathcal{NFI}$ are not limited to toy models: they also appear in practical large-model training and can substantially affect the optimization trajectory. 
Thus, finite-precision loss computation should be treated as a first-order factor in the analysis of long-term training stability.

\paragraph{Limitations.} Our analysis of $\mathcal{NFI}$ primarily investigates the interaction between the penultimate layer features and the last-layer classifier. 
This perspective implicitly adopts the Unconstrained Feature Model assumption \cite{mixon2022}, 
positing that the backbone network is sufficiently expressive to generate arbitrary features 
and that gradients on features translate directly to feature displacement. 
While this assumption generally holds for deep neural networks, 
it may oversimplify the dynamics in shallow networks, 
where the class means $\bm{\mu}_k$ might be insufficient to characterize feature evolution. 
Indeed, Nanda et al.~\cite{nanda2023} reported difficulties in observing Slingshots in shallow networks.

\section*{Impact Statement}

This paper presents work whose goal is to advance the field of Machine
Learning. There are many potential societal consequences of our work, none
which we feel must be specifically highlighted here.

\begin{ack}
We thank Tsinghua University for access to the computing credits used in this work.
Hanqing Liu acknowledges an International Exchange Grant from the Marubun Research Promotion Foundation for travel to NeurIPS 2026.
\end{ack}

\small{
\bibliographystyle{unsrtnat}
\bibliography{my_paper}

@article{power2022,
  title={Grokking: Generalization beyond overfitting on small algorithmic datasets},
  author={Power, Alethea and Burda, Yuri and Edwards, Harri and Babuschkin, Igor and Misra, Vedant},
  journal={arXiv preprint arXiv:2201.02177},
  year={2022}
}

@inproceedings{tian2025,
  title={Provable Scaling Laws of Feature Emergence from Learning Dynamics of Grokking},
  author={Tian, Yuandong},
  booktitle={The Fourteenth International Conference on Learning Representations},
  year={2026}
}

@inproceedings{liu2023,
  author       = {Ziming Liu and
                  Eric J. Michaud and
                  Max Tegmark},
  title        = {Omnigrok: Grokking Beyond Algorithmic Data},
  booktitle    = {The Eleventh International Conference on Learning Representations},
  year         = {2023},
  bibsource    = {dblp computer science bibliography, https://dblp.org}
}

@inproceedings{ziyin2025,
title={Neural Thermodynamics: Entropic Forces in Deep and Universal Representation Learning},
author={Liu Ziyin and Yizhou Xu and Isaac L. Chuang},
booktitle={The Thirty-ninth Annual Conference on Neural Information Processing Systems},
year={2025}
}

@inproceedings{lyu2024,
title={Dichotomy of Early and Late Phase Implicit Biases Can Provably Induce Grokking},
author={Kaifeng Lyu and Jikai Jin and Zhiyuan Li and Simon Shaolei Du and Jason D. Lee and Wei Hu},
booktitle={The Twelfth International Conference on Learning Representations},
year={2024}
}

@article{thilak2022,
  title={The slingshot mechanism: An empirical study of adaptive optimizers and the grokking phenomenon},
  author={Thilak, Vimal and Littwin, Etai and Zhai, Shuangfei and Saremi, Omid and Paiss, Roni and Susskind, Joshua},
  journal={arXiv preprint arXiv:2206.04817},
  year={2022}
}

@article{gromov2023,
  title={Grokking modular arithmetic},
  author={Gromov, Andrey},
  journal={arXiv preprint arXiv:2301.02679},
  year={2023}
}

@inproceedings{kumar2024,
title={Grokking as the transition from lazy to rich training dynamics},
author={Tanishq Kumar and Blake Bordelon and Samuel J. Gershman and Cengiz Pehlevan},
booktitle={The Twelfth International Conference on Learning Representations},
year={2024}
}

@article{golechha2024,
  title={Progress Measures for Grokking on Real-world Tasks},
  author={Golechha, Satvik},
  journal={arXiv preprint arXiv:2405.12755},
  year={2024}
}

@article{prakash2025,
  title={Grokking and Generalization Collapse: Insights from $\texttt{HTSR}$ theory},
  author={Prakash, Hari K and Martin, Charles H},
  journal={arXiv preprint arXiv:2506.04434},
  year={2025}
}

@inproceedings{nanda2023,
title={Progress measures for grokking via mechanistic interpretability},
author={Neel Nanda and Lawrence Chan and Tom Lieberum and Jess Smith and Jacob Steinhardt},
booktitle={The Eleventh International Conference on Learning Representations },
year={2023}
}

@inproceedings{reddi2018,
  title={On the Convergence of Adam and Beyond},
  author={Reddi, Sashank J and Kale, Satyen and Kumar, Sanjiv},
  booktitle={International Conference on Learning Representations},
  year={2018}
}

@InProceedings{ma2022,
  title = 	 {A Qualitative Study of the Dynamic Behavior for Adaptive Gradient Algorithms},
  author =       {Ma, Chao and Wu, Lei and E, Weinan},
  booktitle = 	 {Proceedings of the 2nd Mathematical and Scientific Machine Learning Conference},
  pages = 	 {671--692},
  year = 	 {2022},
  volume = 	 {145},
  series = 	 {Proceedings of Machine Learning Research},
  publisher =    {PMLR}
}

@inproceedings{cohen2021,
title={Gradient Descent on Neural Networks Typically Occurs at the Edge of Stability},
author={Jeremy Cohen and Simran Kaur and Yuanzhi Li and J Zico Kolter and Ameet Talwalkar},
booktitle={International Conference on Learning Representations},
year={2021}
}

@inproceedings{cohen2023,
title={Adaptive Gradient Methods at the Edge of Stability},
author={Jeremy Cohen and Behrooz Ghorbani and Shankar Krishnan and Naman Agarwal and Sourabh Medapati and Michal Badura and Daniel Suo and Zachary Nado and George E. Dahl and Justin Gilmer},
booktitle={NeurIPS 2023 Workshop Heavy Tails in Machine Learning},
year={2023}
}

@inproceedings{cohen2025,
title={Understanding Optimization in Deep Learning with Central Flows},
author={Jeremy Cohen and Alex Damian and Ameet Talwalkar and J Zico Kolter and Jason D. Lee},
booktitle={The Thirteenth International Conference on Learning Representations},
year={2025}
}

@inproceedings{prieto2025,
title={Grokking at the Edge of Numerical Stability},
author={Lucas Prieto and Melih Barsbey and Pedro A. M. Mediano and Tolga Birdal},
booktitle={The Thirteenth International Conference on Learning Representations},
year={2025}
}

@inproceedings{kingma2015,
  author       = {Diederik P. Kingma and
                  Jimmy Ba},
  title        = {Adam: {A} Method for Stochastic Optimization},
  booktitle    = {The Third International Conference on Learning Representations},
  year         = {2015},
}

@article{bai2025,
  title={Adaptive Preconditioners Trigger Loss Spikes in {Adam}},
  author={Bai, Zhiwei and Zhou, Zhangchen and Zhao, Jiajie and Li, Xiaolong and Li, Zhiyu and Xiong, Feiyu and Yang, Hongkang and Zhang, Yaoyu and Xu, Zhi-Qin John},
  journal={arXiv preprint arXiv:2506.04805},
  year={2025}
}

@inproceedings{lyu2020,
title={Gradient Descent Maximizes the Margin of Homogeneous Neural Networks},
author={Kaifeng Lyu and Jian Li},
booktitle={International Conference on Learning Representations},
year={2020}
}

@article{soudry2018,
  title={The implicit bias of gradient descent on separable data},
  author={Soudry, Daniel and Hoffer, Elad and Nacson, Mor Shpigel and Gunasekar, Suriya and Srebro, Nathan},
  journal={Journal of Machine Learning Research},
  volume={19},
  number={70},
  pages={1--57},
  year={2018}
}

@article{papyan2020,
  title={Prevalence of neural collapse during the terminal phase of deep learning training},
  author={Papyan, Vardan and Han, XY and Donoho, David L},
  journal={Proceedings of the National Academy of Sciences},
  volume={117},
  number={40},
  pages={24652--24663},
  year={2020b},
  publisher={National Academy of Sciences}
}

@inproceedings{boursier2025,
title={A Theoretical Framework for Grokking: Interpolation followed by Riemannian Norm Minimisation},
author={Etienne Boursier and Scott Pesme and Radu-Alexandru Dragomir},
booktitle={The Thirty-ninth Annual Conference on Neural Information Processing Systems},
year={2025},
}

@inproceedings{xu2025,
title={Let Me Grok for You: Accelerating Grokking via Embedding Transfer from a Weaker Model},
author={Zhiwei Xu and Zhiyu Ni and Yixin Wang and Wei Hu},
booktitle={The Thirteenth International Conference on Learning Representations},
year={2025}
}

@inproceedings{garrod2025,
title={The Persistence of Neural Collapse Despite Low-Rank Bias},
author={Connall Garrod and Jonathan P. Keating},
booktitle={The Thirty-ninth Annual Conference on Neural Information Processing Systems},
year={2025}
}

@ARTICLE{IEEE754-2019,
  author={},
  journal={IEEE Std 754-2019 (Revision of IEEE 754-2008)}, 
  title={IEEE Standard for Floating-Point Arithmetic}, 
  year={2019},
  volume={},
  number={},
  pages={1-84}
}

@article{paszke2019,
  title={{PyTorch}: An imperative style, high-performance deep learning library},
  author={Paszke, Adam and Gross, Sam and Massa, Francisco and Lerer, Adam and Bradbury, James and Chanan, Gregory and Killeen, Trevor and Lin, Zeming and Gimelshein, Natalia and Antiga, Luca and others},
  journal={Advances in neural information processing systems},
  volume={32},
  year={2019}
}

@article{lu2022,
title = {Neural collapse under cross-entropy loss},
journal = {Applied and Computational Harmonic Analysis},
volume = {59},
pages = {224-241},
year = {2022},
note = {Special Issue on Harmonic Analysis and Machine Learning},
issn = {1063-5203},
author = {Jianfeng Lu and Stefan Steinerberger}
}

@article{sakamoto2025,
  title={Explaining Grokking and Information Bottleneck through Neural Collapse Emergence},
  author={Sakamoto, Keitaro and Sato, Issei},
  journal={arXiv preprint arXiv:2509.20829},
  year={2025}
}

@inproceedings{han2025,
title={Flatness is Necessary, Neural Collapse is Not: Rethinking Generalization via Grokking},
author={Ting Han and Linara Adilova and Henning Petzka and Jens Kleesiek and Michael Kamp},
booktitle={The Thirty-ninth Annual Conference on Neural Information Processing Systems},
year={2025},
}

@inproceedings{dang2024,
title={Neural Collapse for Cross-entropy Class-Imbalanced Learning with Unconstrained Re{LU} Features Model},
author={Hien Dang and Tho Tran Huu and Tan Minh Nguyen and Nhat Ho},
booktitle={Forty-first International Conference on Machine Learning},
year={2024},
}

@inproceedings{zhang2022,
title={Adam Can Converge Without Any Modification On Update Rules},
author={Yushun Zhang and Congliang Chen and Naichen Shi and Ruoyu Sun and Zhi-Quan Luo},
booktitle={Advances in Neural Information Processing Systems},
year={2022}
}

@article{daneshmand2020,
  title={Batch normalization provably avoids ranks collapse for randomly initialised deep networks},
  author={Daneshmand, Hadi and Kohler, Jonas and Bach, Francis and Hofmann, Thomas and Lucchi, Aurelien},
  journal={Advances in Neural Information Processing Systems},
  volume={33},
  pages={18387--18398},
  year={2020}
}

@inproceedings{xiao2024,
title={Efficient Streaming Language Models with Attention Sinks},
author={Guangxuan Xiao and Yuandong Tian and Beidi Chen and Song Han and Mike Lewis},
booktitle={The Twelfth International Conference on Learning Representations},
year={2024}
}

@article{mixon2022,
  title={Neural collapse with unconstrained features},
  author={Mixon, Dustin G and Parshall, Hans and Pi, Jianzong},
  journal={Sampling Theory, Signal Processing, and Data Analysis},
  volume={20},
  number={2},
  pages={11},
  year={2022},
  publisher={Springer}
}

@article{papyan2020a,
  author  = {Vardan Papyan},
  title   = {Traces of Class/Cross-Class Structure Pervade Deep Learning Spectra},
  journal = {Journal of Machine Learning Research},
  year    = {2020a},
  volume  = {21},
  number  = {252},
  pages   = {1--64}
}

@inproceedings{xie2024,
title={Implicit Bias of AdamW: $\ell_{\infty}$-Norm Constrained Optimization},
author={Shuo Xie and Zhiyuan Li},
booktitle={Forty-first International Conference on Machine Learning},
year={2024}
}

@article{chowdhery2023,
  title={Palm: Scaling language modeling with pathways},
  author={Chowdhery, Aakanksha and Narang, Sharan and Devlin, Jacob and Bosma, Maarten and Mishra, Gaurav and Roberts, Adam and Barham, Paul and Chung, Hyung Won and Sutton, Charles and Gehrmann, Sebastian and others},
  journal={Journal of Machine Learning Research},
  volume={24},
  number={240},
  pages={1--113},
  year={2023}
}

@inproceedings{shazeer2018,
  title={Adafactor: Adaptive learning rates with sublinear memory cost},
  author={Shazeer, Noam and Stern, Mitchell},
  booktitle={International Conference on Machine Learning},
  pages={4596--4604},
  year={2018},
  organization={PMLR}
}

@inproceedings{wortsman2024,
title={Small-scale proxies for large-scale Transformer training instabilities},
author={Mitchell Wortsman and Peter J Liu and Lechao Xiao and Katie E Everett and Alexander A Alemi and Ben Adlam and John D Co-Reyes and Izzeddin Gur and Abhishek Kumar and Roman Novak and Jeffrey Pennington and Jascha Sohl-Dickstein and Kelvin Xu and Jaehoon Lee and Justin Gilmer and Simon Kornblith},
booktitle={The Twelfth International Conference on Learning Representations},
year={2024}
}

@inproceedings{wortsman2023,
title={Stable and low-precision training for large-scale vision-language models},
author={Mitchell Wortsman and Tim Dettmers and Luke Zettlemoyer and Ari S. Morcos and Ali Farhadi and Ludwig Schmidt},
booktitle={Thirty-seventh Conference on Neural Information Processing Systems},
year={2023}
}

@inproceedings{qiu2025,
  title={Why Low-Precision Transformer Training Fails: An Analysis on Flash Attention},
  author={Qiu, Haiquan and Yao, Quanming},
  booktitle={The Fourteenth International Conference on Learning Representations},
  year={2026},
}

@article{molybog2023,
  title={A theory on {Adam} instability in large-scale machine learning},
  author={Molybog, Igor and Albert, Peter and Chen, Moya and DeVito, Zachary and Esiobu, David and Goyal, Naman and Koura, Punit Singh and Narang, Sharan and Poulton, Andrew and Silva, Ruan and others},
  journal={arXiv preprint arXiv:2304.09871},
  year={2023}
}

@article{vaswani2017,
  title={Attention is all you need},
  author={Vaswani, Ashish and Shazeer, Noam and Parmar, Niki and Uszkoreit, Jakob and Jones, Llion and Gomez, Aidan N and Kaiser, {\L}ukasz and Polosukhin, Illia},
  journal={Advances in neural information processing systems},
  volume={30},
  year={2017}
}

@inproceedings{he2016,
  title={Deep residual learning for image recognition},
  author={He, Kaiming and Zhang, Xiangyu and Ren, Shaoqing and Sun, Jian},
  booktitle={Proceedings of the IEEE conference on computer vision and pattern recognition},
  pages={770--778},
  year={2016}
}

@inproceedings{simonyan2015,
  author       = {Karen Simonyan and
                  Andrew Zisserman},
  editor       = {Yoshua Bengio and
                  Yann LeCun},
  title        = {Very Deep Convolutional Networks for Large-Scale Image Recognition},
  booktitle    = {3rd International Conference on Learning Representations},
  year         = {2015}
}

@inproceedings{dosovitskiy2021,
  author       = {Alexey Dosovitskiy and
                  Lucas Beyer and
                  Alexander Kolesnikov and
                  Dirk Weissenborn and
                  Xiaohua Zhai and
                  Thomas Unterthiner and
                  Mostafa Dehghani and
                  Matthias Minderer and
                  Georg Heigold and
                  Sylvain Gelly and
                  Jakob Uszkoreit and
                  Neil Houlsby},
  title        = {An Image is Worth 16x16 Words: Transformers for Image Recognition
                  at Scale},
  booktitle    = {9th International Conference on Learning Representations},
  year         = {2021}
}

@inproceedings{keskar2017,
title={On Large-Batch Training for Deep Learning: Generalization Gap and Sharp Minima},
author={Nitish Shirish Keskar and Dheevatsa Mudigere and Jorge Nocedal and Mikhail Smelyanskiy and Ping Tak Peter Tang},
booktitle={International Conference on Learning Representations},
year={2017}
}

@inproceedings{smith2018,
title={A Bayesian Perspective on Generalization and Stochastic Gradient Descent},
author={Samuel L. Smith and Quoc V. Le},
booktitle={International Conference on Learning Representations},
year={2018}
}

@inproceedings{orvieto2026,
title={In Search of {A}dam{\textquoteright}s Secret Sauce},
author={Antonio Orvieto and Robert M. Gower},
booktitle={The Thirty-ninth Annual Conference on Neural Information Processing Systems},
year={2026}
}

@article{stollenwerk2026,
  title={Output Embedding Centering for Stable LLM Pretraining},
  author={Stollenwerk, Felix and Lokrantz, Anna and Hertzberg, Niclas},
  journal={arXiv preprint arXiv:2601.02031},
  year={2026}
}

@inproceedings{gao2019,
title={Representation Degeneration Problem in Training Natural Language Generation Models},
author={Jun Gao and Di He and Xu Tan and Tao Qin and Liwei Wang and Tieyan Liu},
booktitle={International Conference on Learning Representations},
year={2019}
}

@article{thilak2024,
title={The Slingshot Effect: A Late-Stage Optimization Anomaly in Adaptive Gradient Methods},
author={Vimal Thilak and Etai Littwin and Shuangfei Zhai and Omid Saremi and Roni Paiss and Joshua M. Susskind},
journal={Transactions on Machine Learning Research},
issn={2835-8856},
year={2024}
}
}
\newpage
\appendix
{
\hypersetup{linkcolor=black}

\addtocontents{toc}{\protect\setcounter{tocdepth}{2}}

\etocsettocstyle{\hrule height 0.5pt \vspace{1em} \section*{Appendix Contents}}{}

\tableofcontents

\vspace{1em} \hrule height 0.5pt \vspace{2em} 
}
\section{Additional Discussion}

\subsection{More Mitigation Experiments}
\label{app:mitigation}

\subsubsection{Dataset Size}
To rigorously evaluate whether the Slingshot mechanism is an artifact confined to the small-data regime, we conducted experiments on CIFAR-10 subsets of varying magnitudes: $N \in \{200, 2000, 20000, 50000\}$, where $N=50000$ represents the full training set. 
To disentangle the effect of dataset size from the noise induced by stochastic sampling (as discussed in Appendix \ref{app:minibatch}), we employed \emph{Full Batch Adam} for all configurations.

We observed distinct Slingshot loss spikes across all dataset sizes, including the full CIFAR-10 set. 
This empirical evidence demonstrates that the occurrence of Slingshots is independent of the dataset size. 
It confirms that as long as the model possesses sufficient capacity to drive the residual probability mass $\epsilon$ below the floating-point absorption threshold, the $\mathcal{NFI}$ feedback loop will be triggered, provided that the optimization trajectory is not continuously disrupted by mini-batch noise.

\subsubsection{Learning Rate}
We conducted a sensitivity analysis by sweeping the global learning rate $\eta$ across the range $\{10^{-2}, 10^{-3}, 10^{-4}, 10^{-5}\}$. 
Counter-intuitively, we observe that \emph{smaller} learning rates are significantly more prone to triggering Slingshot instabilities.

This phenomenon implies that a larger learning rate provides a crucial form of implicit regularization. 
Large step sizes introduce stochastic noise that prevents the optimizer from settling into sharp, rugged basins of attraction, effectively smoothing the optimization trajectory. 
In contrast, a small learning rate allows the model to converge faithfully to the nearest local minimum. 
In the complex loss landscape of deep networks, these ``nearest" solutions are frequently ``bad" local minima characterized by extreme sharpness.

\subsubsection{Bias}
We find that including a bias term in the classification layer accelerates the occurrence of Slingshots. 
This instability stems from a significant scale discrepancy between the updates of weights and biases.
The gradient with respect to the weight $\boldsymbol{W}_k$ is scaled by the feature vector $\boldsymbol{h}$:
\begin{equation} 
    \nabla_{\boldsymbol{W}_k} \mathcal{L} = (\hat{y}_k - y_k)\boldsymbol{h} 
\end{equation}
In contrast, the gradient with respect to the bias $b_k$ is:
\begin{equation} 
    \nabla{b_k} \mathcal{L} = (\hat{y}_k - y_k) 
\end{equation}
In our experiments, the feature norm $\|\boldsymbol{h}\|$ grows to approximately $10^2$. 
This implies that the weight matrix $\boldsymbol{W}$ updates two orders of magnitude faster than the bias $b$. 
As a result, the bias term effectively ``lags behind'' the rapidly evolving weights. 
This dynamic mismatch destabilizes the logit computation $z_k = \boldsymbol{W}_k^T \boldsymbol{h} + b_k$, 
making the system more prone to violating the absorption threshold and triggering spikes.

\begin{figure}[t!]
  \centering
     \begin{subfigure}[b]{0.49\textwidth}
         \centering
         \includegraphics[width=\textwidth]{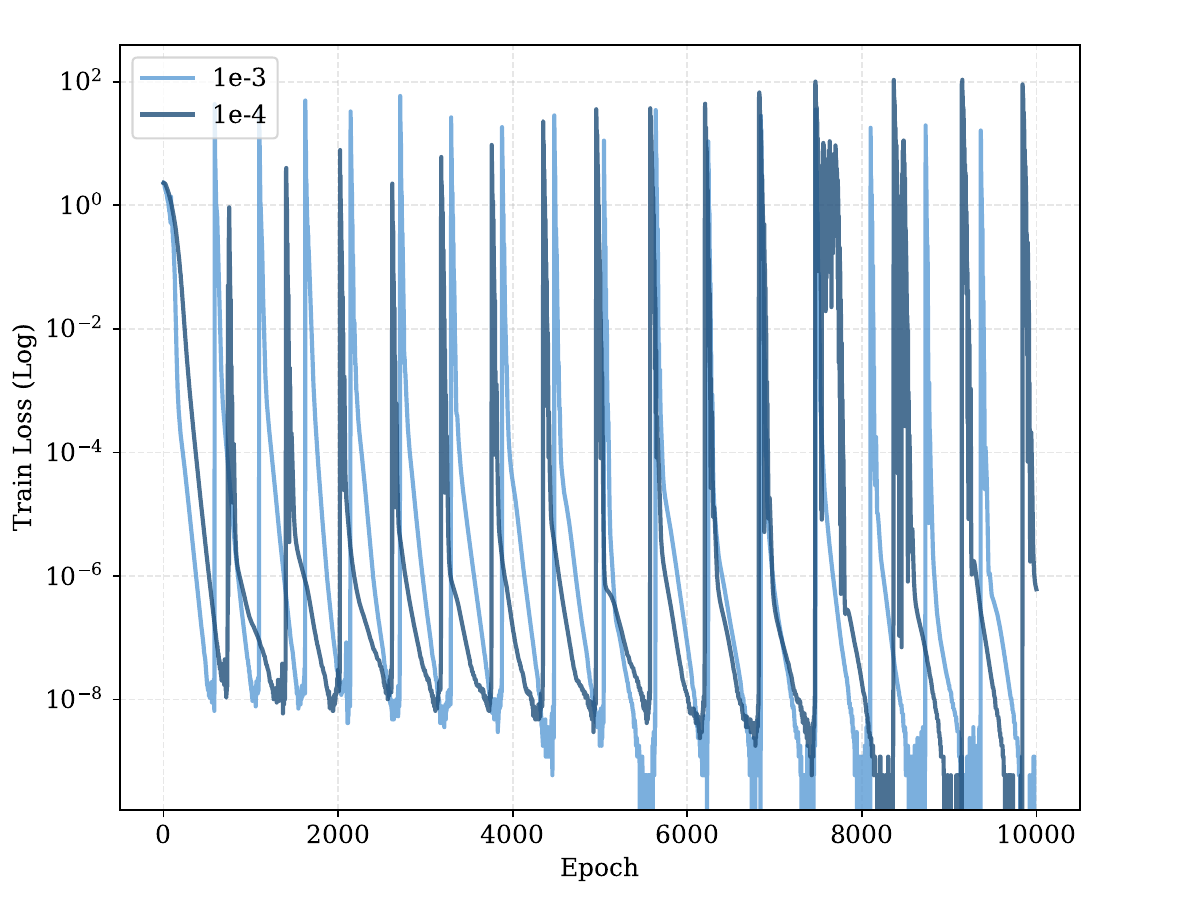}
         \caption{}
         \label{fig:lr}
     \end{subfigure}
     \hfill 
     \begin{subfigure}[b]{0.49\textwidth}
         \centering
        \includegraphics[width=\textwidth]{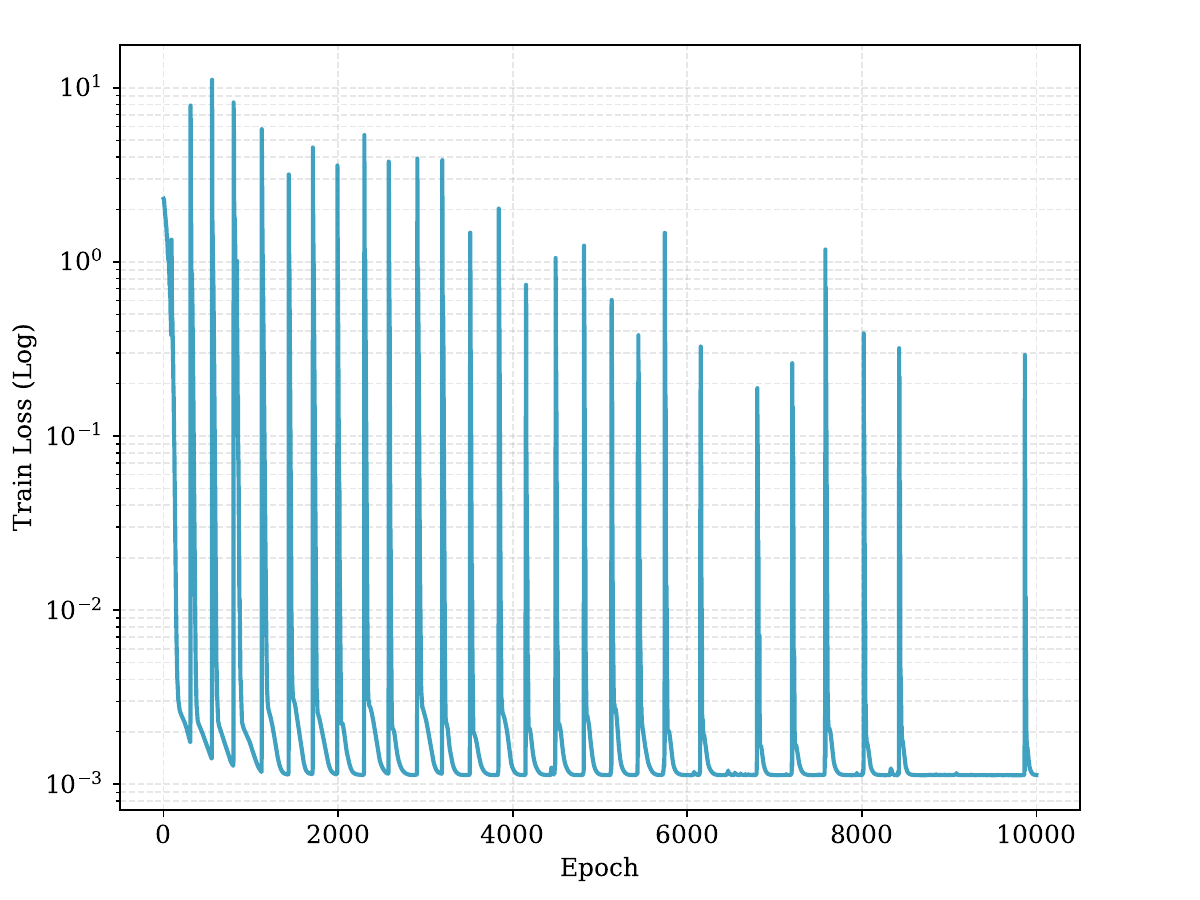}
        \caption{}
        \label{fig:ls}
    \end{subfigure}
    \caption{
    \textbf{(a)} Training Loss with different learning rates.
    \textbf{(b)} Train Loss with label smoothing.
    }
    \label{fig:add}
\end{figure}

\subsubsection{Numerical Underflow}
We demonstrate that the Slingshot mechanism is unrelated to numerical underflow (often referred to as precision overflow or range truncation).

Floating-point arithmetic is subject to two distinct types of precision loss: 
\textit{absorption error} (due to limited mantissa bits) and \textit{underflow} (due to limited exponent bits). 
Underflow occurs when a value is smaller than the minimum representable positive number, causing it to be truncated to zero. 
Specifically, for the Softmax function, the gradient with respect to the minimum logit $z_{min}$ vanishes if:
\begin{equation}
\exp(z_{min} - z_{max}) < 2^{-(p - 1 + 2^{E-1} - 2)}
\end{equation}
where $p$ denotes the number of mantissa bits and $E$ the number of exponent bits. 
For standard \texttt{float32} precision, this cutoff occurs when the logit difference satisfies $z_{max} - z_{min} > 149 \ln 2 \approx 103.28$.

To rule out underflow as a cause, 
we conducted an experiment where we explicitly clamped the logits such that the margin satisfies $z_{min} \ge z_{max} - 100$ at all times. 
This constraint guarantees that gradients remain strictly within the representable range of \texttt{float32}. 
We observed that Slingshot spikes persist despite this intervention, confirming that the phenomenon is driven purely by absorption errors in the mantissa, independent of exponent-based underflow.

\subsection{Is Slingshot an Edge of Stability Phenomenon?}
\label{app:eos}

While both the Slingshot mechanism and the Edge of Stability (EOS) phenomenon manifest as frequent oscillations and spikes in the training loss trajectory, 
our theoretical analysis and empirical evidence demonstrate that they are governed by fundamentally different mechanisms.

In the context of convex optimization, 
the condition for Gradient Descent (GD) to converge stably on a function $f$ is that the learning rate must satisfy $\eta < 2/\lambda$, 
where $\lambda$ represents the Lipschitz constant of the gradient $\nabla f$. 
For twice-differentiable functions, $\lambda$ corresponds to the upper bound of the maximum eigenvalue (spectral norm) of the Hessian matrix $\nabla^2 f(x)$.

Cohen et al.~\cite{cohen2021} identified the phenomenon of ``Progressive Sharpening'' in neural network training, 
where the maximum eigenvalue of the Hessian, $\lambda_{max}$, 
steadily increases until it reaches the stability threshold $2/\eta$. 
Upon breaching this limit, the optimization enters an unstable regime characterized by oscillations, 
which effectively regulate $\lambda_{max}$ to hover around the critical value $2/\eta$. 
This self-stabilizing behavior near the limit of divergence is termed the ``Edge of Stability''. 
The visual signature of EOS in the training loss curve is typically a series of high-frequency oscillations.

For adaptive optimizers, the stability condition extends to the maximum eigenvalue of the \textit{Preconditioned Hessian} $\mathbf{P}\mathbf{H}(w)$, 
where $\mathbf{P} = \text{diag}(1/(\sqrt{v} + \epsilon))$ acts as the preconditioner \cite{cohen2023, cohen2025}. 
In the presence of momentum (e.g., in Adam), this stability threshold is further modified by 
a coefficient $\kappa(\beta_1, \beta_2)$ dependent on the hyperparameters \cite{cohen2023, bai2025}.

Although Slingshot events also manifest as loss spikes, their underlying principle is distinct from EOS for several reasons:

\begin{figure}[t!]
  \centering
     \begin{subfigure}[b]{0.49\textwidth}
         \centering
         \includegraphics[width=\textwidth]{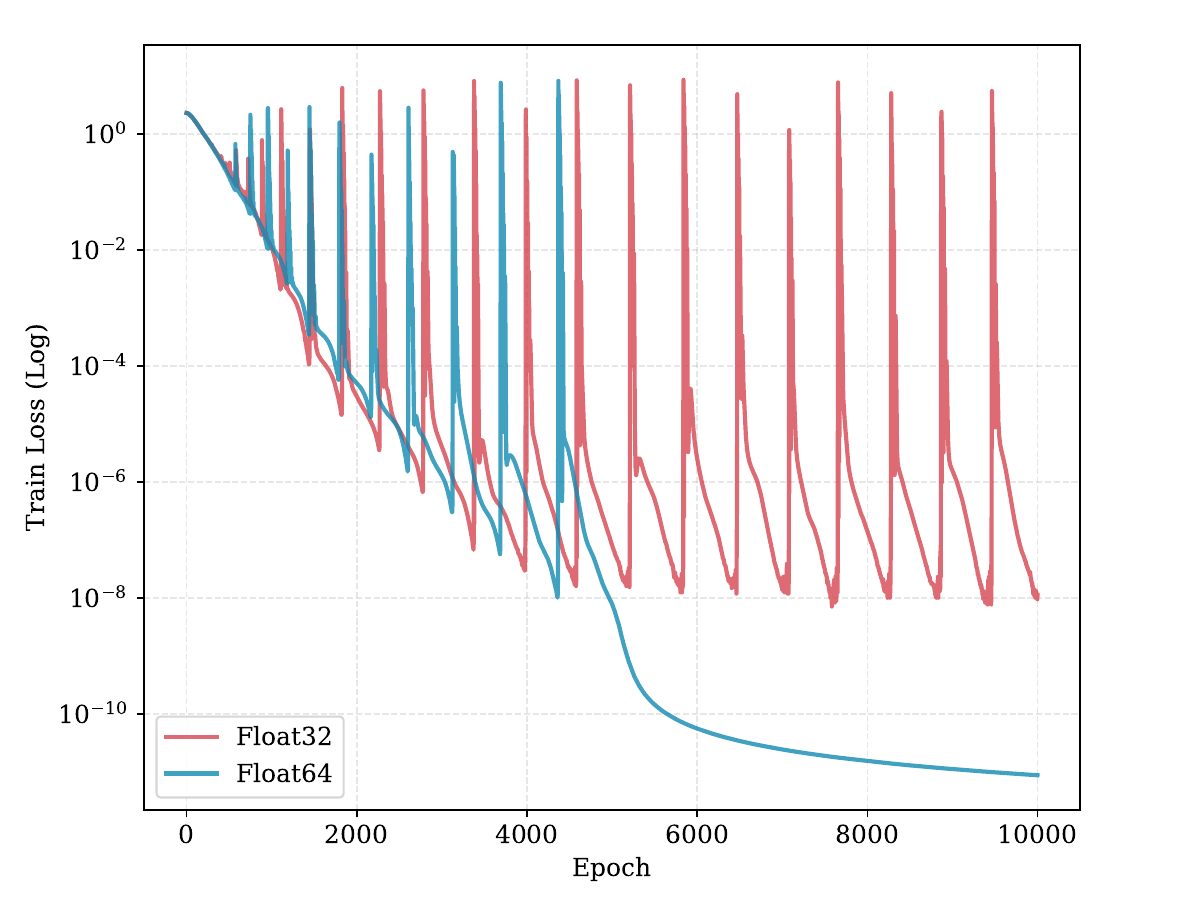}
         \caption{}
         \label{fig:eos_32}
     \end{subfigure}
     \hfill 
     \begin{subfigure}[b]{0.49\textwidth}
         \centering
        \includegraphics[width=\textwidth]{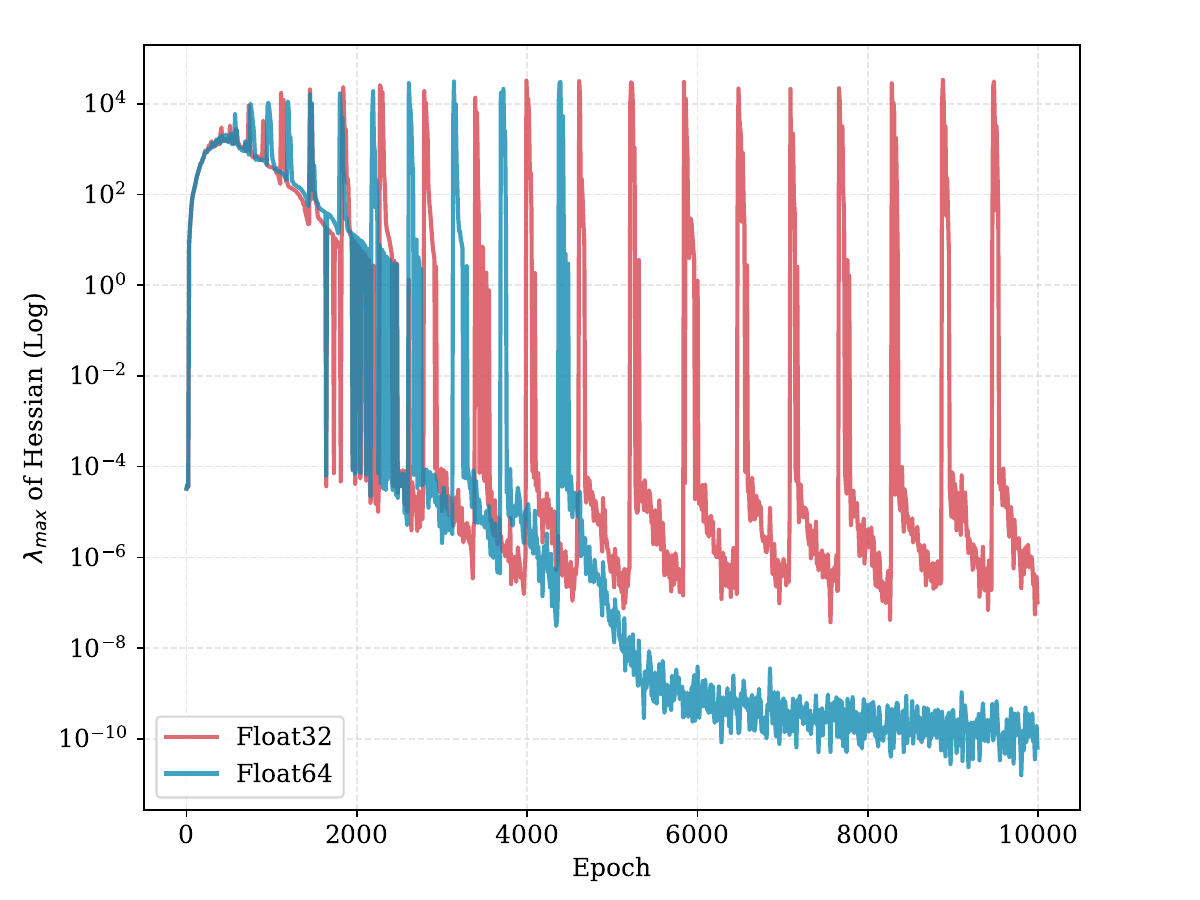}
        \caption{}
        \label{fig:eos_64}
    \end{subfigure}
    \caption{\textbf{Coexistence of EOS and Slingshot Phenomena.}
    \textbf{(a)} Training Loss showing early EOS oscillations versus late-stage numerical spikes.
    \textbf{(b)} Evolution of Maximum Hessian Eigenvalue $\lambda_{max}$.
    }
    \label{fig:eos}
\end{figure}

\subsubsection{Vanishing Hessian under Cross-Entropy.} 
First, in the specific setting of CE loss, the progressive sharpening required for EOS does not occur in the late stages of training. 
Instead, $\lambda_{max}$ tends toward zero.

\begin{theorem}
\label{thm:vanishing_hessian}
Consider a neural network with parameters $\theta \in \mathbb{R}^d$ and output logits $z(\theta, x) \in \mathbb{R}^K$. 
For a classification task with Cross-Entropy loss $\mathcal{L}$, 
if the model converges to an interpolation solution (i.e., the predicted probability vector $\hat{y}$ converges to the one-hot label $y$), 
then the maximum eigenvalue of the Hessian matrix $\lambda_{\max}(H_\theta)$ converges to 0, 
provided that the derivatives of the network are locally bounded.
\end{theorem}

The proof is provided in \cref{app:proof_lambda}. 
Consequently, the optimization operates far below the stability threshold $2/\eta$ during the late training phase, 
implying that the instability does not arise from the curvature of the loss landscape. 

\subsubsection{Numerical Artifact vs. Intrinsic Property.} 
Second, EOS is an intrinsic property of the optimization landscape that persists regardless of numerical precision. 
In stark contrast, the Slingshot effect is a numerical artifact. 
As demonstrated in our main results, simply increasing the precision to \texttt{float64} eliminates Slingshot spikes entirely, 
whereas true EOS oscillations would persist in double precision.

\subsubsection{Coexistence of Phenomena.}
Finally, we note that these two phenomena can coexist within the same training run. 
We verified this by training a narrow network (hidden dimension $d=20$). As illustrated in \cref{fig:eos}, 
the model trained in \texttt{float32} exhibits loss spikes in both the early and late stages of training. 
However, when trained in \texttt{float64}, the late-stage spikes (Slingshots) disappear, 
while the early-stage oscillations—occurring while $\lambda_{max}$ is still finite and large—persist. 
This confirms that the early instability is genuine EOS, while the late-stage instability is strictly a result of numerical breakdown.

\subsection{Why are Slingshots Rare in Practical Settings?}
\label{app:slingshot_rarity}
The Slingshot mechanism has remained largely underexplored primarily because it rarely manifests under standard deep learning training configurations. We identify two dominant factors that naturally suppress this phenomenon in conventional settings:

\subsubsection{Weight Decay.} 
Explicit regularization effectively constrains the magnitude of weights and, by extension, the growth of output logits. 
Xie \& Li \cite{xie2024} theoretically demonstrated that for the AdamW optimizer with a weight decay factor $\lambda$, 
the parameter magnitude is bounded by $\|w\|_\infty \lesssim 1/\lambda$. Since the logits are kept within a finite, 
bounded range, the logit margin $z_m - \max_{k \neq m} z_k$ rarely exceeds the critical threshold required to trigger floating-point absorption errors (approx. 16 for \texttt{float32}). 
Without this numerical breakdown, the feedback loop cannot initiate.

\subsubsection{Mini-Batch}
\label{app:minibatch}

While we have confirmed that Slingshot events can occur on the full CIFAR-10 dataset, 
they are often more difficult to observe in standard training regimes due to the dynamics of mini-batch. We identify two primary mechanisms through which mini-batching suppresses or delays the onset of $\mathcal{NFI}$:

\emph{Stochastic Delay of Convergence.}
The onset of the $\mathcal{NFI}$ feedback loop requires the model to enter a ``silent'' regime where the residual probability mass $\epsilon$ drops below the floating-point absorption threshold (approx. $10^{-7}$ for \texttt{float32}). In full-batch gradient descent, the loss decreases monotonically and rapidly, allowing the model to hit this numerical floor quickly. 
However, the stochastic noise introduced by mini-batch sampling creates fluctuations in the optimization trajectory. These fluctuations significantly impede the precise, asymptotic convergence required to trigger absorption errors. 

\emph{Implicit Regularization towards Flat Minima.}
More fundamentally, mini-batch SGD creates an implicit regularization effect that biases the optimization towards flatter minima~\cite{keskar2017, smith2018}. 
Therefore, Slingshots are most prominent when the batch size is large enough (approximating full-batch dynamics).
\subsubsection{Implications for Low-Precision Training.}
While Slingshot events are effectively mitigated in standard \texttt{float32} training, 
our findings suggest they may pose a significant latent risk for modern low-precision paradigms. 
As the field moves toward \texttt{BF16}, \texttt{FP8}, and even \texttt{FP4} training for Large Language Models, 
the margin required to trigger absorption errors shrinks dramatically. 
We posit that loss spikes induced by $\mathcal{NFI}$ could become a critical source of instability in these aggressive quantization regimes.

\subsection{Discussion on output logit divergence in LLMs.}
\label{app:llm}
Abnormally rapid logit growth after long-term LLM training was first observed in PaLM~\cite{chowdhery2023} 
and was later termed ``output logit divergence'' by Wortsman et al.~\cite{wortsman2024}. 
In their experiments, the average logit rapidly drifts toward negative infinity, 
which appears qualitatively similar to the drift dynamics observed in our setting (see \cref{fig:drift}). 
In our setting, the origin of this drift is explicit: as $\bm{W}_G$ and $\bm{\mu}_G$ become anti-parallel and grow exponentially, every logit contains a large common component
\begin{equation}
z_k = \bm{W}_k^\top \bm{h} = (\bm{W}_G + \bm{W}_k^*)^\top(\bm{\mu}_G + \bm{h}^*) 
= \bm{W}_G^\top \bm{\mu}_G + \bm{W}_G^\top \bm{h}^* + (\bm{W}_k^*)^\top \bm{\mu}_G + (\bm{W}_k^*)^\top \bm{h}^* .
\end{equation}
The dominant term $\bm{W}_G^\top \bm{\mu}_G$ tends to $-\infty$, thereby driving all logits downward.
Both Wortsman et al. and Thilak et al.~\cite{thilak2024} noted possible similarities between these phenomena, 
and we initially considered that they might share a related origin.
However, when we attempted to reproduce this phenomenon, we did not observe the same behavior. 
Following the model architectures and hyperparameters described in their paper, 
we trained GPT-2 models with 19M and 150M parameters. We first implemented the experiments in PyTorch, 
but we could not reproduce the divergence of logits toward negative infinity reported in 
Figure~4 of the arXiv version of their paper (Figure~G.2 in the proceedings version). 
We then used NanoDo\footnote{\url{https://github.com/google-deepmind/nanodo}}, 
the JAX/Flax-based small GPT-2 implementation that their paper states it is based on. 
After modifying the configuration to match the paper description as closely as possible, 
we still did not observe the reported phenomenon.

After further checking related materials, we found that, despite the paper being an ICLR Oral, 
we are not aware of any independent work that has successfully reproduced the same phenomenon in the three years since its publication. 
Stollenwerk et al.~\cite{stollenwerk2026} report a successful reproduction, 
but when we ran their released source code, the observed behavior was not exactly the same. 
In addition, we noticed that only Google-related open-source models, such as PaLM and Gemma, 
discuss how to mitigate output logit divergence in their technical reports, while other open-source models do not mention this issue. 
This is surprising, because Wortsman et al. claim that the phenomenon can simply be reproduced in 
a model with only 2.4M parameters by increasing the learning rate to $0.1$.

Therefore, we suspect that this phenomenon may be specific to running JAX/Flax on TPUs, 
or that there may be additional implementation details that were not described in the paper. 
Unlike PyTorch, which provides relatively mature precision protection through autocast, 
mixed-precision training in Flax often requires manual casting when implementing neural network modules. 
This may introduce additional precision-related issues.

The phenomenon we observe in LLMs is instead consistent with Stollenwerk et al.: 
the output logits diverge toward positive infinity.

\section{Additional Results}
\label{app:results}

\begin{table}[h]
    \centering
    \caption{Summary of Loss Spike observations across different model architectures and datasets.}
    \label{tab:slingshot_summary}
    \begin{tabular}{lccl}
        \toprule
        \textbf{Model} & \textbf{Dataset} & \textbf{Loss Spike?} & \textbf{Figure} \\
        \midrule
        Transformer & Modular Division & Yes & \cref{fig:mod_trm} \\
        MLP & Modular Division & Yes & \cref{fig:mod_mlp} \\
        MLP & CIFAR-10 & Yes & \cref{fig:mlp} \\
        VGG11 & CIFAR-10 & Yes & \cref{fig:vgg11} \\
        VGG11 with BN & CIFAR-10 & Yes & \cref{fig:vgg11_bn} \\
        ViT & CIFAR-10 & Yes & \cref{fig:vit} \\
        ResNet18 & CIFAR-10 & \textbf{No} & \cref{fig:resnet18} \\
        \bottomrule
    \end{tabular}
\end{table}

\begin{figure}[h]
     \centering
     \begin{subfigure}[b]{0.45\textwidth}
         \centering
         \includegraphics[width=\textwidth]{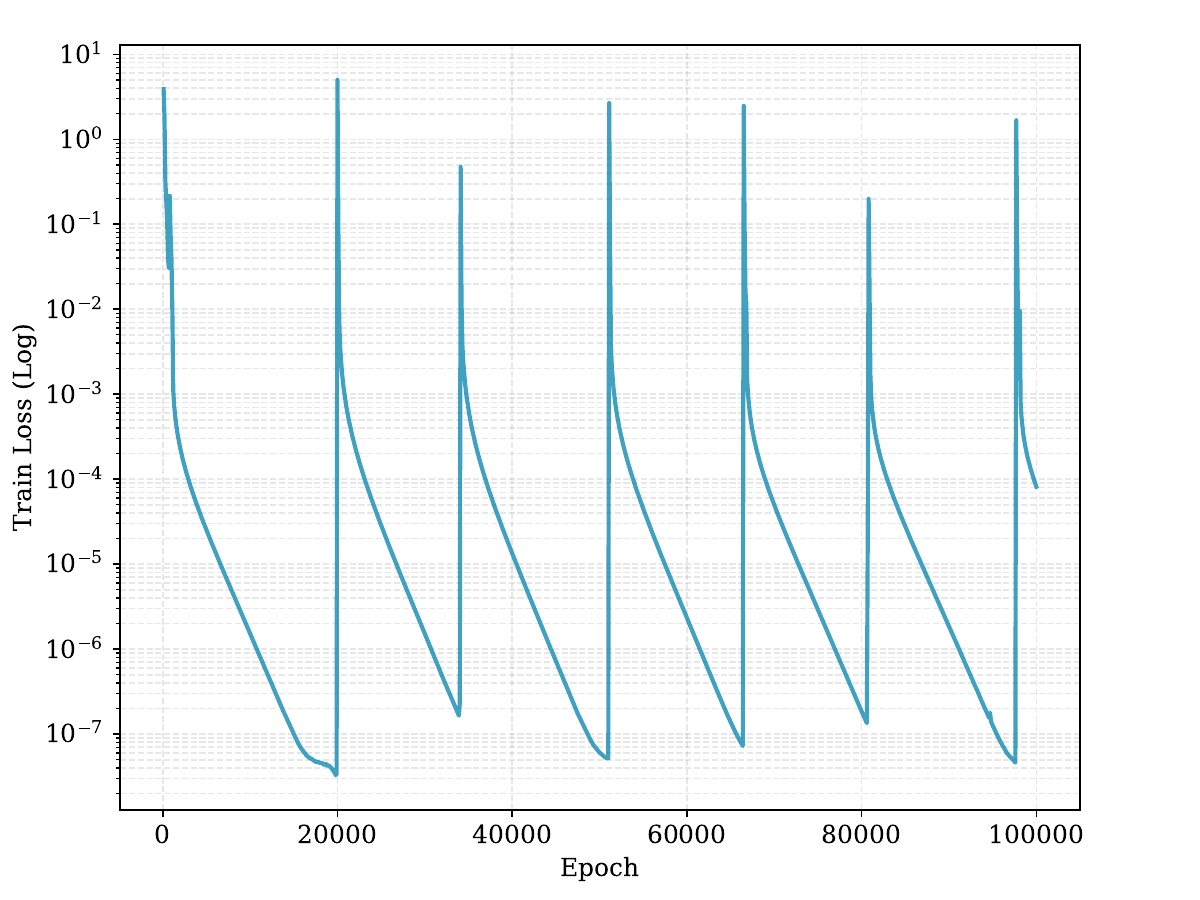}
     \end{subfigure}
     \hfill 
     \begin{subfigure}[b]{0.45\textwidth}
         \centering
        \includegraphics[width=\textwidth]{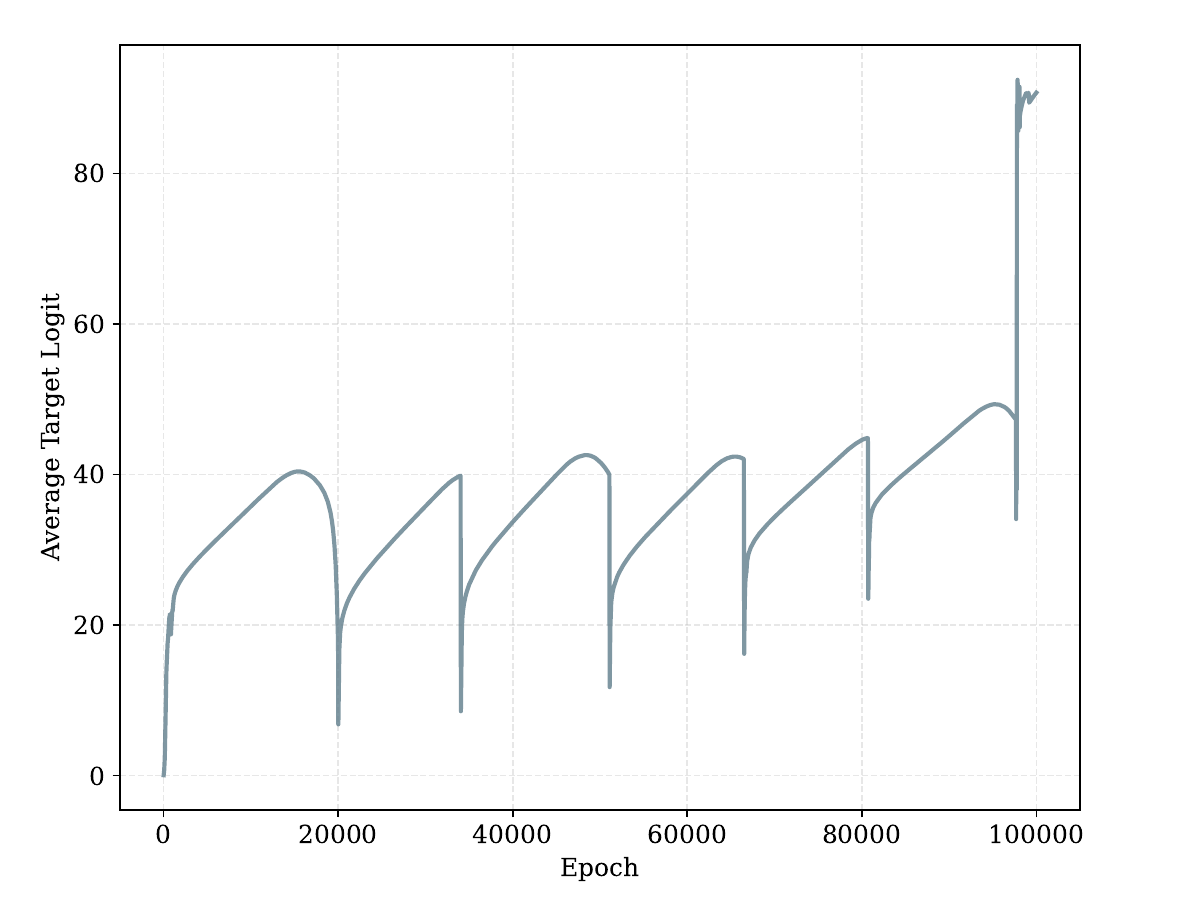}
    \end{subfigure}

     \caption{Slingshot in Transformer on modular division.}
     \label{fig:mod_trm}
\end{figure}

\begin{figure}[h]
     \centering
     \begin{subfigure}[b]{0.45\textwidth}
         \centering
         \includegraphics[width=\textwidth]{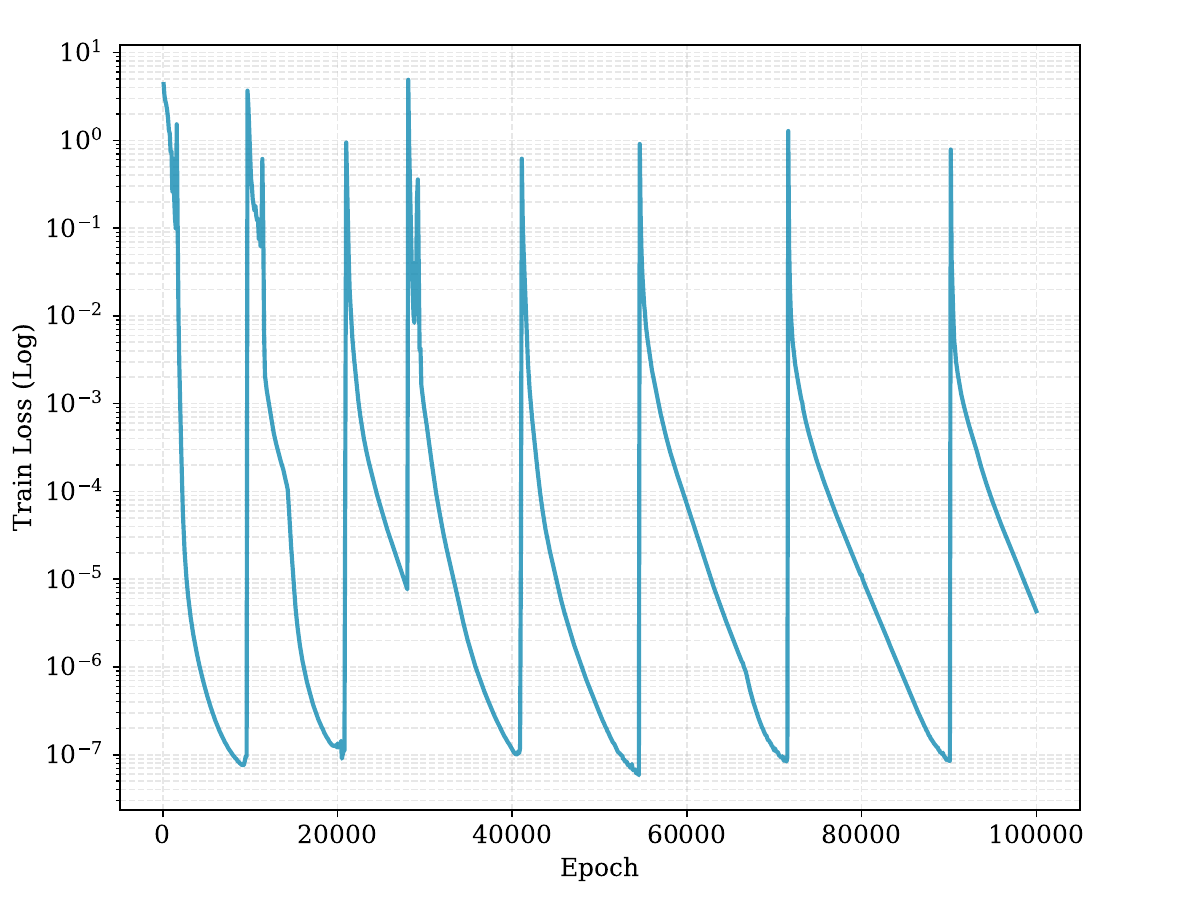}
     \end{subfigure}
     \hfill 
     \begin{subfigure}[b]{0.45\textwidth}
         \centering
        \includegraphics[width=\textwidth]{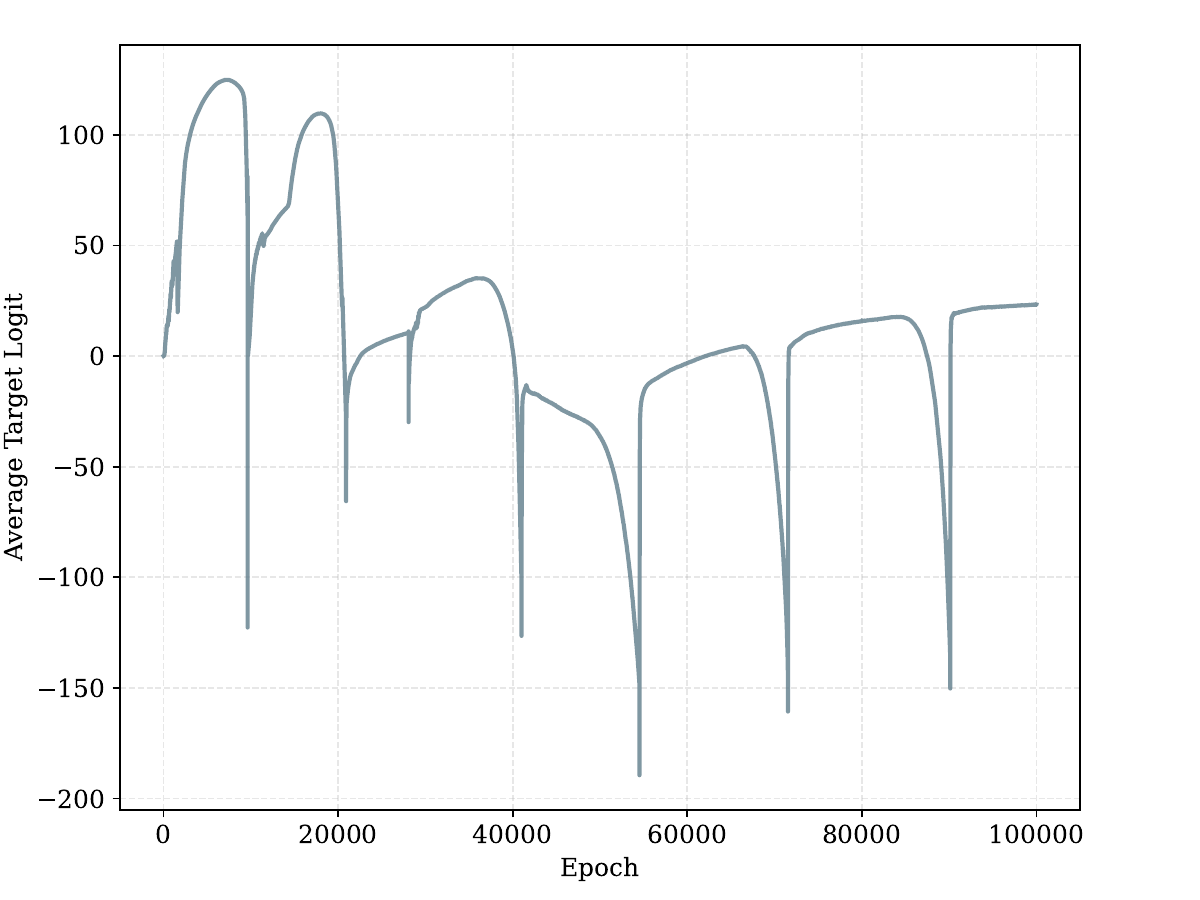}
    \end{subfigure}

     \caption{Slingshot in MLP on modular division.}
     \label{fig:mod_mlp}
\end{figure}

\begin{figure}[h]
     \centering
     \begin{subfigure}[b]{0.45\textwidth}
         \centering
         \includegraphics[width=\textwidth]{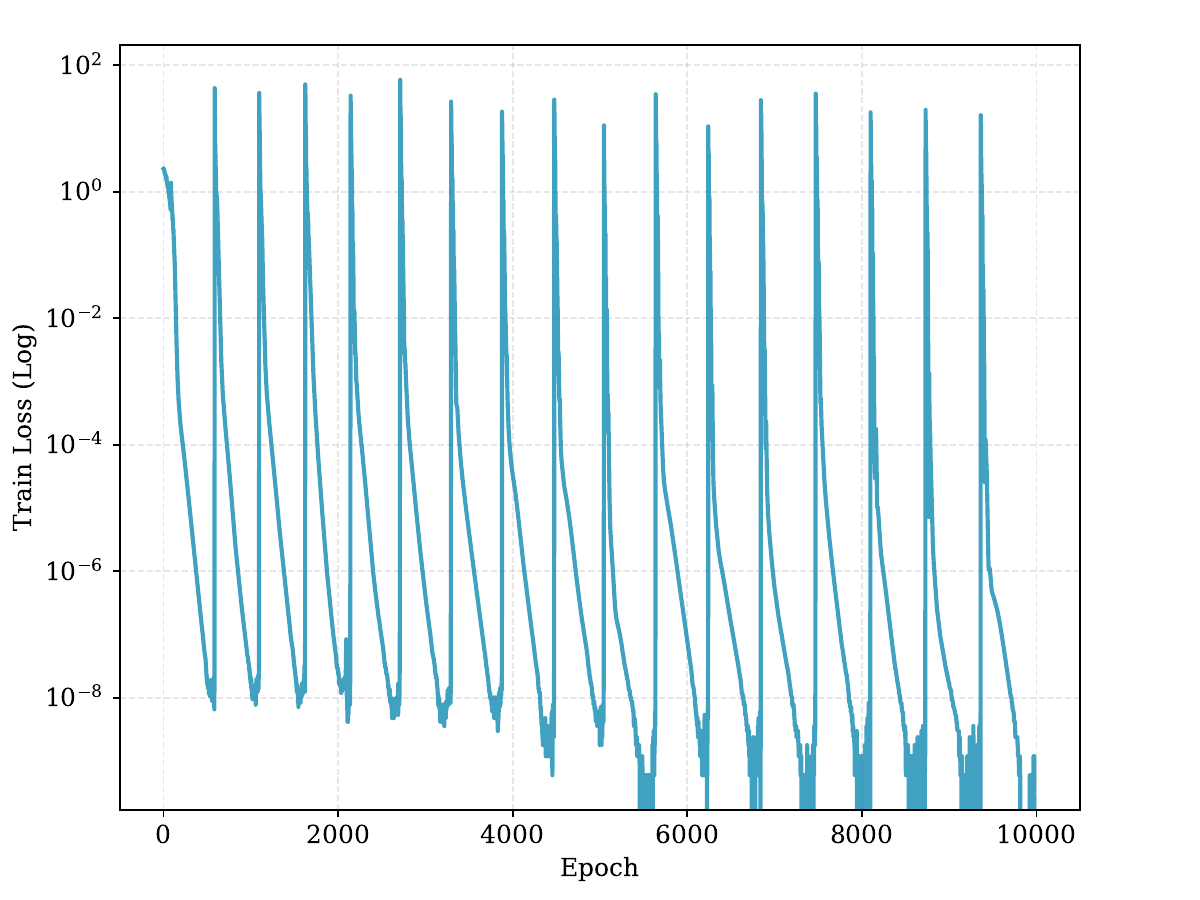}
     \end{subfigure}
     \hfill 
     \begin{subfigure}[b]{0.45\textwidth}
         \centering
        \includegraphics[width=\textwidth]{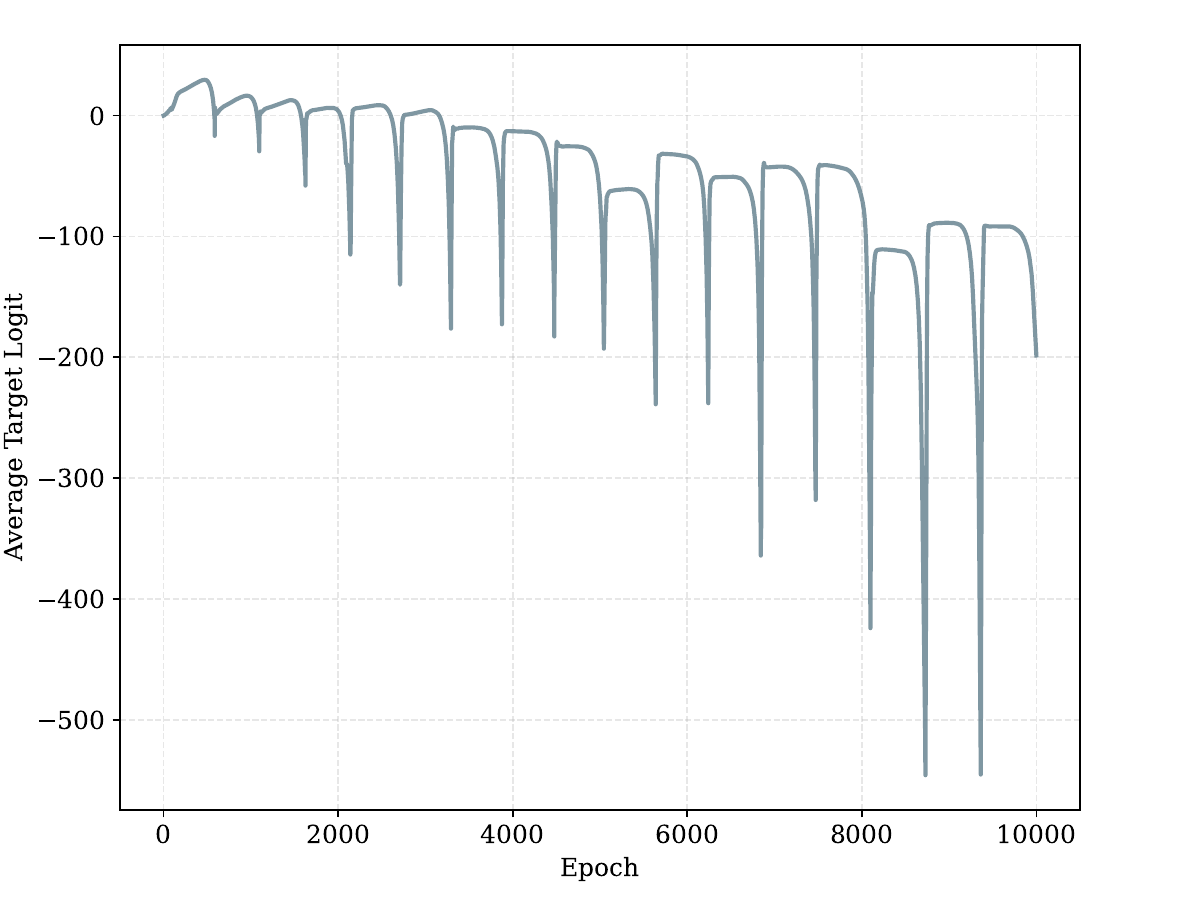}
    \end{subfigure}

     \caption{Slingshot in MLP on CIFAR-10.}
     \label{fig:mlp}
\end{figure}

\begin{figure}[h]
     \centering
     \begin{subfigure}[b]{0.45\textwidth}
         \centering
         \includegraphics[width=\textwidth]{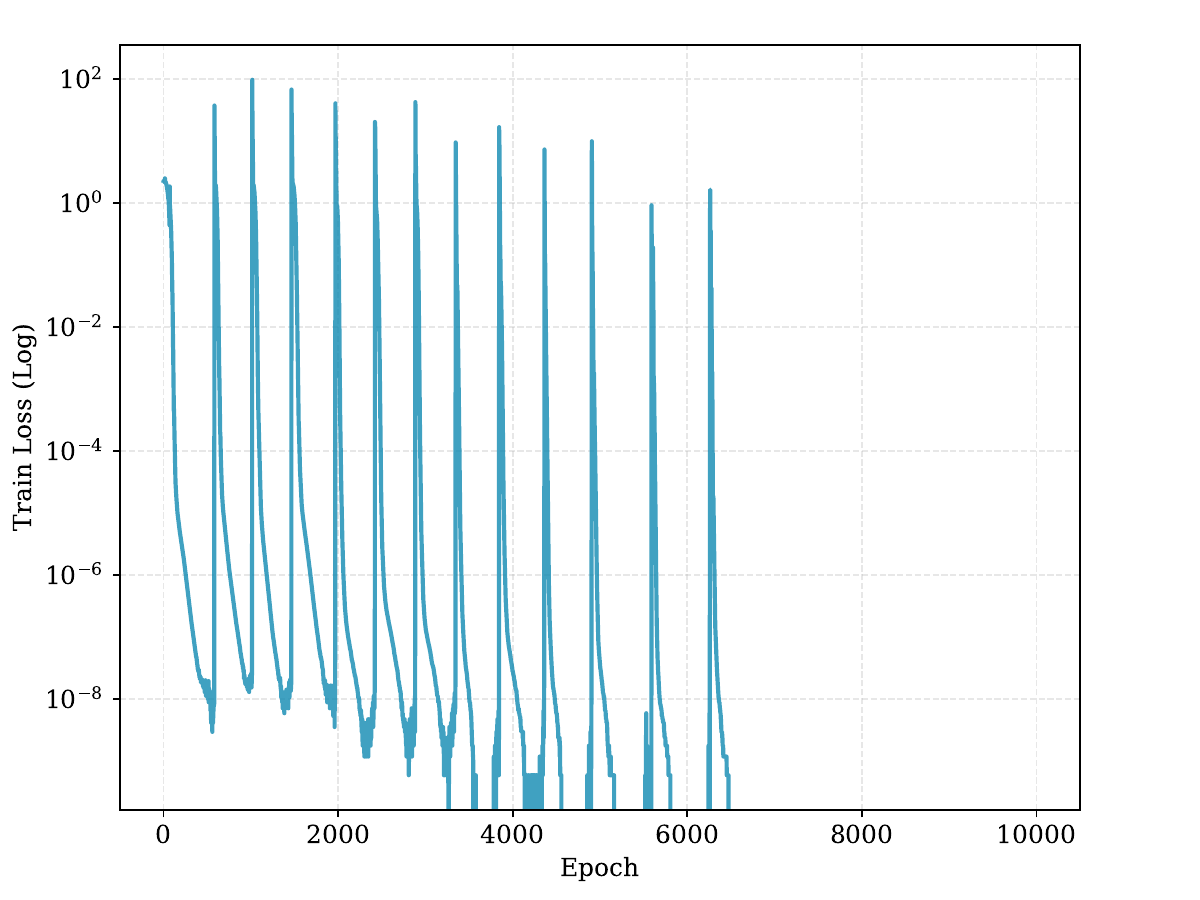}
     \end{subfigure}
     \hfill 
     \begin{subfigure}[b]{0.45\textwidth}
         \centering
        \includegraphics[width=\textwidth]{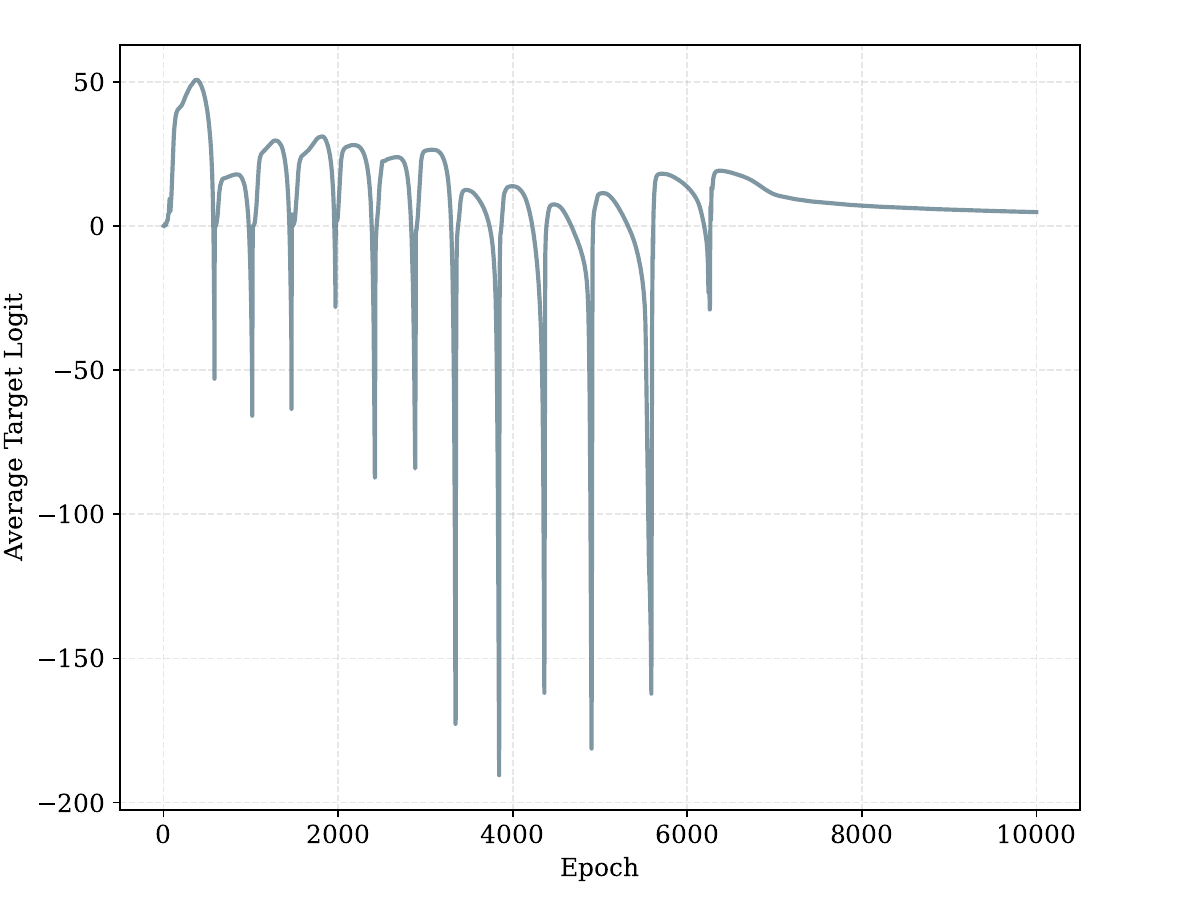}
    \end{subfigure}

     \caption{Slingshot in VGG11 on CIFAR-10.}
     \label{fig:vgg11}
\end{figure}

\begin{figure}[h!]
     \centering
     \begin{subfigure}[b]{0.45\textwidth}
         \centering
         \includegraphics[width=\textwidth]{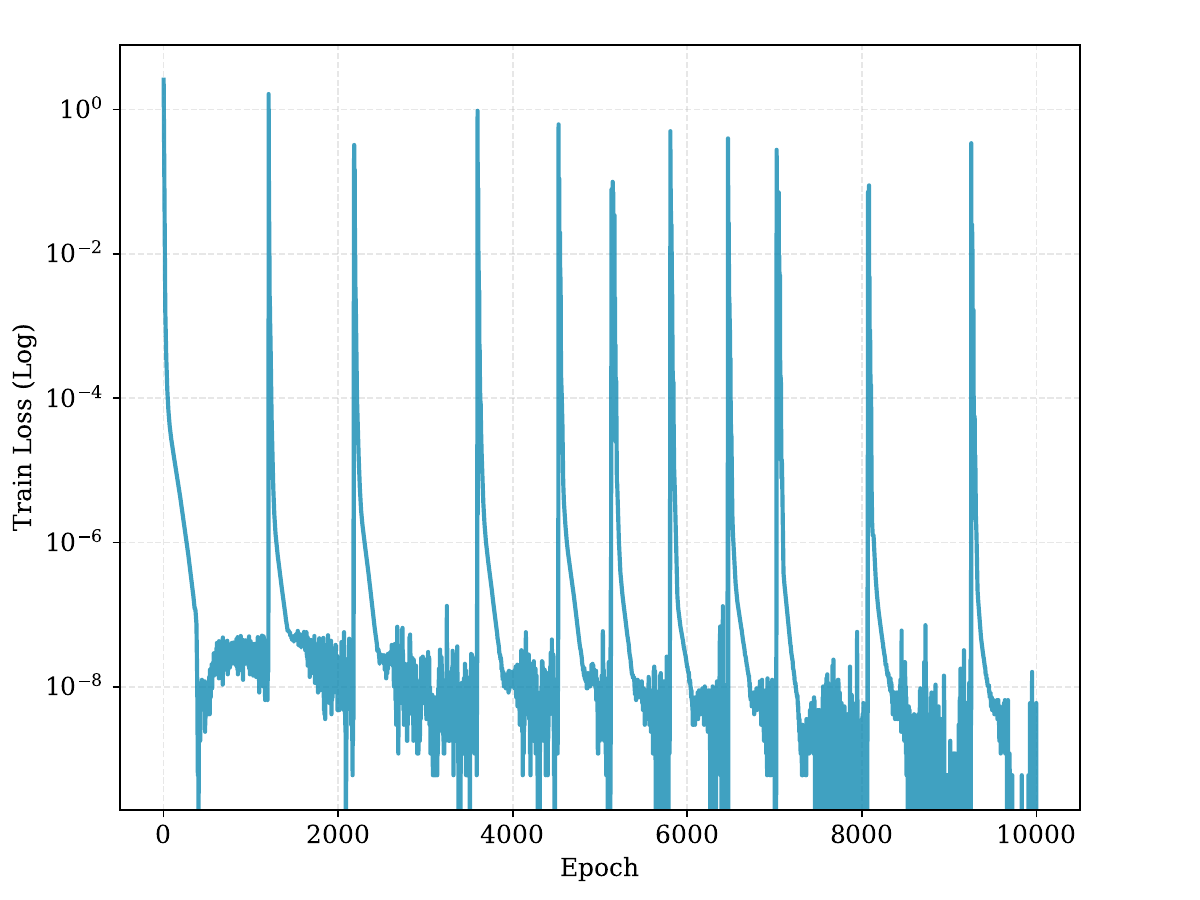}
     \end{subfigure}
     \hfill 
     \begin{subfigure}[b]{0.45\textwidth}
         \centering
        \includegraphics[width=\textwidth]{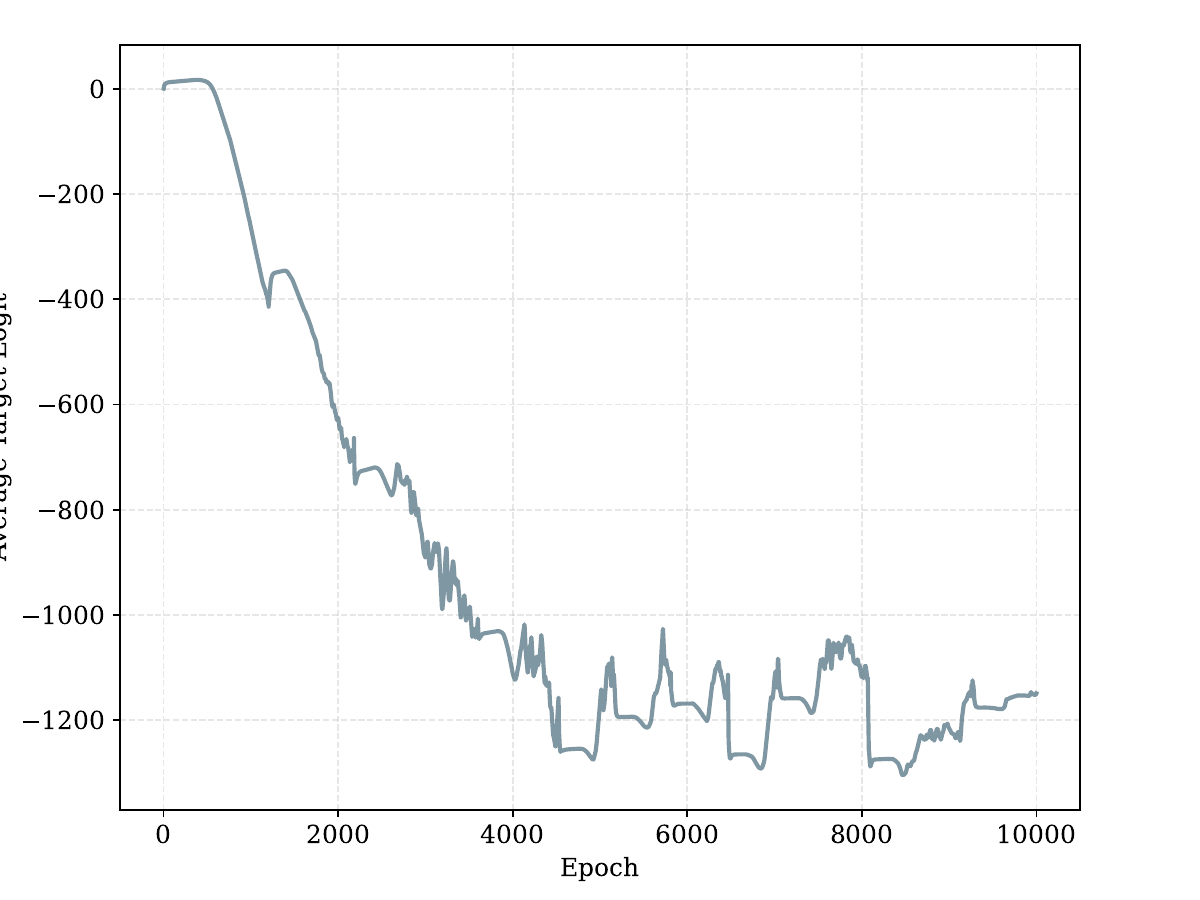}
    \end{subfigure}

     \caption{Slingshot in VGG11 with BN on CIFAR-10.}
     \label{fig:vgg11_bn}
\end{figure}

\begin{figure}[H]
     \centering
     \begin{subfigure}[b]{0.45\textwidth}
         \centering
         \includegraphics[width=\textwidth]{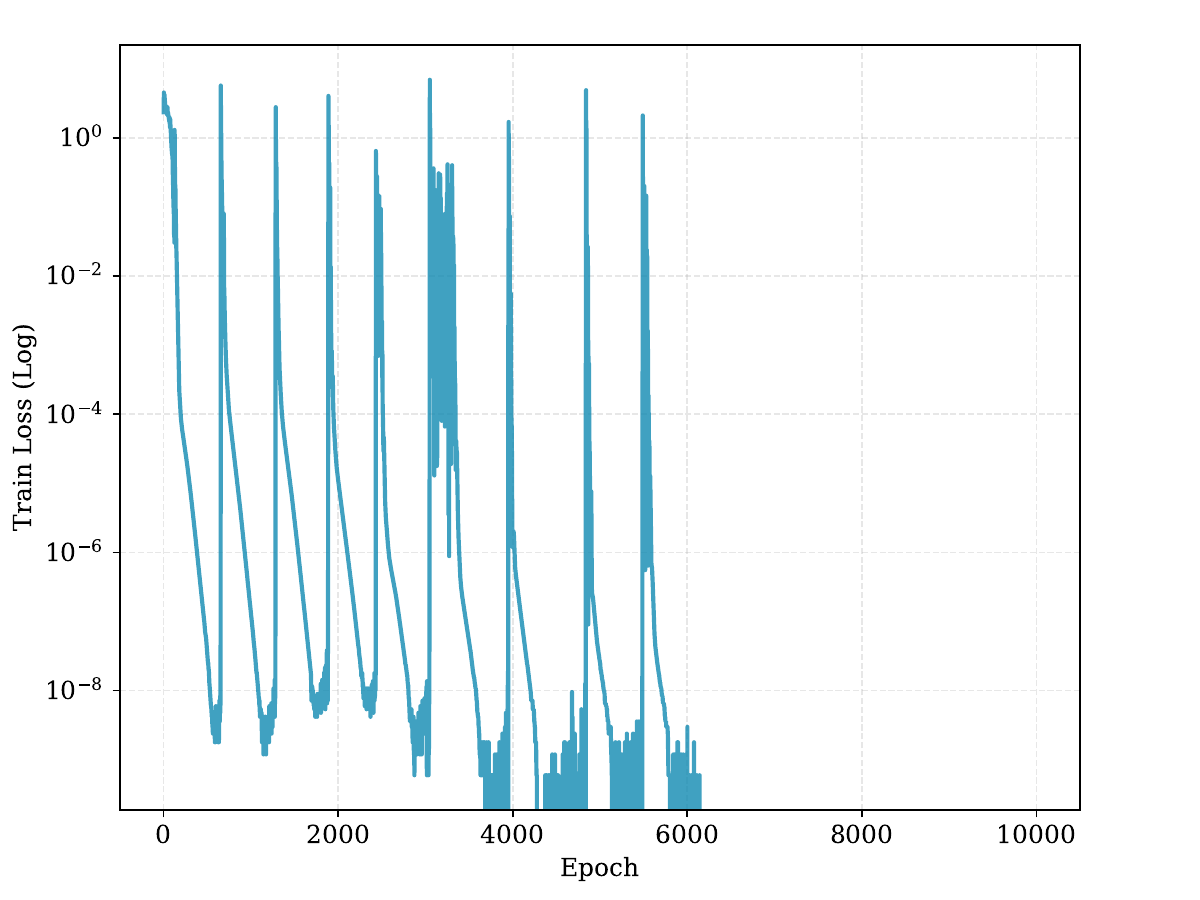}
     \end{subfigure}
     \hfill 
     \begin{subfigure}[b]{0.45\textwidth}
         \centering
        \includegraphics[width=\textwidth]{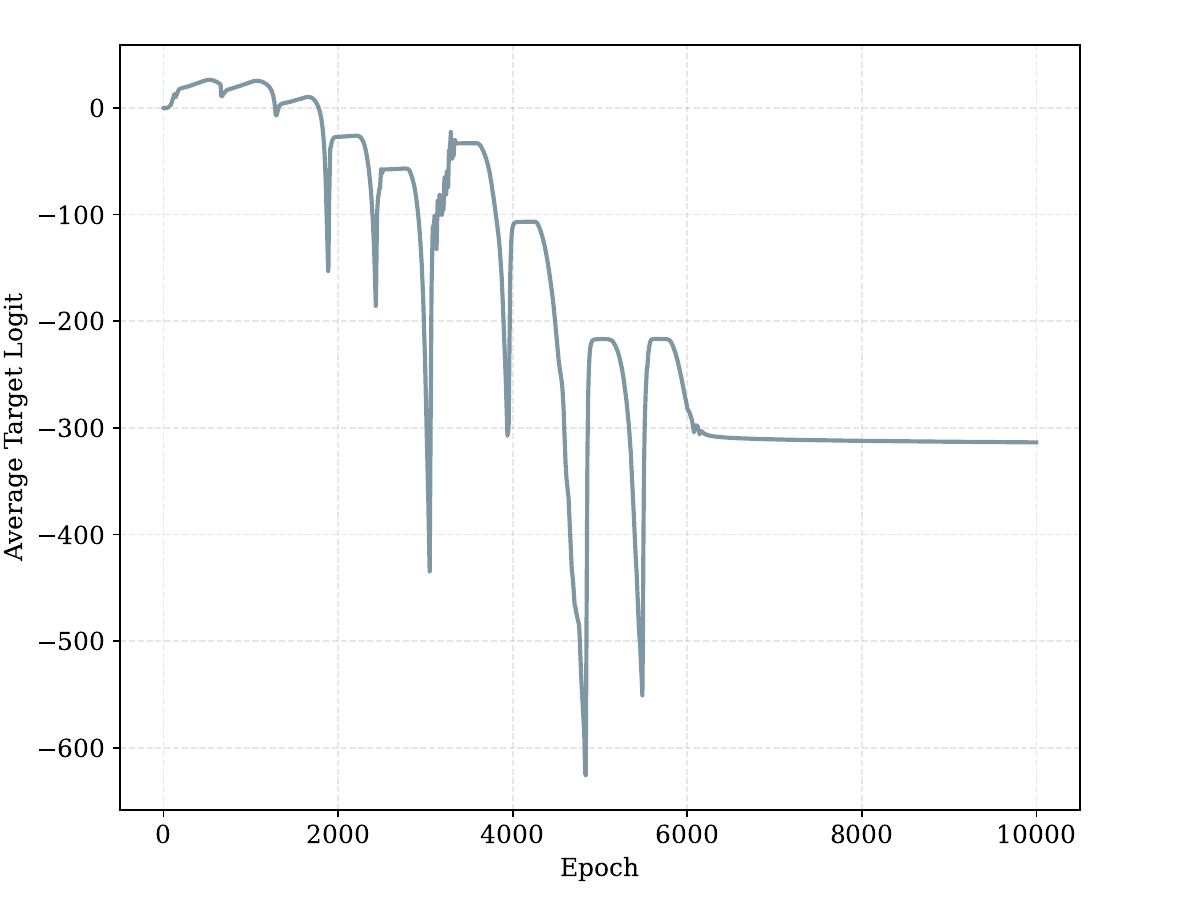}
    \end{subfigure}

     \caption{Slingshot in ViT on CIFAR-10.}
     \label{fig:vit}
\end{figure}

\begin{figure}[H]
     \centering
     \begin{subfigure}[b]{0.45\textwidth}
         \centering
         \includegraphics[width=\textwidth]{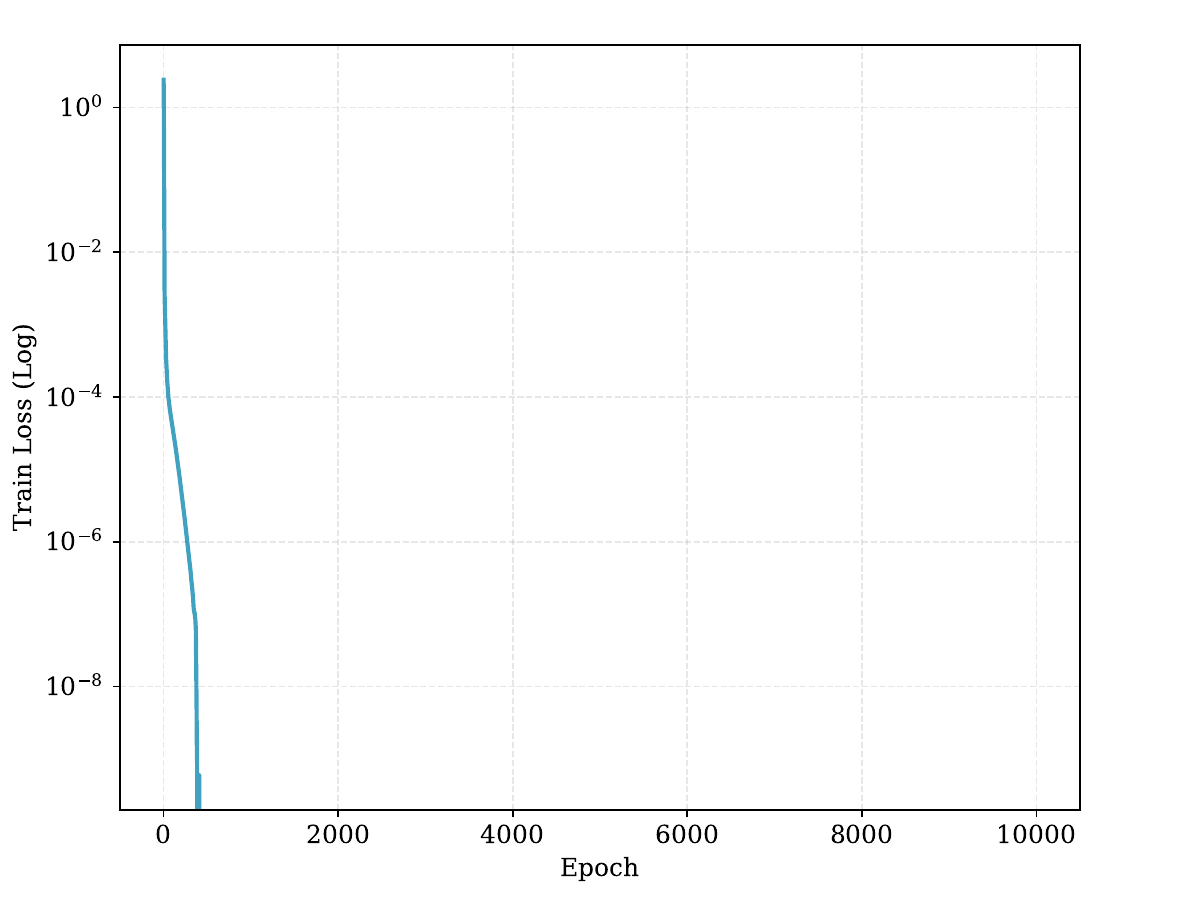}
     \end{subfigure}
     \hfill 
     \begin{subfigure}[b]{0.45\textwidth}
         \centering
        \includegraphics[width=\textwidth]{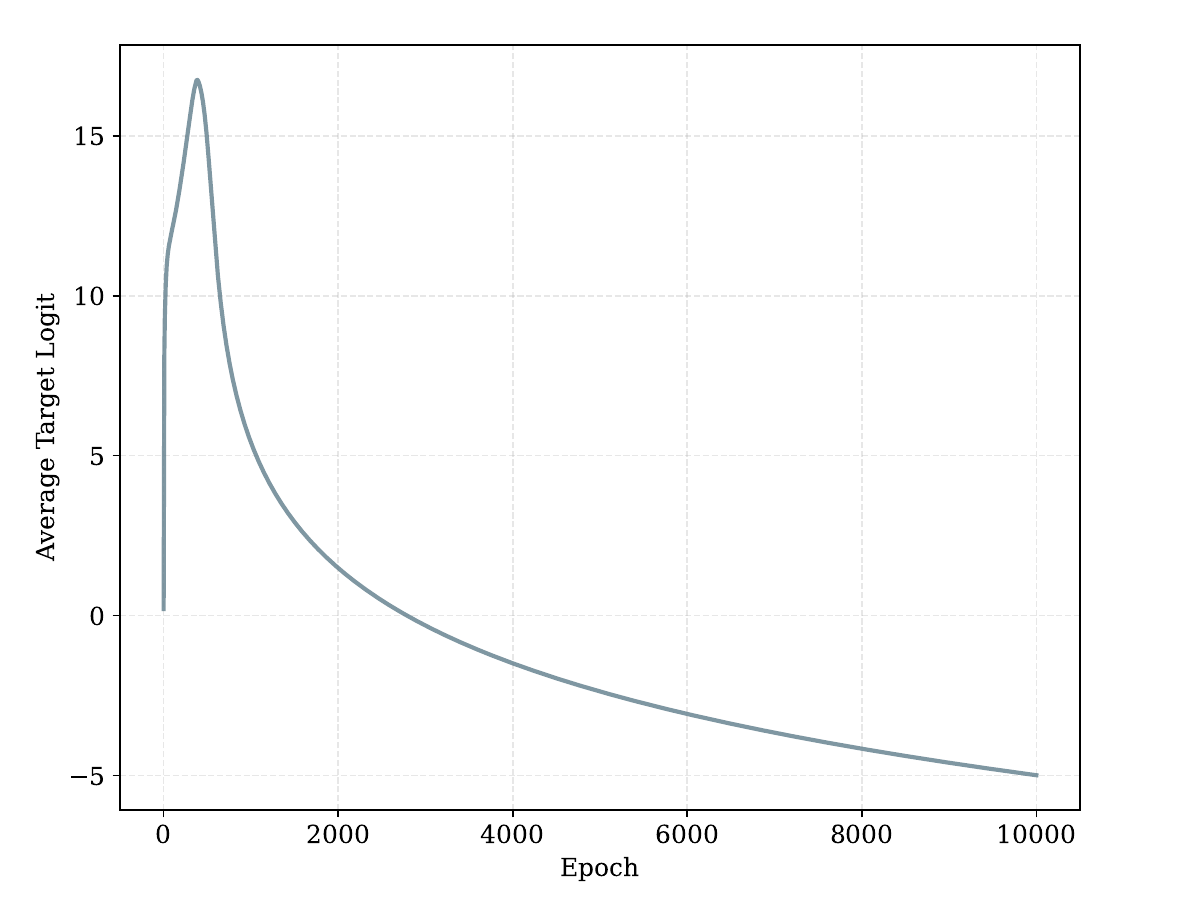}
    \end{subfigure}
        
     \caption{No Slingshot in ResNet18 on CIFAR-10.}
     \label{fig:resnet18}
\end{figure}

\begin{figure}[h]
     \centering
         \includegraphics[width=0.7\textwidth]{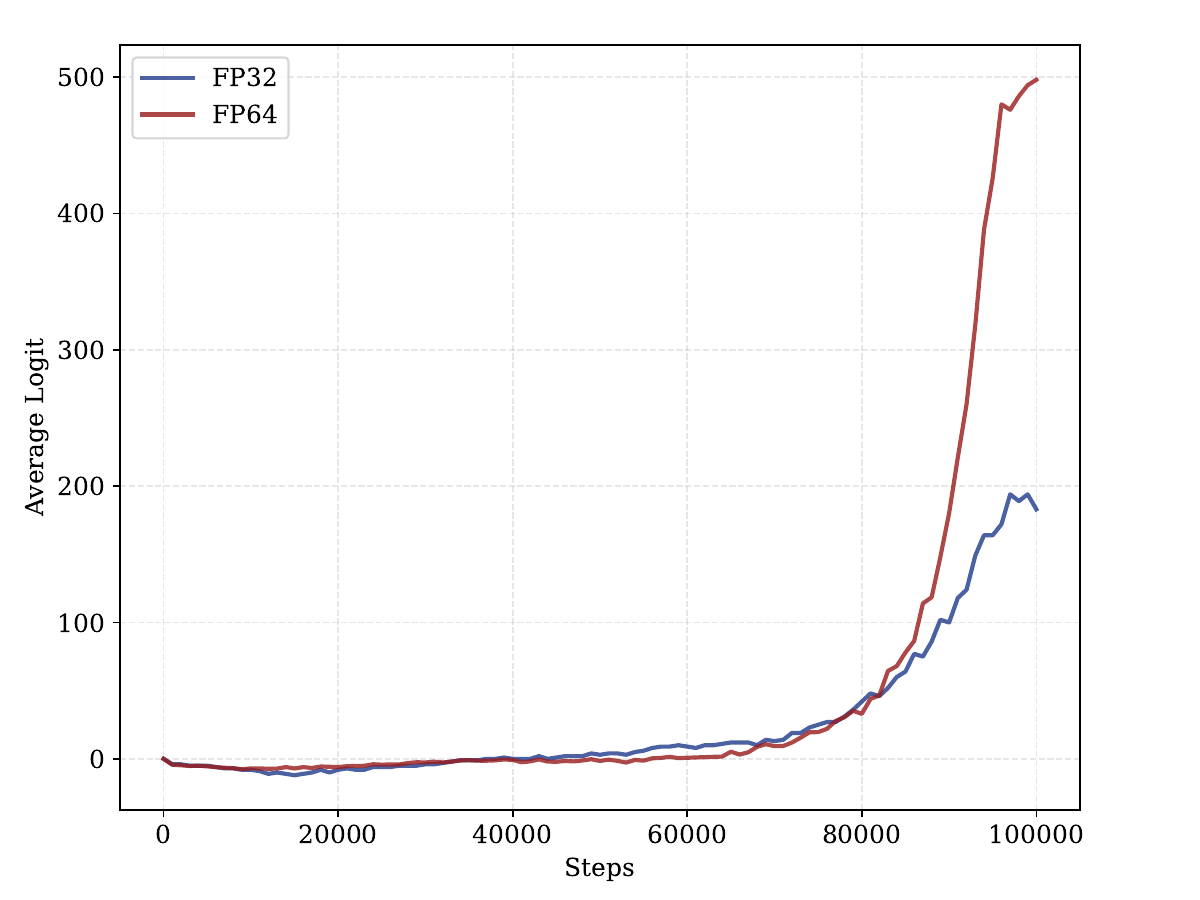} 
     \caption{The average logit of different precisions in LLM training.}
     \label{fig:llm_fp}
\end{figure}

\section{Proofs}
\label{appendix:proofs}
\subsection{Proof of \cref{thm:weight_drift}}
\label{app:proof_drift}

In this section, we provide the detailed derivation for \cref{thm:weight_drift}, demonstrating how floating-point absorption errors (Softmax Collapse) coupled with the geometry of Neural Collapse induce a deterministic drift in the global weight mean.

\subsubsection{Preliminaries and Assumptions}
We consider a classification task with $K$ classes, a batch size of $B$, and a learning rate $\eta$. Let $(\bm{x}, y)$ denote an input-label pair, and $\bm{h} \in \mathbb{R}^d$ be the feature vector of $\bm{x}$. The logits are given by $\bm{z} = \bm{W}\bm{h}$, where $\bm{W} \in \mathbb{R}^{K \times d}$.

We assume the model is in a state of approximate Neural Collapse (NC), satisfying:
\begin{itemize}
    \item \textbf{NC1:} Features collapse to class means: $\bm{h}_{k,i} \approx \bm{\mu}_k = \bm{\mu}_G + \bm{\mu}_k^*$.
    \item \textbf{NC2:} Centered class means $\bm{\mu}_k^*$ form a Simplex ETF satisfying $\langle \bm{\mu}_p^*, \bm{\mu}_q^* \rangle = -\frac{1}{K-1}\|\bm{\mu}^*\|^2$ for $p \neq q$.
    \item \textbf{NC3:} Classifiers align with features: $\bm{W}_k \propto \bm{\mu}_k^*$, implying $\|\bm{W}_k\| = \|\bm{W}\|$ and $\langle \bm{W}_p, \bm{W}_q \rangle = -\frac{1}{K-1}\|\bm{W}\|^2$ for centered weights.
\end{itemize}

\subsubsection{Breakdown of the Zero-Sum Constraint}
The gradient of the Cross-Entropy loss $\mathcal{L}$ with respect to the $k$-th classifier weight $\bm{W}_k$ for a single sample $(\bm{x}, y)$ is:
\begin{equation}
    \nabla_{\bm{W}_k} \mathcal{L} = (\hat{y}_k - y_k) \bm{h}
\end{equation}
where $y_k$ is the one-hot label (1 if $k=r$, 0 otherwise) and $\hat{y}_k$ is the softmax probability.

Consider the gradient of the global weight mean $\bm{W}_G = \frac{1}{K}\sum_{k=1}^K \bm{W}_k$. By summing the gradients over all classes:
\begin{equation}
    \nabla_{\bm{W}_G} \mathcal{L} = \frac{1}{K} \sum_{k=1}^K (\hat{y}_k - y_k) \bm{h} = \frac{1}{K} \left( \underbrace{\sum_{k=1}^K \hat{y}_k}_{1} - \underbrace{\sum_{k=1}^K y_k}_{1} \right) \bm{h} = 0
\end{equation}
In ideal arithmetic, $\bm{W}_G$ receives zero gradient and remains static.

\subsubsection{The Softmax Collapse (SC) Scenario}
Under SC, as defined in Section 3, the floating-point absorption error occurs when calculating the loss contribution of the correct class $r$. Specifically, the term $\hat{y}_r - 1$ is numerically rounded to $0$ because the precision limit prevents subtracting the small residual $(1-\hat{y}_r)$ from the large mantissa of 1.
Thus, the summation becomes:
\begin{align}
    \sum_{k=1}^K \nabla_{\bm{W}_k} \mathcal{L} &= (\hat{y}_r - y_r)\bm{h} + \sum_{k \neq r} (\hat{y}_k - 0) \bm{h} \\
    &\approx 0 \cdot \bm{h} + \sum_{k \neq r} \hat{y}_k \bm{h} \\
    &= \epsilon \bm{h}
\end{align}
where $\epsilon = \sum_{k \neq r} \hat{y}_k > 0$ represents the total probability mass assigned to incorrect classes.

\subsubsection{Quantifying the Residual $\epsilon$}
Assuming the model is in the NC state, we can approximate $\hat{y}_k$ for incorrect classes ($k \neq r$) using the ETF geometry.
The probability is given by $\hat{y}_k \approx \exp(z_k - z_r)$.
Using NC3 alignment, for a sample of class $r$:
\begin{itemize}
    \item Correct logit: $z_r = \langle \bm{W}_r, \bm{h} \rangle \approx \|\bm{W}\|\|\bm{\mu}^*\|$
    \item Incorrect logit ($k \neq r$): $z_k = \langle \bm{W}_k, \bm{h} \rangle \approx \|\bm{W}\|\|\bm{\mu}^*\| \cos \theta_{ETF} = -\frac{1}{K-1}\|\bm{W}\|\|\bm{\mu}^*\|$
\end{itemize}
Substituting these into the exponent:
\begin{align}
    \hat{y}_k &\approx \exp\left( -\frac{1}{K-1}\|\bm{W}\|\|\bm{\mu}^*\| - \|\bm{W}\|\|\bm{\mu}^*\| \right) \\
    &= \exp\left( -\frac{K}{K-1}\|\bm{W}\|\|\bm{\mu}^*\| \right)
\end{align}
Summing over the $K-1$ incorrect classes:
\begin{equation}
    \epsilon = \sum_{k \neq r} \hat{y}_k \approx (K-1) \exp\left( -\frac{K}{K-1}\|\bm{W}\|\|\bm{\mu}^*\| \right)
\end{equation}

\subsubsection{Batch Aggregation and Drift}
Finally, we consider the update rule for $\bm{W}_G$ over a batch of size $B$ with learning rate $\eta$. We assume mean reduction for the loss.
\begin{equation}
    \Delta \bm{W}_G = - \eta \left( \frac{1}{B} \sum_{i=1}^B \nabla_{\bm{W}_G}^{(i)} \mathcal{L} \right)
\end{equation}
Substituting the residual gradient $\sum_k \nabla_{\bm{W}_k}^{(i)} \mathcal{L} = \epsilon_i \bm{h}_i$, and noting that $\nabla_{\bm{W}_G} = \frac{1}{K} \sum_k \nabla_{\bm{W}_k}$:
\begin{equation}
    \Delta \bm{W}_G = - \frac{\eta}{KB} \sum_{i=1}^B \epsilon_i \bm{h}_i
\end{equation}
Assuming $\epsilon_i \approx \epsilon$ is roughly constant across the batch (due to NC), we focus on the sum of features $\sum \bm{h}_i$. Under NC1 (feature collapse) and a class balanced dataset:
\begin{equation}
    \mathbb{E}\left[\sum_{i=1}^B \bm{h}_i\right] = \sum_{k=1}^K \frac{B}{K} \bm{\mu}_k = \sum_{k=1}^K \frac{B}{K} (\bm{\mu}_G + \bm{\mu}_k^*)
\end{equation}
Since the centered means $\bm{\mu}_k^*$ form an ETF, $\sum_k \bm{\mu}_k^* = \mathbf{0}$. Thus:
\begin{equation}
    \mathbb{E}\left[\sum_{i=1}^B \bm{h}_i\right] = B \bm{\mu}_G
\end{equation}
Substituting this back into the update equation:
\begin{equation}
    \mathbb{E}[\Delta \bm{W}_G] = - \frac{\eta \epsilon}{KB} (B \bm{\mu}_G) = - \frac{\eta \epsilon}{K} \bm{\mu}_G
\end{equation}
\textbf{Conclusion:} Under Softmax Collapse, the global weight vector $\bm{W}_G$ drifts continuously in the direction opposite to the global feature mean $\bm{\mu}_G$, proving \cref{thm:weight_drift}.

\subsection{Proof of Proposition \ref{thm:nfc_force}}
\label{app:proof_nfc_force}

In this section, we provide the derivation for Proposition \ref{thm:nfc_force}. 
We analyze the gradient flow back to the feature layer under the regime of Softmax Collapse (SC) and drifted weights.

\subsubsection{Gradient Decomposition}
The gradient of the loss $\mathcal{L}$ with respect to the feature vector $\bm{h}$ is given by:
\begin{equation}
    \nabla_{\bm{h}} \mathcal{L} = \bm{W}^T (\hat{\bm{y}} - \bm{y}) = \sum_{k=1}^K (\hat{y}_k - y_k) \bm{W}_k
\end{equation}
We decompose the weight vectors into the global mean and the centered components: $\bm{W}_k = \bm{W}_G + \bm{W}_k^*$. Substituting this into the gradient:
\begin{equation}
    \nabla_{\bm{h}} \mathcal{L} = \sum_{k=1}^K (\hat{y}_k - y_k) (\bm{W}_G + \bm{W}_k^*)
\end{equation}

\subsubsection{Applying Softmax Collapse}
Under the SC condition, the gradient contribution from the correct class $r$ vanishes (i.e., the term corresponding to $(\hat{y}_r - 1)$ becomes strictly 0 due to absorption). For incorrect classes $k \neq r$, the label $y_k = 0$. Thus, the summation simplifies to:
\begin{equation}
    \nabla_{\bm{h}} \mathcal{L} \approx \sum_{k \neq r} \hat{y}_k (\bm{W}_G + \bm{W}_k^*)
\end{equation}
We can separate this sum into two components:
\begin{equation}
    \nabla_{\bm{h}} \mathcal{L} = \left( \sum_{k \neq r} \hat{y}_k \right) \bm{W}_G + \sum_{k \neq r} \hat{y}_k \bm{W}_k^*
\end{equation}
Using the definition $\epsilon = \sum_{k \neq r} \hat{y}_k$, the first term simplifies to $\epsilon \bm{W}_G$.

\subsubsection{Projection onto the Drift Direction}
We now calculate the projection of this gradient onto the direction of the global weight drift $\bm{W}_G$. The projection operator is defined as:
\begin{equation}
    \text{Proj}_{\bm{W}_G}(\nabla_{\bm{h}} \mathcal{L}) = \frac{\langle \nabla_{\bm{h}} \mathcal{L}, \bm{W}_G \rangle}{\|\bm{W}_G\|^2} \bm{W}_G
\end{equation}
First, we compute the inner product $\langle \nabla_{\bm{h}} \mathcal{L}, \bm{W}_G \rangle$:
\begin{align}
    \langle \nabla_{\bm{h}} \mathcal{L}, \bm{W}_G \rangle &= \left\langle \epsilon \bm{W}_G + \sum_{k \neq r} \hat{y}_k \bm{W}_k^*, \bm{W}_G \right\rangle \\
    &= \epsilon \langle \bm{W}_G, \bm{W}_G \rangle + \sum_{k \neq r} \hat{y}_k \langle \bm{W}_k^*, \bm{W}_G \rangle
\end{align}
We invoke the orthogonality assumption stated in Theorem 4: the global drift is orthogonal to the centered classification subspace, i.e., $\bm{W}_G \perp \text{span}\{\bm{W}_k^*\}$. Therefore, $\langle \bm{W}_k^*, \bm{W}_G \rangle = 0$ for all $k$.
The inner product simplifies to:
\begin{equation}
    \langle \nabla_{\bm{h}} \mathcal{L}, \bm{W}_G \rangle = \epsilon \|\bm{W}_G\|^2
\end{equation}
Substituting this back into the projection equation:
\begin{equation}
    \text{Proj}_{\bm{W}_G}(\nabla_{\bm{h}} \mathcal{L}) = \frac{\epsilon \|\bm{W}_G\|^2}{\|\bm{W}_G\|^2} \bm{W}_G = \epsilon \bm{W}_G
\end{equation}
\textbf{Conclusion:} The feature vector $\bm{h}$ receives a consistent, non-zero gradient component $\epsilon \bm{W}_G$. Since $\bm{W}_G$ drifts towards $-\bm{\mu}_G$ (from Theorem 2), the update step $\bm{h} \leftarrow \bm{h} - \eta \nabla_{\bm{h}} \mathcal{L}$ effectively adds a component aligned with $+\bm{\mu}_G$, driving Numerical Feature Inflation.

\subsection{Proof of Theorem \ref{thm:nfc}: Numerical Feature Inflation}
\label{app:proof_nfc}
In this section, we provide the rigorous proof for Theorem \ref{thm:nfc}. 
We define the alignment of features towards the global mean as the inevitable consequence of 
the feedback loop established in Theorems \ref{thm:weight_drift} and \ref{thm:nfc_force}.
\subsubsection{Geometric Orthogonality}
We first establish the geometric relationship between the global mean and the classification subspace.
\begin{lemma}[Orthogonality of Global Mean]
\label{lemma:orthogonality}
For a ReLU network in the NC state with $K$ balanced classes,
the global mean $\bm{\mu}_G = \frac{1}{K}\sum_{k=1}^K \bm{\mu}_k$ is orthogonal to the centered class mean subspace $\mathcal{S}_{NC} = \text{span}\{\bm{\mu}_k^*\}_{k=1}^K$.
\end{lemma}

\begin{proof}
Dang et al.~\cite{dang2024} have proved that for a ReLU network in the NC state, 
the uncentered class means $\bm{\mu}_k$ are mutually orthogonal ($\bm{\mu}_p^T \bm{\mu}_q = 0$ for $p \neq q$) and reside in the first orthant. 

The dataset is class balanced so that $\|\bm{\mu}_k\| = R$ for all $k$.
Since the uncentered class means $\{\bm{\mu}_k\}$ form an orthogonal basis in the feature space (scaled by $R$), we can compute the squared norm of the global mean:
\begin{equation}
    \|\bm{\mu}_G\|^2 = \left\| \frac{1}{K} \sum_{k=1}^K \bm{\mu}_k \right\|^2 = \frac{1}{K^2} \sum_{k=1}^K \|\bm{\mu}_k\|^2 = \frac{1}{K^2} \cdot K R^2 = \frac{R^2}{K}
\end{equation}
Next, we compute the projection of any class mean $\bm{\mu}_k$ onto the global mean:
\begin{equation}
    \langle \bm{\mu}_G, \bm{\mu}_k \rangle = \frac{1}{K} \sum_{j=1}^K \langle \bm{\mu}_j, \bm{\mu}_k \rangle = \frac{1}{K} \|\bm{\mu}_k\|^2 = \frac{R^2}{K}
\end{equation}
noting that cross-terms $\langle \bm{\mu}_j, \bm{\mu}_k \rangle$ vanish for $j \neq k$.
The centered class means are defined as $\bm{\mu}_k^* = \bm{\mu}_k - \bm{\mu}_G$. Calculating the inner product with $\bm{\mu}_G$:
\begin{equation}
    \langle \bm{\mu}_G, \bm{\mu}_k^* \rangle = \langle \bm{\mu}_G, \bm{\mu}_k - \bm{\mu}_G \rangle = \langle \bm{\mu}_G, \bm{\mu}_k \rangle - \|\bm{\mu}_G\|^2 = \frac{R^2}{K} - \frac{R^2}{K} = 0
\end{equation}
Since $\bm{\mu}_G$ is orthogonal to every basis vector $\bm{\mu}_k^*$ of the subspace $\mathcal{S}_{NC}$, it follows that $\bm{\mu}_G \perp \text{span}\{\bm{\mu}_k^*\}$.
\end{proof}
\subsubsection{Asymptotic Alignment via Coupled Dynamics}
\label{app:coupled_dynamics}
We treat the evolution of the global weight mean $\bm{W}_G$ and global feature mean $\bm{\mu}_G$ 
as a coupled linear dynamical system in $\mathbb{R}^d$.

We derive the update rules for the vectors.
\begin{itemize}
    \item From Theorem \ref{thm:weight_drift}, the weight mean update is:
    \begin{equation}
        \bm{W}_G^{(t+1)} = \bm{W}_G^{(t)} - \alpha \bm{\mu}_G^{(t)}
    \end{equation}
    where $\alpha = \frac{\eta \epsilon}{K} > 0$.
    \item From Proposition \ref{thm:nfc_force}, the feature gradients contain a component parallel to the current weight mean. 
    Assuming the ETF structure leads to the cancellation of orthogonal components when averaged over a batch 
    (i.e., $\sum_{k}\bm{W}_k^* \approx 0$), the update to the feature mean is dominated by:
    \begin{equation}
        \bm{\mu}_G^{(t+1)} = \bm{\mu}_G^{(t)} - \beta \bm{W}_G^{(t)}
    \end{equation}
    where $\beta = \eta \epsilon > 0$.
\end{itemize}

Let $\bm{u}^{(t)} = \begin{bmatrix} \bm{W}_G^{(t)} \\ \bm{\mu}_G^{(t)} \end{bmatrix}$ be the state vector in the product space $\mathbb{R}^{2d}$. 
The dynamics are governed by the block matrix $\mathbf{M}$:
\begin{equation}
\begin{bmatrix} \bm{W}_G^{(t+1)} \\ \bm{\mu}_G^{(t+1)} \end{bmatrix} = \begin{bmatrix} I & -\alpha I \\ -\beta I & I \end{bmatrix} \begin{bmatrix} \bm{W}_G^{(t)} \\ \bm{\mu}_G^{(t)} \end{bmatrix}
\end{equation}
where $I$ is the $d \times d$ identity matrix.

$\mathbf{M}$ has two eigenvalues:
\begin{align}
    \lambda_1 = 1 + \sqrt{\alpha \beta} = 1 + \eta \epsilon / \sqrt{K} > 1\\
    \lambda_2 = 1 - \sqrt{\alpha \beta} = 1 - \eta \epsilon / \sqrt{K} < 1
\end{align}
Solving $\mathbf{M} \bm{v} = \lambda \bm{v}$:
\begin{enumerate}
    \item For $\lambda_1$:
    \begin{equation}
        v_W - \alpha v_\mu = (1 + \sqrt{\alpha\beta}) v_W \implies -\alpha v_\mu = \sqrt{\alpha\beta} v_W \implies \bm{W}_G = -\frac{1}{\sqrt{K}} \bm{\mu}_G
    \end{equation}
    \item For $\lambda_2$, $\bm{W}_G = \frac{1}{\sqrt{K}} \bm{\mu}_G$.
\end{enumerate}
The state vector can be expressed as a linear combination of the eigenvectors:
\begin{equation}
    \bm{u}^{(t)} = c_1 \lambda_1^t \begin{bmatrix} -\frac{1}{\sqrt{K}} \bm{\mu}_G \\ \bm{\mu}_G \end{bmatrix} + c_2 \lambda_2^t \begin{bmatrix} \frac{1}{\sqrt{K}} \bm{\mu}_G \\ \bm{\mu}_G \end{bmatrix}
\end{equation}
As $t \to \infty$, the term with $\lambda_1$ dominates since $\lambda_1 > 1$ and $\lambda_2 < 1$:
\begin{equation}
    \lim_{t \to \infty} \bm{u}^{(t)} \propto \begin{bmatrix} -\frac{1}{\sqrt{K}} \bm{\mu}_G \\ \bm{\mu}_G \end{bmatrix}
\end{equation}
This leads to the asymptotic relationship:
\begin{equation}
    \lim_{t \to \infty} \bm{W}_G^{(t)} \approx -\frac{1}{\sqrt{K}} \bm{\mu}_G^{(t)}
\end{equation}
Therefore, the vectors asymptotically align with the relationship defined by the dominant eigenvector:
\begin{equation}
\lim_{t \to \infty} \cos(\bm{W}_G^{(t)}, \bm{\mu}_G^{(t)}) = -1
\end{equation}
And the norms grow exponentially:
\begin{align}
\|\bm{W}_G^{(t)}\| &\propto (1 + \frac{\eta\epsilon}{\sqrt{K}})^t \\
\|\bm{\mu}_G^{(t)}\| &\propto (1 + \frac{\eta\epsilon}{\sqrt{K}})^t
\end{align}
\subsubsection{Numerical Feature Inflation ($\mathcal{NFI}$)}
Finally, we prove the second claim: the features condense toward $\bm{\mu}_G$.
We decompose the feature update $\Delta \bm{h}$ into a parallel component along $\bm{\mu}_G$ and a perpendicular component in the subspace $\mathcal{S}_{NC}$.
From Proposition \ref{thm:nfc_force}, the gradient component along the drift direction is:
\begin{equation}
    \bm{g}_{\parallel} = \text{Proj}_{\bm{\mu}_G}(\nabla_{\bm{h}} \mathcal{L}) = \epsilon \bm{W}_G
\end{equation}
Given the result of \cref{app:coupled_dynamics}, for large $t$, we can approximate $\bm{W}_G \approx -\|\bm{W}_G\| \frac{\bm{\mu}_G}{\|\bm{\mu}_G\|}$. 
The gradient update step $h \leftarrow h - \eta \nabla_h \mathcal{L}$ adds a component:
\begin{equation}
\Delta h_{\parallel} = -\eta (\epsilon \bm{W}_G) \approx \eta \epsilon \|\bm{W}_G\| \frac{\bm{\mu}_G}{\|\bm{\mu}_G\|}
\end{equation}

For the perpendicular component, consider NC3$'$, the gradient is the weighted sum of centered ETF weights:
\begin{equation}
    \bm{g}_{\perp} = \sum_{k \neq r} \hat{y}_k \bm{W}_k^*
\end{equation}
In the ETF configuration, the vectors $\bm{W}_k^*$ sum to zero. 
The weighted sum $\bm{g}_{\perp}$ represents the residual interference. 
The magnitude of the parallel update is driven by the scalar sum of probabilities $\epsilon = \sum \hat{y}_k$,
 whereas the perpendicular update is a vector sum of randomly oriented Simplex vectors weighted by $\hat{y}_k \approx \epsilon/(K-1)$.
Comparing the growth rates:
\begin{equation}
    \frac{\|\Delta \bm{h}_{\parallel}\|}{\|\Delta \bm{h}_{\perp}\|} = \frac{\|\epsilon \bm{W}_G\|}{\|\sum_{k \neq r} \hat{y}_k \bm{W}_k^*\|} \approx \frac{\epsilon \|\bm{W}_G\|}{\sqrt{K-1} \frac{\epsilon}{K-1} \|\bm{W}^*\|} \propto \frac{\|\bm{W}_G\|}{\|\bm{W}^*\|}
\end{equation}
As $\|\bm{W}_G\|$ grows exponentially due to the feedback loop while $\|\bm{W}^*\|$ (representing the fixed classification structure) grows linearly or remains bounded relative to the drift, the parallel component dominates.
Therefore,
\begin{equation}
    \lim_{t \to \infty} \frac{\|\bm{h}_{\perp}\|}{\|\bm{h}_{\parallel}\|} \to 0 \implies \lim_{t \to \infty} \cos(\bm{h}_t, \bm{\mu}_G) \to 1
\end{equation}
This confirms that the feature space effectively collapses into a rank-1 subspace aligned with the global mean.

\subsection{Proof of \cref{thm:vanishing_hessian}}
\label{app:proof_lambda}
In this section, we provide the detailed proof for \cref{thm:vanishing_hessian}, 
demonstrating that the maximum eigenvalue of the Hessian matrix with respect to model parameters 
converges to zero as the model approaches the interpolation solution.

The Hessian matrix of the loss with respect to parameters, $H_\theta = \nabla^2_\theta \mathcal{L}$, can be decomposed into the generalized Gauss-Newton (GGN) term and the residual term \cite{papyan2020a}:
\begin{equation}
    H_\theta = \underbrace{J^T H_z J}_{G(\theta)} + \underbrace{\sum_{k=1}^K (\nabla_z \mathcal{L})_k \nabla^2_\theta z_k}_{E(\theta)}
    \label{eq:Hessian}
\end{equation}
where $J = \nabla_\theta z \in \mathbb{R}^{K \times d}$ is the Jacobian matrix of the logits, and $H_z = \nabla^2_z \mathcal{L} \in \mathbb{R}^{K \times K}$ is the Hessian of the loss with respect to the logits. We analyze the convergence of these two terms separately.

\subsubsection{Analysis of the Gauss-Newton Term $G(\theta)$.}
For Cross-Entropy loss with Softmax activation, the Hessian with respect to logits is given explicitly by:
\begin{equation}
    H_z = \text{diag}(\hat{y}) - \hat{y}\hat{y}^T
\end{equation}
where $\hat{y} = \text{softmax}(z)$ is the predicted probability vector. Since $H_z$ is positive semi-definite, its spectral norm (maximum eigenvalue) is bounded by its trace:
\begin{equation}
    \|H_z\|_2 = \lambda_{\max}(H_z) \leq \text{tr}(H_z) = \sum_{k=1}^K (\hat{y}_k - \hat{y}_k^2) = 1 - \|\hat{y}\|_2^2
\end{equation}
In the late stage of training, assuming the model enters the interpolation regime, the prediction $\hat{y}$ converges to the distinct one-hot label vector $y$. Since $\|y\|_2^2 = 1$, we have:
\begin{equation}
    \lim_{\hat{y} \to y} \|H_z\|_2 \leq \lim_{\hat{y} \to y} (1 - \|\hat{y}\|_2^2) = 0
\end{equation}
Using the sub-multiplicativity of the matrix norm, the norm of the GGN term is bounded by:
\begin{equation}
    \|G(\theta)\|_2 = \|J^T H_z J\|_2 \leq \|J\|_2^2 \|H_z\|_2
\end{equation}
Assuming the Jacobian $J$ is bounded in the local convergence region (i.e., $\exists M_1 > 0$ such that $\|J\|_2 \leq M_1$), it follows that:
\begin{equation}
    \lim_{\hat{y} \to y} \|G(\theta)\|_2 = 0
\end{equation}

\subsubsection{Analysis of the Residual Term $E(\theta)$.}
The gradient of the loss with respect to logits is the prediction error: $\nabla_z \mathcal{L} = \hat{y} - y$. The residual term can be expanded as:
\begin{equation}
    E(\theta) = \sum_{k=1}^K (\hat{y}_k - y_k) \nabla^2_\theta z_k
\end{equation}
Applying the triangle inequality:
\begin{equation}
    \|E(\theta)\|_2 \leq \sum_{k=1}^K |\hat{y}_k - y_k| \cdot \|\nabla^2_\theta z_k\|_2
\end{equation}
Assume the network function $z(\theta)$ is $C^2$ continuous and its second-order derivatives are locally bounded (i.e., $\exists M_2 > 0$ such that $\|\nabla^2_\theta z_k\|_2 \leq M_2$). As $\hat{y} \to y$, the error term $|\hat{y}_k - y_k| \to 0$ for all classes $k$. Consequently:
\begin{equation}
    \lim_{\hat{y} \to y} \|E(\theta)\|_2 = 0
\end{equation}

\paragraph{Conclusion.}
Combining the bounds for both terms:
\begin{equation}
    \|H_\theta\|_2 \leq \|G(\theta)\|_2 + \|E(\theta)\|_2
\end{equation}
Since both $\|G(\theta)\|_2 \to 0$ and $\|E(\theta)\|_2 \to 0$ as the model converges to the interpolation solution, we conclude that:
\begin{equation}
    \lim_{\hat{y} \to y} \lambda_{\max}(H_\theta) = 0
\end{equation}

\subsubsection{Label Smoothing leads to Non-Vanishing Hessian}
\label{app:ls}
In \cref{thm:vanishing_hessian}, we proved that under standard Cross-Entropy loss with hard targets, 
the maximum eigenvalue of the Hessian $\lambda_{max} \to 0$ 
because the predicted probability vector $\hat{y}$ converges to a one-hot vector. 
Here, we demonstrate that Label Smoothing fundamentally alters this asymptotic behavior.

\begin{definition}
[Label Smoothing]
Let $y$ be the one-hot label for the correct class $r$. Label Smoothing replaces $y$ with a soft target $y^{LS}$:
\begin{equation}
y^{LS}_k = \begin{cases} 1 - \alpha & \text{if } k = r \\ \frac{\alpha}{K-1} & \text{if } k \neq r \end{cases}
\end{equation}
where $\alpha \in (0, 1)$ is the smoothing parameter.
\end{definition}

\begin{proof}
    Consider the Gauss-Newton decomposition of the Hessian as defined in \cref{eq:Hessian}. 
    The core component governing the scale of the Hessian eigenvalues is the Hessian of the loss with respect to logits, 
    $H_z = \text{diag}(\hat{y}) - \hat{y}\hat{y}^T$.

    Unlike standard training where logits diverge to infinity to minimize loss (driving $\hat{y} \to y$), 
    under Label Smoothing, the global minimum is achieved at finite logit values 
    where the predicted distribution matches the soft target exactly:
    \begin{equation}
    \lim_{t \to \infty} \hat{y} = y^{LS}
    \end{equation}
    Consequently, for the correct class $r$, the prediction $\hat{y}_r$ converges to $1-\alpha$ rather than $1$.

    The trace of $H_z$ represents the sum of eigenvalues (variances of the categorical distribution):
    \begin{equation}
    \text{tr}(H_z) = \sum_{k=1}^K (\hat{y}_k - \hat{y}_k^2) = 1 - \|\hat{y}\|_2^2
    \end{equation}
    Substituting the limit $\hat{y} \to y^{LS}$:
    \begin{align}
    \lim_{\hat{y} \to y^{LS}} \text{tr}(H_z) &= 1 - \left( (1-\alpha)^2 + (K-1) \left(\frac{\alpha}{K-1}\right)^2 \right) \\
    &= 1 - \left( (1-\alpha)^2 + \frac{\alpha^2}{K-1} \right) \\
    &= 2\alpha - \alpha^2 - \frac{\alpha^2}{K-1} \\
    &= \alpha \left( 2 - \alpha \left(1 + \frac{1}{K-1}\right) \right)
    \end{align}
    Since $\alpha \in (0, 1)$, the term $2 - \alpha \left(1 + \frac{1}{K-1}\right) > 0$. 
    Therefore, the trace of $H_z$ converges to a positive constant:
    \begin{equation}
    \lim_{\hat{y} \to y^{LS}} \text{tr}(H_z) = C_{LS} > 0
    \end{equation}
    Since $H_z$ is positive semi-definite, its maximum eigenvalue is bounded below by the average eigenvalue:
    \begin{equation}
    \lambda_{\max}(H_z) \geq \frac{\text{tr}(H_z)}{K} = \frac{C_{LS}}{K} > 0
    \end{equation}
    Thus, the Gauss-Newton term satisfies:
    \begin{equation}
    \|G(\theta)\|_2 = \|J^T H_z J\|_2 \geq \|J\|_2^2 \cdot \frac{C_{LS}}{K} > 0
    \end{equation}
    Assuming the Jacobian $J$ remains bounded away from zero, we conclude that:
    \begin{equation}
    \lim_{t \to \infty} \lambda_{\max}(H_\theta) \geq \lim_{t \to \infty} \|G(\theta)\|_2 > 0
    \end{equation}
    Therefore, under Label Smoothing, the maximum eigenvalue of the Hessian does not vanish, 
    confirming that $\lambda_{\max}(H_\theta)$ converges to a positive constant rather than zero.
\end{proof}


\section{Experimental Details}
All of the experiments were conducted on NVIDIA RTX5090 GPUs. 
For modular and CIFAR-10 experiments, each run takes approximately 1 hour. 
For LLM experiments, each run takes approximately 12 hours in \texttt{float32} 
and more than 24 hours with the loss computed in \texttt{float64}.
\label{app:exp_details}
\subsection{Modular Arithmetic}
For the modular arithmetic experiments, we focus on the task of modular division. 
Given a prime $p=97$, the model is trained to predict $c = (a \cdot b^{-1}) \pmod p$ for 
all pairs $a \in \{0, \dots, p-1\}$ and $b \in \{1, \dots, p-1\}$.

\paragraph{Dataset and Task.}
The dataset consists of $p(p-1)$ samples. 
Each input is formatted as a sequence of four tokens: $[a, \text{OP}, b, =]$, 
where OP represents the division operator. 
The vocabulary size is $p+2$ to account for the operator and the equal sign. 
We employ a 50/50 training and validation split.

\paragraph{Model Architectures.}
We evaluate two primary architectures: a decoder-only Transformer and a deep MLP.

\begin{itemize}
\item \textbf{Transformer}~\cite{vaswani2017}: A 2-layer decoder-only Transformer. Each layer consists of a multi-head self-attention mechanism with $n_{head}=4$ and a feed-forward network (FFN) with a hidden dimension of $4 \times d_{model}$. The model uses $d_{model}=128$. In our primary unregularized experiments, Layer Normalization is disabled to isolate the Slingshot dynamics.
\item \textbf{MLP}: A 6-layer fully connected MLP. The architecture comprises 5 hidden layers with a width of 512 and ReLU activations, followed by a final linear classification layer. The input is provided as flattened one-hot encodings of the token sequence. Bias terms are included in all linear layers.
\end{itemize}

\paragraph{Training and Optimization.}
All models are trained using the Adam optimizer with a cross-entropy loss function. 
No explicit regularization is applied during the baseline runs.

\begin{table}[h]
\centering
\caption{Hyperparameters for Modular Arithmetic Experiments}
\label{tab:hyperparams-mod}
\begin{tabular}{ll}
\hline
\textbf{Hyperparameter} & \textbf{Value} \\ 
\hline
Prime $p$ & 97 \\
Optimizer & Adam \\
Learning Rate & $10^{-3}$ \\
Warmup Steps & 10 \\
$\beta_1, \beta_2$ & 0.9, 0.999 \\
$\varepsilon$ (Adam) & $10^{-8}$ \\
Weight Decay & 0.0 \\
Batch Size & 512 \\
Total Steps & 100,000 \\
Numerical Precision & Float32 \\
Random Seed & 42 \\
\hline
\end{tabular}
\end{table}

\subsection{Image Classification (CIFAR-10)}
For image classification tasks, we evaluate the Slingshot mechanism using the CIFAR-10 dataset.

\paragraph{Dataset and Preprocessing}
Consistent with prior work on the Slingshot mechanism, 
we utilize the CIFAR-10 dataset. 
For evaluation, we use a fixed test set of 1,000 samples. 
Images are converted to tensors using a standard \texttt{transforms.ToTensor()} operation 
without additional data augmentation to ensure the optimization dynamics remain deterministic and 
focused on the numerical properties of the loss landscape.

\paragraph{Model Architectures}
We conduct experiments across four distinct architectures to verify 
the universality of $\mathcal{NFI}$:

\begin{itemize}
\item \textbf{MLP}: A 6-layer fully connected MLP. The architecture comprises 5 hidden layers with a width of 512 and ReLU activations, followed by a final linear classification layer.
\item \textbf{ResNet18}~\cite{he2016}: A modified ResNet18 architecture optimized for the $32 \times 32$ resolution of CIFAR-10. The first convolutional layer is replaced with a $3 \times 3$ kernel (stride 1, padding 1) and the initial max-pooling layer is removed to prevent excessive spatial downsampling in early layers.
\item \textbf{VGG11}~\cite{simonyan2015}: A standard VGG11 architecture, evaluated both with and without Batch Normalization. The classification head is simplified to a single linear layer to match the feature-classifier interface of our theoretical model.
\item \textbf{ViT}~\cite{dosovitskiy2021}: A small-scale Vision Transformer with approximately 10M parameters. It utilizes a patch size of 4, an embedding dimension of 256, 12 transformer blocks, and 8 attention heads.
\end{itemize}

\paragraph{Training and Optimization}
All models are trained for 10,000 epochs using the Adam optimizer. 
To isolate the effects of $\mathcal{NFI}$, we set weight decay to zero and use a standard cross-entropy loss without label smoothing.
\begin{table}[h!]
\centering
\caption{Hyperparameters for CIFAR-10 Image Classification}
\label{tab:hyperparams-cifar}
\begin{tabular}{ll}
\hline
\textbf{Hyperparameter} & \textbf{Value} \\ 
\hline
Optimizer & Adam \\
Learning Rate & $10^{-3}$ \\
$\beta_1, \beta_2$ & 0.9, 0.95 \\
$\varepsilon$ (Adam) & $10^{-8}$ \\
Weight Decay & 0.0 \\
Label Smoothing & 0.0 \\
Batch Size & 2048 for ViT, Full Batch for others \\
Total Epochs & 10,000 \\
Numerical Precision & Float32 \\
Random Seed & 42 \\ 
\hline
\end{tabular}
\end{table}

\newpage

\subsection{Large Language Models}

For the language modeling experiments, we use a nanoGPT-style decoder-only Transformer. 
We follow the experimental setup of Stollenwerk et al.~\cite{stollenwerk2026} as closely as possible, 
including the optimizer, architecture, precision, training schedule, and data-processing settings. 
We do not change the training recipe, and only add additional measurements needed for our analysis, 
including the fraction of tokens in Softmax Collapse, the mean and norm of output logits, 
the global mean of the output embedding $\bm{W}_G$, and the effect of removing $\bm{W}_G$ during evaluation or training. 
The main hyperparameters are summarized in Table~\ref{tab:llm_details}.

\begin{table}[H]
\centering
\caption{
Hyperparameters for the language modeling experiments. 
}
\label{tab:llm_details}
\begin{tabular}{ll}
\toprule
\textbf{Hyperparameter} & \textbf{Value} \\
\midrule
optimizer & AdamW \\
$\beta_1$ & $0.9$ \\
$\beta_2$ & $0.95$ \\
$\varepsilon_{\mathrm{Adam}}$ & $10^{-8}$ \\
weight decay & $0.0$ \\
gradient clipping & $1.0$ \\
dropout & $0.0$ \\
weight tying & false \\
QK-layernorm & yes \\
bias & no \\
learning rate schedule & cosine decay \\
minimum learning rate & $10^{-5}$ \\
normalization & LayerNorm \\
precision & BF16 \\
positional embedding & RoPE \\
vocabulary size & $50304$\\
hidden activation & SwiGLU \\
sequence length & $2048$  \\
batch size (samples) & $64$  \\
batch size (tokens) & $131072$ \\
training length & $100000$ steps $\approx 13.1$B tokens \\
warmup & $5000$ steps $\approx 0.7$B tokens \\
embedding initialization & Normal with standard deviation $1/\sqrt{d}$ \\
weight initialization & Xavier with average of fan\_in and fan\_out \\
 
\bottomrule
\end{tabular}
\end{table}



\end{document}